\documentclass{article}
\usepackage{ijcai25}

\usepackage{times}
\usepackage{soul}
\usepackage{url}
\usepackage[hidelinks]{hyperref}
\usepackage[utf8]{inputenc}
\usepackage[small]{caption}
\usepackage{graphicx}
\usepackage{amsmath}
\usepackage{amsthm}
\usepackage{booktabs}
\usepackage{algorithm}
\usepackage{algorithmic}
\usepackage[switch]{lineno}

\usepackage[caption=false,font=footnotesize,labelfont=rm,textfont=rm]{subfig} 
\usepackage{amsfonts}
\usepackage{IEEEtrantools}
\usepackage{siunitx}    
\usepackage{nicefrac}   
\usepackage{makecell}   

\newtheorem{definition}{Definition}
\newtheorem{theorem}{Theorem}
\newtheorem{lemma}{Lemma}
\allowdisplaybreaks[3] 

\title{Imagination-Limited Q-Learning for Offline Reinforcement Learning}

\author{
Wenhui Liu$^1$
\and
Zhijian Wu$^1$\and
Jingchao Wang$^{1}$\and
Dingjiang Huang$^1$\footnote{Corresponding author}\and
Shuigeng Zhou$^2$\\
\affiliations
$^1$East China Normal University\\
$^2$Fudan University\\
\emails
\{whliu\_14, zjwu\_97, jcwang\}@stu.ecnu.edu.cn,
djhuang@dase.ecnu.edu.cn,
sgzhou@fudan.edu.cn
}

\begin{document}

\maketitle

\begin{abstract}
Offline reinforcement learning seeks to derive improved policies entirely from historical data but often struggles with over-optimistic value estimates for out-of-distribution (OOD) actions.
This issue is typically mitigated via policy constraint or conservative value regularization methods. However, these approaches may impose overly constraints or biased value estimates, potentially limiting performance improvements. To balance exploitation and restriction, we propose an Imagination-Limited Q-learning (ILQ) method, which aims to maintain the optimism that OOD actions deserve within appropriate limits. Specifically, we utilize the dynamics model to imagine OOD action-values, and then clip the imagined values with the maximum behavior values. Such design maintains reasonable evaluation of OOD actions to the furthest extent, while avoiding its over-optimism. Theoretically, we prove the convergence of the proposed ILQ under tabular Markov decision processes. Particularly, we demonstrate that the error bound between estimated values and optimality values of OOD state-actions possesses the same magnitude as that of in-distribution ones, thereby indicating that the bias in value estimates is effectively mitigated. Empirically, our method achieves state-of-the-art performance on a wide range of tasks in the D4RL benchmark.
\end{abstract}

\section{Introduction}

Offline Reinforcement Learning (RL)~\cite{lange2012batchRL,fujimoto2019offRL} is designed to learn optimal policies purely from a static dataset previously collected by an unknown policy (behavior policy). By eliminating the need for online interaction with environments, it offers dual benefits. On the one hand, it can mitigate the expensive costs~\cite{gu2017robotic} and potential risks \cite{sallab2017autodriving} associated with trial-and-error learning in real-world applications; on the other hand, it can be leveraged to enhance the generalization ability and scalability of RL models when the logged data is massive and diverse. However, the offline learning paradigm unavoidably incurs distributional shifts \cite{levine2020offlineRL} of state-action visitation between the learned policy and the behavior policy, which makes it difficult to correctly assess out-of-distribution (OOD) action-values. Especially, over-optimistic estimates may even invalidate the learned policy.

To address this challenge, two main classes of technical routes are commonly employed in the model-free approaches \cite{prudencio2023OfflineRL}. 1) Policy constraint: It usually explicitly restricts the gap between the learned and behavior policy. The batch-constrained Q-learning (BCQ) \cite{fujimoto2019offRL} was devised to restrict the action space via adding perturbations on a state-conditioned behavior model. Kumar et al. \cite{kumar2019BEAR} proposed BEAR to reduce maximum mean discrepancy between the learned policy and the behavior one. Subsequently, different methods corresponding to other metrics have been proposed, such as KL divergence for BRAC \cite{wu2019BRAC} and mean squared error for TD3-BC \cite{fujimoto2021TD3-BC,srinivasan2024TD3-BST}, which similarly seeks to steer the policy closer to actions in the dataset. However, these learned models are limited to the neighborhood of the behavior policy hindering their performance, especially when the dataset is collected by poor policies. 2) Value regularization: It aims to utilize value regularizations to suppress OOD action-values. The conservative Q-learning (CQL) \cite{kumar2020CQL} was designed to penalize the expectation of OOD action-values, thus mitigating optimistic estimates outside the dataset. Kostrikov et al. \cite{kostrikov2021IQL} introduced implicit Q-learning (IQL) to estimate value function through expectile regression to implicitly depress OOD action-values. 
The MCQ \cite{lyu2022MCQ} was proposed to regularize OOD action-values with the maximum behavior value.
And the OAC-BVR \cite{huang2024OAC-BVR} was developed to regard the difference between the Q-function and the behavior value as a regularization term.
While these methods effectively limit OOD optimism, they also introduce uncontrollable bias into Q-value estimates, as illustrated in Fig. \ref{fig:CQL_and_ILQ}\subref{fig:value_reg_illustration}. Especially, the Q-values under CQL exhibit noticeable bias of pessimism over all MuJoCo tasks, as shown in Fig. \ref{fig:CQL_and_ILQ}\subref{fig:Q_value_CQL_VS_ILQ}.

\begin{figure*}[!tb]
    \centering
    {
        \includegraphics[width=0.45\textwidth]{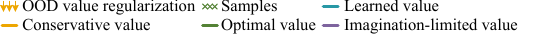}
    }
    {
        \includegraphics[width=0.45\textwidth]{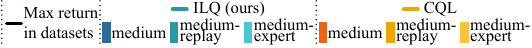}
    }
    \\
    [-2mm]
    \subfloat[ ]{
        \includegraphics[width=0.23\textwidth]{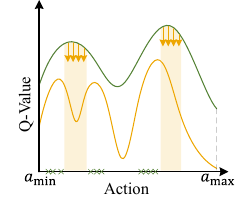}
        \label{fig:value_reg_illustration}
    }
    \subfloat[ ]{
        \includegraphics[width=0.23\textwidth]{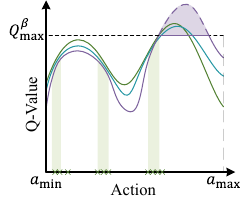}
        \label{fig:ILQ_illustration}
    }
    \subfloat[ ]{
        \includegraphics[width=0.22\textwidth]{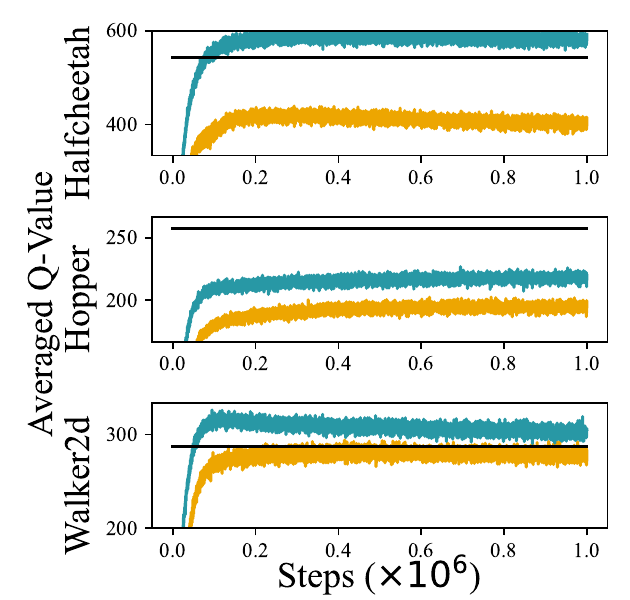}
        \label{fig:Q_value_CQL_VS_ILQ}
    }
    \subfloat[ ]{
        \includegraphics[width=0.22\textwidth]{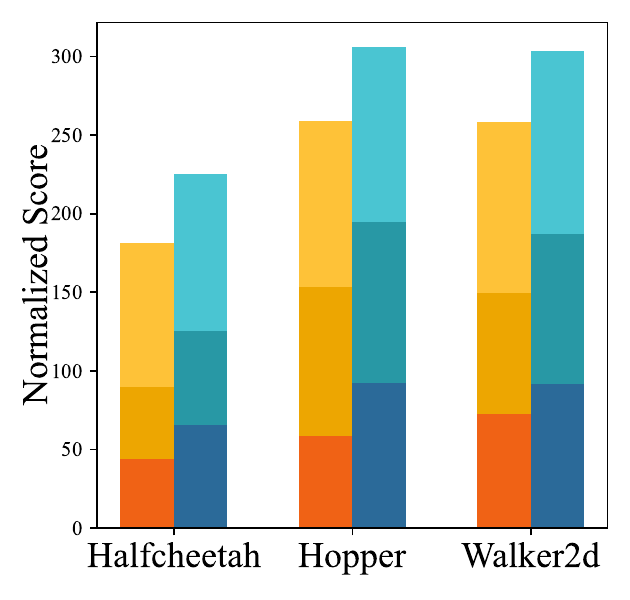}
        \label{fig:score_CQL_VS_ILQ}
    }
    \caption{(a) illustrates the fundamental principle of value regularization methods. While effectively suppressing OOD action-values, it may introduce uncontrolled bias in estimations. In contrast, instead of indiscriminately suppressing OOD action-values, ILQ, depicted in (b), envisions reasonable values (purple line) and then appropriately limits potential over-estimations using the maximum behavior value $Q^{\beta}_{\rm max}$ (black dashed line), resulting in more appropriate policy evaluation (cyan line). (c) demonstrates that Q-value estimations of CQL across MuJoCo ``-v2" tasks are notably compromised, falling well below maximum returns (black line) in datasets. Conversely, ILQ maintains reasonably optimistic Q-value estimations in anticipation of superior policies. Finally, (d) shows that ILQ's ultimate performance is significantly improved, particularly in medium tasks.}
    \label{fig:CQL_and_ILQ}
\end{figure*}

To mitigate value bias while maintaining appropriate restrictions on over-optimism,
we propose an Imagination-Limited Q-learning (ILQ) method. The insight of our method is straightforward: For in-sample state-action pairs, we adopt the standard Bellman backup based on in-sample transitions to estimate their values. For OOD state-action pairs, instead of blindly applying value regularizations to suppress their action-values, we envision \textit{what the values would be without any restrictions}. The one-step bootstrapped values under an imaged dynamics model would be reasonable estimations, ideally approximating the ground truth when the imaged model closely matches the environment. However, errors in the model fitting are inevitable, and optimistic estimates may still exist. 
Therefore, we need to further consider \textit{how to appropriately limit the imagination values}. We employ the maximum behavior value as the ceiling of the imagined one. Specifically, if the imagined value is less than the maximum value, it is maintained; otherwise, the maximum behavior value is applied. Figure \ref{fig:CQL_and_ILQ}\subref{fig:ILQ_illustration} illustrates the intuition behind our method. This design ensures a more reasonable evaluation with appropriate constraints on OOD actions, thereby avoiding unnecessary value suppression and improving the generalization ability of the learned policy.

We validate the effectiveness of ILQ both theoretically and empirically. We prove that the policy evaluation Bellman operator of ILQ is a contraction operator under tabular Markov Decision Processes (MDPs), ensuring convergence. Particularly, we analyze the action-value gap between the fixed point obtained by our method and that obtained by the Bellman optimality equation, demonstrating that the error bound of OOD action-values can reach the same order of magnitude as in-distribution ones. Empirically, our method maintains reasonably optimistic Q-values compared to conservative Q-learning (Fig. \ref{fig:CQL_and_ILQ}\subref{fig:Q_value_CQL_VS_ILQ}) and achieves state-of-the-art performance across a wide range of tasks in the standard benchmark.

\section{Background}

Reinforcement Learning (RL) is commonly modeled as a Markov Decision Process (MDP), characterized by a tuple $(\mathcal{S}, \mathcal{A}, r, P, \rho_0, \gamma)$ \cite{sutton2018RL}, where $\mathcal{S}$ is state space, $\mathcal{A}$ is action space, $r: \mathcal{S} \times \mathcal{A} \to [-r_{\rm max}, r_{\rm max}]$ is reward function, $P: \mathcal{S} \times \mathcal{A} \times \mathcal{S} \to [0, 1]$ is transition dynamics, $\rho_0$ is probability distribution of initial states, and $\gamma \in [0,1)$ is discount factor. The RL agent takes actions on the environment according to its policy, defined as $\pi: \mathcal{S} \times \mathcal{A} \to [0,1]$.

During the learning process, evaluating the expected return of a state-action pair $(s, a)$ under a policy $\pi$, called the action-value (or Q-value) $Q(s, a)$, is essential. 
Theoretically, it can be obtained by $\mathbb{E}_{\pi}[ \sum_{t=0}^{\infty} \gamma^t r(s_t, a_t) \mid  s_0=s, a_0=a ]$. In practice, it is usually approximated by minimizing the Bellman residual $\mathbb{E} [ (Q(s,a) - (\mathcal{T}Q)(s,a) )^2 ]$, where $\mathcal{T}$ is the Bellman optimality operator defined as
\begin{equation}\label{eq:Bellman_opt_operator}
    (\mathcal{T} Q) (s,a) := r(s,a) + \gamma \mathbb{E}_{s^\prime \sim P(\cdot \mid s,a)} \left[ \max_{a^\prime \sim \pi} Q(s^\prime, a^\prime) \right].
\end{equation}
Generally, a delayed approximator $Q^-$ is applied in the above target $\mathcal{T} Q$ for training stability \cite{mnih2015DQNNature}.

\subsection{Offline Reinforcement Learning}

In offline RL, the agent is no longer allowed to interact with the environment and learns policies exclusively from a limited static dataset $\mathcal{D}$, which is typically gathered from an unknown behavior policy $\beta$. The dataset $\mathcal{D}$ is usually represented by a set of transition tuples $\{ (s,a,r,s^\prime) \}$. 
The goal of the agent remains to maximize the expectation of cumulative rewards. 
Without online correction, erroneously optimistic estimates of OOD action-values are inevitable \cite{levine2020offlineRL}. These biases can cause the learned policy to favor incorrect OOD actions, potentially leading to it failure. 

The representative value regularization method is CQL \cite{kumar2020CQL}, which adds penalties for OOD action-values to the standard Q-value update objective, as follows:
\begin{IEEEeqnarray*}{rl}
    \alpha_{\rm CQL} \Bigl( \mathbb{E}_{s \sim \mathcal{D}, a \sim \mu(\cdot \mid s)} &[Q(s,a)] - \mathbb{E}_{s \sim \mathcal{D}, a \sim \beta(\cdot \mid s)} [Q(s,a)] \Bigr) \\
    & + \frac{1}{2} \mathbb{E} \Bigl[ \bigl( Q(s,a)- \mathcal{T}Q(s,a) \bigr)^2 \Bigr],
\end{IEEEeqnarray*}
where $\mu (\cdot | s)$ is a distribution to produce OOD actions, $\alpha_{\rm CQL}$ is a hyperparameter to adjust the degree of conservatism.

\section{Related Work}

\subsection{Model-free Offline RL}

Recent advancements in offline RL have focused on addressing the challenges posed by OOD actions. Importance sampling methods \cite{Nachum2019DualDICE} have been developed to correct evaluation under distributional shifts, but they often suffer from high variance. Policy constraint methods \cite{kumar2019BEAR,wu2021UWAC,fujimoto2021TD3-BC,li2023DOGE,chen2024DTQL} are employed to limit the deviation of learned policies; however, their performance tends to degrade when the behavior policy is suboptimal.

Our method aligns more closely with value regularization approaches. For instance, CQL \cite{kumar2020CQL} and CSVE \cite{chen2023CSVE} directly penalize OOD action-values, effectively controlling over-optimism but often resulting in overly pessimistic estimates. Methods like MCQ \cite{lyu2022MCQ} and OAC-BVR \cite{huang2024OAC-BVR} attempt to relax restrictions by assigning maximum behavior values or behavior values, respectively, to OOD action values. However, this introduces uncontrolled value bias for OOD actions. Although MCQ additionally employs policy constraint weighting to mitigate this bias in practical implementations, it offers limited theoretical guarantees for OOD action-value estimates. 
In contrast, our proposed method preserves more honest estimates of OOD action-values within the limitation of maximum behavior values, and offers theoretical guarantees for its value estimates. 

\subsection{Model-based Offline RL}

Model-based methods aim to enhance collected datasets by generating synthetic trajectories using learned dynamics models. Various strategies have been proposed to effectively leverage these synthetic data, including uncertainty quantification \cite{Ovadia2019ModelUncertainty,Kidambi2020MOReL,yu2020MOPO,diehl2021umbrella}, conservative value estimation \cite{yu2021combo}, representation balancing \cite{lee2021repbal_model}, and adversarial learning \cite{Bhardwaj2023ARMOR}. In contrast, our proposed ILQ avoids trajectories generation, relying solely on the dynamics model to produce one-step subsequent states and rewards of in-sample states for estimating imagined OOD action-values. While ILQ utilizes the dynamics model, it remains fundamentally a model-free learning framework and circumvents challenges of error accumulation associated with longer trajectories in model-based methods.

\section{Imagination-Limited Q-learning Method}

We start by elucidating our novel Imagination-Limited Bellman (ILB) operator in Subsection \ref{sec:ilq-approx-free}.
Subsequently, we elaborate on its practical implementation details and theoretical analysis in Subsection \ref{sec:ilq-diffusion} and \ref{sec:theory}, respectively. And the Imagination-Limited Q-learning (ILQ) algorithm is ultimately summarized in Subsection \ref{sec:ilq-practical}.

\subsection{Imagination-Limited Bellman Operator} \label{sec:ilq-approx-free}

In online RL, researchers typically use the Bellman optimality operator Eq. \eqref{eq:Bellman_opt_operator} to evaluate policies. However, in offline settings, the absence of online corrections makes policy evaluation highly susceptible to OOD over-optimism \cite{levine2020offlineRL} under the standard operator. Existing value regularization methods primarily focus on directly restricting the OOD action-values \cite{kumar2020CQL,lyu2022MCQ,chen2023CSVE,huang2024OAC-BVR}, which introduces uncontrollable bias in value estimates and lacks theoretical guarantees for OOD actions.

We argue that establishing reasonable estimates for out-of-distribution state-actions should take precedence, followed by the imposition of suitable restrictions, rather than directly employing value regularization. To achieve this goal, we introduce a novel Imagination-Limited Bellman operator, defined as follows.
\begin{definition}\label{def:operator}
The Imagination-Limited Bellman (ILB) operator is defined as 

\begin{IEEEeqnarray*}{rl} 
    \mathcal{T}_{\mathrm{ILB}} & Q(s,a) \\
    & = \begin{cases}
r(s,a) + \gamma \mathbb{E}_{s^\prime \sim P} \bigl[ \underset{{\tilde{a}^\prime \sim \pi }}{\max} Q(s^{\prime},\tilde{a}^\prime) \bigr], &{\text{if}}~\beta(a|s)>0 \\ 
\min \left \{ y_{{\rm img}}^{Q}, y_{{\rm lmt}}^{Q} \right \} + \delta, &{\text{otherwise.}} 
\end{cases} \IEEEeqnarraynumspace \IEEEyesnumber\label{eq:def-ILB-operator}
\end{IEEEeqnarray*}
where $\beta$ is the behavior policy, 
\begin{equation} \label{eq:def-img-value}
    y_{{\rm img}}^{Q} = \widehat{r} (s, a) 
    + \gamma \mathbb{E}_{ {\widehat{s}^{\prime} \sim \widehat{P}(\cdot \mid s, a)} } \left[ \max_{\tilde{a}^{\prime} \sim  \pi } Q (\widehat{s}^{\prime}, \tilde{a}^{\prime}) \right],
\end{equation}
and
\begin{equation} \label{eq:def-lmt-value}
    y_{{\rm lmt}}^{Q} = \max_{{\widehat{a} \in {\rm Supp}(\beta(\cdot \mid s))}} Q (s, \widehat{a})
\end{equation}
are the imagined value and its limitation, respectively. The $\widehat{P}$ is the empirical transition kernel, $\widehat{r}$ is the empirical reward function, $\delta$ is a hyperparameter with a small absolute value, and ${\rm Supp} (\cdot)$ means support-constrained on the dataset.
\end{definition}

Here is the insight behind the proposed ILB operator. For an in-sample state-action pair $(s, a)$, i.e., $\beta(a \mid s)>0$, we have its corresponding transition $(s, a, r, s^\prime)$ in the dataset, allowing us to apply the standard Bellman operator without any obstacles. However, for an out-of-sample state-action pair $(s, a^{\rm oos})$, the standard Bellman backup cannot be applied solely due to the absence of its successor state and reward. To address this, one could utilize empirical dynamics model to predict the next state $\widehat{s}^\prime$ and reward $\widehat{r}$ and obtain the imagination value $y_{\rm img}^Q$ as Eq. \eqref{eq:def-img-value}, which provides a relatively accurate approximation for an OOD state-action. Nevertheless, it may still result in optimistic estimates because of fitting errors. To tackle this issue, we use the maximum in-distribution action-value Eq. \eqref{eq:def-lmt-value} as the upper limit for the imagined value. This design offers dual benefits: First, it maximally maintains the imagined value, reducing estimation bias on OOD actions; second, the maximum behavior value ensures that there is always an in-distribution action-value greater than or equal to the OOD ones, encouraging the policy to more likely favor in-distribution actions during the actor improvement process.

We analyze the ILB operator's properties and demonstrate that the ILB operator exhibits the same $\gamma$-contraction property as the standard Bellman operator, ensuring convergence in policy evaluation. The proof is provided in the Appendix.

\begin{theorem}[\textbf{Convergence}]\label{thm:convergence}
    The ILB operator defined in Eq. \eqref{eq:def-ILB-operator} is a $\gamma$-contraction operator in the $\mathcal{L}_{\infty}$ norm, and Q-function iteration rule obeying the ILB operator can converge to a unique fixed point.
\end{theorem}

\subsection{Practical Implementation of Imagination and Limitation Value}

\subsubsection{Imagination Value}

We fit the environment dynamics to derive empirical transfer kernel $\widehat{P}$ and reward function $\widehat{r}$ in the simplest manner using
\begin{equation} \label{eq:opt_empirical_dynamics}
    \max_{\widehat{T}_{\psi}} \mathbb{E}_{(s,a,r,s^\prime) \sim \mathcal{D}} \left[ \log \widehat{T}_{\psi} (r,s^\prime \mid s,a) \right],
\end{equation}
where $\widehat{T}_{\psi}$ stands for both $\widehat{P}$ and $\widehat{r}$ for brevity, and is represented by a multivariate Gaussian distribution with parameters $\psi$ practically. We then obtain the imagined value $y_{\rm img}^Q$. 

\subsubsection{Limitation Value} \label{sec:ilq-diffusion}

In fact, the behavior policy $\beta(\cdot \mid s)$ in Eq. \eqref{eq:def-lmt-value} is unknown and needs to be empirically modeled.
In light of the expressiveness of diffusion models \cite{ho2020DDPM}, we fit the behavior policy using a conditional diffusion model. Specifically, it is constructed via a reverse diffusion chain, formulated as
\begin{equation} \label{eq:cond_diff_behavior}
    {\rm Diff}_{\omega}(a \mid s) := \mathcal{N}(a^K; 0, I) \prod^K_{k=1} p_{\omega}(a^{k-1} \mid a^k, s)
\end{equation}
where superscript $k$ denotes the diffusion timestep, $a:=a^0$ is the final sampled action, $a^{k},~ k=1, \cdots, K-1$, are latent variables, $a^K\sim \mathcal{N}(0,I)$ is Gaussian noise. Typically, $p_{\omega}(a^{k-1} \mid a^k, s)$ is modeled as a Gaussian distribution $\mathcal{N}\left( a^{k-1}; \mu_{\omega}(a^k,s,k), \Sigma_{\omega}(a^k,s,k) \right)$ with the covariance matrix $\Sigma_{\omega}(a^k,s,k) = \beta_k I$ and the mean defined as
\begin{equation}
    \mu_{\omega}(a^k,s,k) = \frac{1}{\sqrt{\alpha_k}} \left( a^k - \frac{\beta_k}{\sqrt{1-\bar{\alpha}_k}} \xi_{\omega}(a^k,s,k) \right),
\end{equation}
where $\beta_k$ is the variance schedule, $\alpha_k := 1-\beta_k$,  $\bar{\alpha}_k := \prod_{i=1}^k \alpha_i$, and $\xi_{\omega}(\cdot)$ is the noise prediction network with parameters $\omega$. The conditional diffusion model is optimized by maximizing the evidence lower bound, which can be simplified \cite{ho2020DDPM} to minimize the following objective
\begin{IEEEeqnarray}{C}\label{eq:opt_cond_diff}
    \min_{\omega} \mathbb{E}_{\substack{k \sim \mathcal{U}, \xi \sim \mathcal{N}(0,I)\\ (s,a)\sim \mathcal{D}}} \left \| \xi - \xi_{\omega} (\sqrt{\bar{\alpha}_k}a + \sqrt{1-\bar{\alpha}_k }\xi, s, k) \right \|^2 , \IEEEeqnarraynumspace
\end{IEEEeqnarray}
where $\mathcal{U}$ is an uniform distribution over $\{1, \cdots, K\}$. Similar diffusion behavior modeling methods are also applied in other works \cite{wang2022diffusionQL,hansen2023IDQL}.

Accordingly, we adopt the limitation value as shown in the following equation:
\begin{equation} \label{eq:cond-diff-lmt-value}
    y_{{\rm lmt}}^Q = \max_{\substack{\widehat{a}_m \sim {\rm Diff}_{\omega} \left( \cdot \mid s \right) \\ m = 1, \cdots, M}} Q \left( s, \widehat{a}_m \right),
\end{equation}
where $M$ is the number of sampled actions. Due to possible errors in the fitting process, this may result in a deviation between the estimated value Eq. \eqref{eq:cond-diff-lmt-value} and the true value Eq. \eqref{eq:def-lmt-value}. 
In order to offset this gap, we introduced the hyperparameter $\delta$ in definition Eq. \eqref{eq:def-ILB-operator}, typically set to a small absolute value. 

\subsection{Theoretical Analysis} \label{sec:theory}

We now theoretically discuss the action-value gap between the fixed point in Theorem \ref{thm:convergence} and the Bellman optimality value. Before proceeding further, we make some commonly used assumptions about the reward function \cite[Assumption 1]{huang2024OAC-BVR}.
\begin{enumerate}
    \item The reward function is bounded, i.e., $|r(s,a)| \leq r_{\max}$. Actually, this is consistent with what is required by its definition $r(s,a): \mathcal{S} \times \mathcal{A} \to [-r_{\rm max}, r_{\rm max}]$.
    \item Similar to the Lipschitz condition, i.e., $|r(s,\tilde{a}_1) - r(s,\tilde{a}_2)| \leq \ell \|\tilde{a}_1 - \tilde{a}_2\|_{\infty}$, $\forall s \in \mathcal{S}$ and $\forall \tilde{a}_1, \tilde{a}_2 \in \mathcal{A}$, where $\ell$ is a constant. This requires that the reward function satisfies Lipschitz continuity with respect to actions.
\end{enumerate}
In addition, the error bound assumption between the empirical models and the real ones are required, which is also utilized in both \cite{kumar2020CQL} and \cite{huang2024OAC-BVR}. Suppose the $\widehat{r}$ and $\widehat{P}$ are the empirical reward function and empirical transition dynamics, respectively, the following relationships hold with high probability $\geq 1-\zeta$, $\zeta \in (0,1)$, 
    \begin{IEEEeqnarray}{C}
        \Bigl\lVert \widehat{r}(s,a) - r(s,a) \Bigr\rVert_{1} \leq \nicefrac{\zeta_r}{\sqrt{D}}, \label{eq:emp_reward_assump}\\
        \Bigl\lVert \widehat{P}(\cdot \mid s,a) - P(\cdot \mid s,a) \Bigr\rVert_{1} \leq \nicefrac{\zeta_P}{\sqrt{D}}, \label{eq:emp_dynamics_assump}
    \end{IEEEeqnarray}
    where $D$ is the constant related to the dataset size, $\zeta_{r}$ and $\zeta_{P}$ are constants related to $\zeta$.

We begin by analyzing the Bellman optimality value gap between the learned policy and behavior policy.
\begin{theorem}\label{lem:optmality_Q_gap}
    Suppose $Q_{\beta^*}$ is the fixed point of the support-constrained Bellman optimality operator. The following gap can be obtained
    \begin{IEEEeqnarray}{C}
        \left \lvert Q_{\beta^*} (s, \pi(s)) - Q_{\beta^*} (s, \beta(s)) \right \rvert
        \leq \ell \epsilon_{\pi} + \gamma \frac{|\mathcal{S}|r_{\rm max}}{1-\gamma} \epsilon_P, \IEEEeqnarraynumspace
    \end{IEEEeqnarray}
    where $\epsilon_{\pi} := \max_{s} \left \lVert \pi(s) - \beta(s) \right \rVert_{\infty}$ and $\epsilon_{P} := \left \lVert P^{\pi} -P^{\beta} \right \rVert_{\infty}$.
\end{theorem}

Accordingly, we could prove the error bound between the imagination value and Bellman optimality value.
\begin{theorem} \label{lem:img_optimality_Q_gap}
    Suppose $Q_{\beta^*}$ is the fixed point of support-constrained Bellman optimality operator. The gap between the imagination value $y_{\rm img}^{Q_{\beta^*}}$ and $Q_{\beta^*}$ has:
    \begin{equation}
        \begin{aligned}
            &\left \lvert y^{Q_{\beta^*}}_{\rm img} - Q_{\beta^*}(s,a) \right \rvert \\
            & \leq \frac{\zeta_r}{\sqrt{D}} + \gamma \ell \epsilon_{\pi} 
            + \gamma^2 \frac{|\mathcal{S}| r_{\rm max}}{1 - \gamma} \epsilon_P + \gamma \frac{\zeta_P}{\sqrt{D}} \frac{r_{\rm max}}{1-\gamma}.
        \end{aligned}
    \end{equation}
\end{theorem}

Based on above theorems, we can estimate the action-value gap between the fixed point of the ILB operator and $Q_{\beta^*}$.

\begin{theorem}[\textbf{Action-value gap}]\label{thm:action-value-gap}
    Suppose $Q_{\rm ILB}$ and $Q_{\beta^*}$ denote the fixed point of the ILB operator and support-constrained Bellman optimality operator, separately. The action-value gap can be bounded as 
    \begin{equation}
        \begin{aligned}
        &\left\| Q_{\rm ILB} (s,a) - Q_{\beta^*} (s,a) \right\|_{\infty}\\
            &\leq \frac{1}{1-\gamma} \frac{\zeta_r}{\sqrt{D}} + \frac{\ell}{1-\gamma} \epsilon_{\pi} \\
            &\phantom{=\;} + \frac{\gamma |\mathcal{S}| r_{\rm max}}{(1-\gamma)^2} \epsilon_P + \frac{\gamma r_{\rm max}}{(1-\gamma)^2} \frac{\zeta_P}{\sqrt{D}} + \frac{1}{1-\gamma} \lvert \delta \rvert,
        \end{aligned}
    \end{equation}
    where $\zeta_r,~\zeta_P$ are defined in Eq. \eqref{eq:emp_reward_assump} and Eq. \eqref{eq:emp_dynamics_assump}, $\epsilon_r,~\epsilon_P$ are defined in Theorem \ref{lem:optmality_Q_gap}, $\delta$ is defined in the ILB operator.
\end{theorem}

According to Theorem \ref{thm:action-value-gap}, we conclude that error bounds for in-sample and out-of-sample actions are of the same magnitude $\mathcal{O}(\nicefrac{r_{\rm max}}{(1-\gamma)^2})$. This result aligns with the conclusion of CQL \cite{kumar2020CQL} within the support region. All proofs are detailed in the Appendix.

\subsection{The ILQ Algorithm} \label{sec:ilq-practical}

In deep RL, the Q-function is commonly approximated by a neural network with parameters $\theta$, while the corresponding target network has parameters $\theta^-$. As described in the background, it can be optimized by minimizing the temporal difference (TD) loss $\mathbb{E}_{(s,a,r,s^\prime)}[(Q_{\theta} (s,a) - \mathcal{T}Q_{\theta^-} (s,a))^2]$. Intuitively, under the ILB operator we developed, the corresponding loss function can be constructed as 
\begin{equation} \label{eq:critic_loss_1}
    \mathbb{E}_{(s,a,r,s^\prime)} \Bigl[ \bigl( Q_{\theta} (s,a) - \mathcal{T}_{\rm ILB}Q_{\theta^-} (s,a) \bigr)^2 \Bigr].
\end{equation}
Since in-sample and out-of-sample state-action pairs have different TD targets, as defined by our ILB operator, a common way is to split the above loss function into a weighted sum of in-sample and out-of-sample components as follows
\begin{equation} \label{eq:critic_loss_2}
    \begin{aligned}
        & \mathcal{L}_{Q} (\theta) = \eta \mathbb{E}_{(s,a,r,s^\prime) \sim \mathcal{D}} \Bigl[ \bigl( Q_{\theta} (s,a) - \mathcal{T}_{\rm ILB}Q_{\theta^-} (s,a) \bigr)^2 \Bigr] \\
    & \phantom{+} + (1-\eta) \mathbb{E}_{\substack{s \sim \mathcal{D} \\ a^{\rm oos} \sim u(\cdot \mid s)}} \Bigl[ \bigl( Q_{\theta} (s,a^{\rm oos}) - \mathcal{T}_{\rm ILB}Q_{\theta^-} (s,a^{\rm oos}) \bigr)^2 \Bigr],
    \end{aligned}
\end{equation}
where $\eta$ is a trade-off factor, $(s,a^{\rm oos})$ refers to the out-of-sample state-action pair. The $u(\cdot \mid s)$ is any distribution that can produce out-of-sample actions. In practice, we directly treat policy $\pi$ as the sampling distribution $u$.

We also incorporate the double network \cite{hasselt2010doubleQ} to reduce overestimation, which is widely utilized in both online \cite{mnih2015DQNNature} and offline RL \cite{fujimoto2021TD3-BC}. 
Therefore, combining the definition of ILB operator Eq. \eqref{eq:def-ILB-operator}, we have the target, for in-sample transition $(s,a,r,s^\prime)$, as:
\begin{equation} \label{eq:in-sample-target}
    \mathcal{T}_{\rm ILB}Q_{\theta^-} (s,a) = r(s, a) + \gamma \min_{j=1,2} \mathbb{E}_{\tilde{a}^{\prime} \sim \pi_{\phi} (\cdot \mid s^{\prime})} \left[ Q_{\theta_j^-} (s^{\prime}, \tilde{a}^{\prime}) \right],
\end{equation}
where $\pi_{\phi}$ is the learned policy represented by a neural network with parameters $\phi$ and $Q_{\theta^-_j}$ is the $j$-th target network. Similarly, for out-of-sample $(s,a^{\rm oos})$, the target is formulated as:
\begin{equation}\label{eq:OOD-sample-target}
        \mathcal{T}_{\rm ILB}Q_{\theta^-} (s,a^{\rm oos}) = \min \left\{ y_{{\rm img}}^Q , y_{{\rm lmt}}^Q  \right\} + \delta,
\end{equation}
where
\begin{IEEEeqnarray}{C} \label{eq:OOD-sample-target-y2}
    y_{{\rm img}}^Q = \widehat{r} \left( s, a^{{\rm oos}} \right) 
    + \gamma \min_{j=1,2} \mathbb{E}_{ \substack{\widehat{s}^{\prime} \sim \widehat{P} \left( \cdot \mid s, a^{{\rm oos}} \right) \\ \tilde{a}^{\prime} \sim \pi_{\phi} (\cdot \mid \widehat{s}^{\prime})} } \left[ Q_{\theta_j^-} (\widehat{s}^{\prime}, \tilde{a}^{\prime}) \right], \IEEEeqnarraynumspace
\end{IEEEeqnarray}
\begin{IEEEeqnarray}{C} \label{eq:OOD-sample-target-y1}
    y_{{\rm lmt}}^Q = \min_{j=1,2} \max_{\substack{\widehat{a}_m \sim {\rm Diff}_{\omega} \left( \cdot \mid s \right) \\ m = 1, \cdots, M}} Q_{\theta_j^-} \left( s, \widehat{a}_m \right). \IEEEeqnarraynumspace
\end{IEEEeqnarray}
Here Eq. \eqref{eq:OOD-sample-target-y1} is obtained by coupling Eq. \eqref{eq:cond-diff-lmt-value}.

During policy improvement, we adopt the same objective as in vanilla SAC \cite{haarnoja2018SAC} to optimize the actor network without any complex design, as follows:
\begin{IEEEeqnarray}{C} \label{eq:opt-actor}
    \max_{\phi} \mathbb{E}_{s \sim \mathcal{D}, a \sim \pi_{\phi}(\cdot \mid s)} \left[ \min_{j=1,2} Q_{\theta_j} (s, a) - \alpha \log \pi_{\phi} (a \mid s) \right], \IEEEeqnarraynumspace
\end{IEEEeqnarray}
where $\alpha$ is a multiplier for the entropy.

Combining above steps, our method is derived, with its pseudo-code presented in Algorithm \ref{alg}.

\begin{algorithm}[!t]
\caption{Imagination-Limited Q-Learning (ILQ)} \label{alg}
\begin{algorithmic}[1]
\REQUIRE The offline dataset $\mathcal{D}$, number of iterations $N$, discount factor $\gamma$, target network update rate $\tau$, trade-off factor $\eta$, and offset parameter $\delta$. 
\STATE Initialize critic networks $Q_{\theta_1}$, $Q_{\theta_2}$, actor network $\pi_{\phi}$, target networks $Q_{\theta_1^-}$, $Q_{\theta_2^-}$ with $\theta_i^- \leftarrow \theta_i, i=\{1, 2\}$.
\STATE // Pre-train the dynamics model and behavior policy
\STATE Train the dynamics model $ \widehat{T}_{\psi} (s^{\prime}, r \mid s, a) $ via \eqref{eq:opt_empirical_dynamics}.
\STATE Train the diffusion model $\text{Diff}_{\omega} (\cdot \mid s)$ for modeling the behavior policy by optimizing \eqref{eq:opt_cond_diff}.
\STATE // Policy training
\FOR {step $n = 1$ to $N$}
\STATE Sample a mini-batch of transitions $\mathcal{B} =\{ (s, a, r, {s}^{\prime}) \}$ from dataset $\mathcal{D}$.
\STATE Compute target value for $(s,a)$ in $\mathcal{B}$ as \eqref{eq:in-sample-target}.
\STATE Sample OOD actions conditioned on sates in $\mathcal{B}$ via $\pi_{\phi}$, and calculate target value according to \eqref{eq:OOD-sample-target} via the pre-trained behavior policy and dynamics model.

\STATE Update parameters $\theta_i, i=1,2$ for each critic network via minimizing \eqref{eq:critic_loss_2}.

\STATE Update actor $\phi$ via \eqref{eq:opt-actor}.

\STATE Update the target networks
\[
\theta_i^- \leftarrow \tau \theta_i + (1-\tau) \theta_i^-, i=1,2.
\]

\ENDFOR
\end{algorithmic} 
\end{algorithm}

\section{Experiments}

In this section, we empirically validate the effectiveness of our method ILQ. 1) We demonstrate the superiority of ILQ over existing methods by comparing performance across a series of tasks. 2) We conduct sensitivity analyses on the hyperparameters involved in ILQ, confirming the stability of the proposed method. 3) We then perform ablation experiments on both imagination and limitation components to verify their impacts. 4) We also delve into the Q-value estimations to further validate the effectiveness of the two components designed; details are in the Appendix due to space constraints.

\subsection{Experimental Settings}
We evaluate ILQ on the D4RL \cite{fu2020D4RL} benchmark. The commonly used domain is Gym MuJoCo ``-v2", including halfcheetah, hopper, and walker2d tasks at four levels: random (r), medium (m), medium-replay (mr), and medium-expert (me). We also assess ILQ on Maze2D ``-v1" domain, which offers three layouts with two reward types, i.e., umaze (u), umaze-dense (ud), medium (m), medium-dense (md), large (l), and large-dense (ld). In addition, comparisons on several adroit ``-v0" tasks are conducted. Due to page limitations, descriptions of compared algorithms and extra experimental results have been relocated to the Appendix. Details of hyperparameters settings and implementation specifics are also provided in the Appendix to ensure reproducibility.

\subsection{Performance Comparison}

\begin{table*}[htb]
  \caption{Comparison of normalized average scores for ILQ and existing state-of-the-art methods on MuJoCo tasks over the final 10 evaluations. Experiments are conducted using $5$ different random seeds. The highest score is \textbf{bolded}.}
  \label{tab:score-mujoco}
  \small
  \centering
  \resizebox{\textwidth}{!}{
  \setlength{\tabcolsep}{3.5pt}
  \sisetup{text-series-to-math}
  \begin{tabular}{@{}l S@{} S@{} S@{} S@{} S@{} S@{} S@{} S@{} S@{} S@{} S@{} S@{} S@{} S[table-format = 2.2(1), separate-uncertainty] @{}}
    \toprule
    \multicolumn{1}{c}{Task Name} & \multicolumn{1}{c}{BC} & \multicolumn{1}{c}{BCQ} & \multicolumn{1}{c}{CQL} & \multicolumn{1}{c}{UWAC} & \multicolumn{1}{c}{One-step} & \multicolumn{1}{c}{TD3+BC} & \multicolumn{1}{c}{IQL} & \multicolumn{1}{c}{MCQ} & \multicolumn{1}{c}{CSVE} & \multicolumn{1}{c}{OAP} & \multicolumn{1}{c}{DTQL} & \multicolumn{1}{c}{OAC-BVR} & \multicolumn{1}{c}{TD3-BST} & \multicolumn{1}{c}{ILQ(Ours)} \\
    \midrule
    halfcheetah-r   & 2.2  & 2.2   & 17.5  & 2.3  & 3.7   & 11.0  & 13.1  & 28.5            & 26.7       & 24.0  &\mbox{-}        & 31.1  & \mbox{-}  & \bfseries 31.7(7) \\
    hopper-r        & 3.7  & 7.8   & 7.9   & 2.7  & 5.6   & 8.5   & 7.9   & \bfseries 31.8  & 27.0       & 8.8   &\mbox{-}        & 7.4   & \mbox{-}  & 31.6(2) \\
    walker2d-r      & 0.2  & 4.9   & 5.1   & 2.0  & 5.2   & 1.6   & 5.4   & 17.0            & 6.1        & 5.1   &\mbox{-}        & 9.8   & \mbox{-}   & \bfseries 21.6(1) \\
    halfcheetah-m   & 42.4 & 46.6  & 44.0  & 42.2 & 48.6  & 48.3  & 47.4  & 64.3            & 48.6       & 56.4  & 57.9           & 52.2  & 62.1  & \bfseries 65.7(5) \\
    hopper-m        & 53.4 & 59.4  & 58.5  & 50.9 & 56.7  & 59.3  & 66.3  & 78.4            & 99.4       & 82.0  & 99.6           & 95.0  &\bfseries 102.9  & 92.1(58) \\
    walker2d-m      & 66.9 & 71.8  & 72.5  & 75.4 & 80.3  & 83.7  & 78.3  & 91.0            & 82.5       & 85.6  & 89.4           & 86.0  & 90.7  & \bfseries 91.5(7) \\
    halfcheetah-mr  & 34.9 & 42.2  & 45.5  & 35.9 & 38.6  & 44.6  & 44.2  & 56.8            & 54.8       & 53.4  & 50.9           & 48.3  & 53.0  & \bfseries 59.6(10) \\
    hopper-mr       & 28.1 & 60.9  & 95.0  & 25.3 & 94.1  & 60.9  & 94.7  & 101.6           & 91.7       & 98.5  & 100.0           & 95.3  & 101.2  & \bfseries 102.7(3) \\
    walker2d-mr     & 19.2 & 57.0  & 77.2  & 23.6 & 49.3  & 81.8  & 73.9  & 91.3            & 78.5       & 84.3  & 88.5           & 77.3  & 90.4  & \bfseries 95.3(18) \\
    halfcheetah-me  & 60.4 & 95.4  & 91.6  & 42.7 & 91.7  & 90.7  & 86.7  & 87.5            & 93.1       & 83.4  & 92.7           & 93.1  &\bfseries 100.7  & \bfseries 100.0(4) \\
    hopper-me       & 51.5 & 106.9 & 105.4 & 44.9 & 83.1  & 98.0  & 91.5  & 111.2           & 95.2       & 85.9  & 109.3          & 96.5  & 110.3 & \bfseries 111.6(6) \\
    walker2d-me     & 96.7 & 107.7 & 108.8 & 96.5 & 112.9 & 110.1 & 109.6 & 114.2           & 109.0      & 111.1 & 110.0          & 112.0 & 109.4 & \bfseries 117.0(12) \\
    \midrule
    MuJoCo\ total   & 459.6 & 662.8 & 729.0 & 444.4 & 669.8 & 698.5 & 719.0 & 873.6 & 812.6 & 778.5 &\mbox{-} & 804.0 & \mbox{-} & \bfseries 920.4 \\
    \bottomrule
  \end{tabular}
  }
  \vskip -0.1in
\end{table*}

In comparing ILQ with state-of-the-art methods on MuJoCo tasks, as shown in Table \ref{tab:score-mujoco}, 
ILQ performs significantly better than policy constraint methods (BCQ \cite{fujimoto2019offRL}, UWAC \cite{wu2021UWAC}, One-step \cite{Brandfonbrener2021onestep}, TD3+BC \cite{fujimoto2021TD3-BC}, and OAP \cite{yang2023OAP}) in a wide range of random- and medium-level tasks. This is because policy constraint methods restrict the learned policy to be within a neighborhood of the behavior policy. Although TD3-BST \cite{srinivasan2024TD3-BST} introduces fine-grained weighting on the constraints of TD3+BC to enhance performance, it remains inferior to ILQ overall. The value regularization methods (CQL \cite{kumar2020CQL}, MCQ \cite{lyu2022MCQ}, CSVE \cite{chen2023CSVE}, OAC-BVR \cite{huang2024OAC-BVR}) show higher scores in average. ILQ continues to show performance beyond them in $\nicefrac{11}{12}$ tasks. DTQL \cite{chen2024DTQL} integrates implicit value regularization and policy constraints to enhance performance, but it still falls short compared to ILQ, especially in halfcheetah tasks.

According to Table \ref{tab:score-maze2d}, neither policy constraint approaches nor value regularization approaches performed well enough on Maze2D tasks. Although ROMI-BCQ \cite{wang2021ROMI} achieves advanced performance on maze2d-u by utilizing a reverse dynamics model, it performs mediocrely on other tasks. Diffuser \cite{janner2022Diffuser} and PlanCP \cite{sun2024PlanCP} leverage diffusion models to improve their planning capabilities in maze2d-u, while still lag behind ILQ on maze2d-m and maze2d-l, further demonstrating stitching abilities of ILQ.
The experimental results on Adroit tasks demonstrate that our method continues to outperform others, highlighting its applicability across different tasks.

\begin{table}[tb]
  \caption{Comparison of normalized scores on Maze2D and Aroit datasets. The scores are also averaged over the final 10 evaluations across 5 different random seeds. }
  \label{tab:score-maze2d}
  \small
  \centering
  \resizebox{\linewidth}{!}{
  \setlength{\tabcolsep}{3.5pt}
  \sisetup{text-series-to-math}
  \begin{tabular}{@{}l S@{} S@{} S@{} S@{} S@{} S@{} S@{} S[table-format = 2.2(1), separate-uncertainty, table-align-uncertainty = false] @{}}
    \toprule
    \multicolumn{1}{c}{Task Name} & \multicolumn{1}{c}{ROMI-BCQ} & \multicolumn{1}{c}{BEAR} & \multicolumn{1}{c}{CQL} & \multicolumn{1}{c}{IQL} & \multicolumn{1}{c}{MCQ} & \multicolumn{1}{c}{Diffuser} & \multicolumn{1}{c}{PlanCP} & \multicolumn{1}{c}{ILQ(Ours)} \\
    \midrule
    maze2d-u    & \bfseries 139.5 & 65.7  & 18.9 & 47.4       & 81.5       & 113.9      & 116.4      & 91.9(260) \\
    maze2d-ud   & 98.3            & 32.6  & 14.4 & 48.9 & 107.8      & \mbox{-} & \mbox{-} & \bfseries 116.2(154) \\
    maze2d-m    & 82.4            & 25.0  & 14.6 & 34.9       & 106.8      & 121.5      & 128.5      & \bfseries 163.6(314) \\
    maze2d-md   & 102.6           & 19.1  & 30.5 & 47.1 & 112.7      & \mbox{-} & \mbox{-} & \bfseries 137.8(92) \\
    maze2d-l    & 83.1            & 81.0  & 16.0 & 58.6       & 111.2 & 123.0      & 130.9      & \bfseries 198.5(238) \\
    maze2d-ld   & 124.0           & 133.8 & 46.9 & 75.4 & 118.5 & \mbox{-} & \mbox{-} & \bfseries 152.8(104) \\
    \midrule
    Maze2D\ total & 629.9 & 357.2 & 141.3 & 312.3 & 638.5 & \mbox{-} & \mbox{-} & \bfseries 860.8 \\
    \midrule
    \multicolumn{1}{c}{Task Name} & \multicolumn{1}{c}{BCQ} & \multicolumn{1}{c}{TD3+BC} & \multicolumn{1}{c}{CQL} & \multicolumn{1}{c}{IQL} & \multicolumn{1}{c}{MCQ} & \multicolumn{1}{c}{DQL} & \multicolumn{1}{c}{DTQL} & \multicolumn{1}{c}{ILQ(Ours)} \\
        \midrule
         pen-human & 68.9 & 64.8 & 35.2 & 71.5 & 68.5 & 72.8 & 64.1 & \bfseries 77.3(79)  \\
         pen-cloned & 44.4 & 49 & 27.2 & 37.3 & 49.4 & 57.3 & 81.3 & \bfseries 85.6(102) \\
         \midrule
    Adroit\ total & 113.3 & 113.8 & 62.4 & 108.8 & 117.9 & 130.1 & 145.4 & \bfseries 162.9 \\
        \bottomrule
  \end{tabular}
  }
  \vskip -0.15in
\end{table}

\subsection{Parameter Study}
\subsubsection{Offset parameter $\delta$}

\begin{figure}[!tb]
    \centering
    \subfloat[ ]{
        \includegraphics[width=0.18\textwidth]{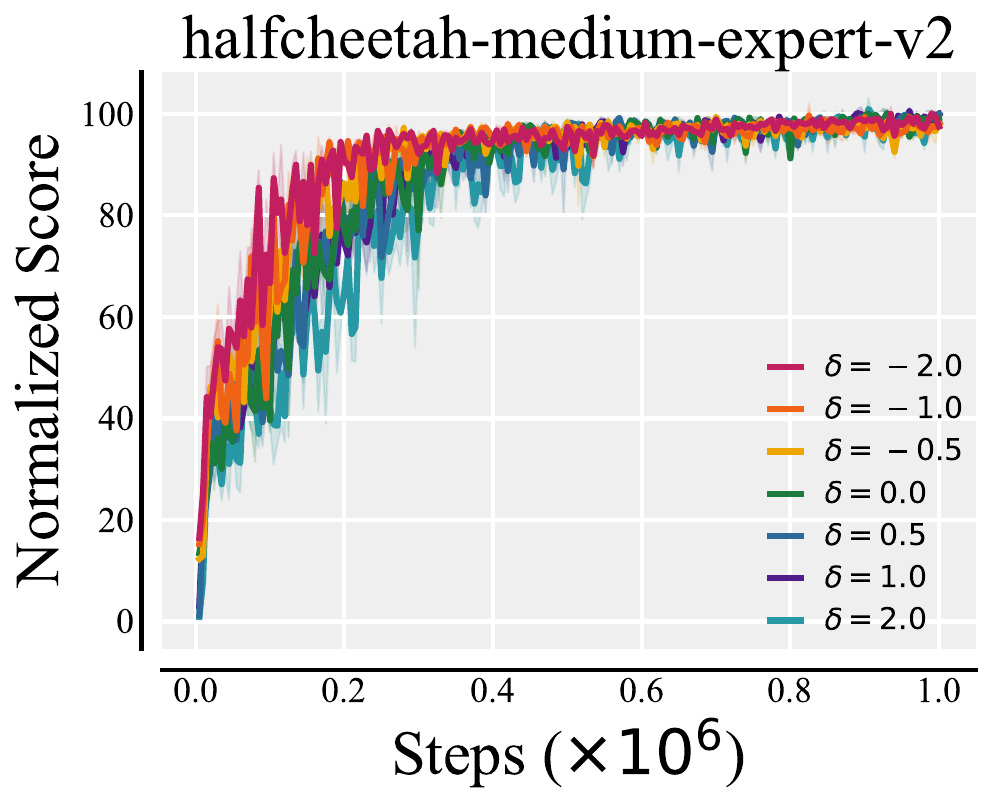}
        \label{fig:ILQ_offset_analysis_halfcheetah-medium-expert-v2}
    }
    \subfloat[ ]{
        \includegraphics[width=0.18\textwidth]{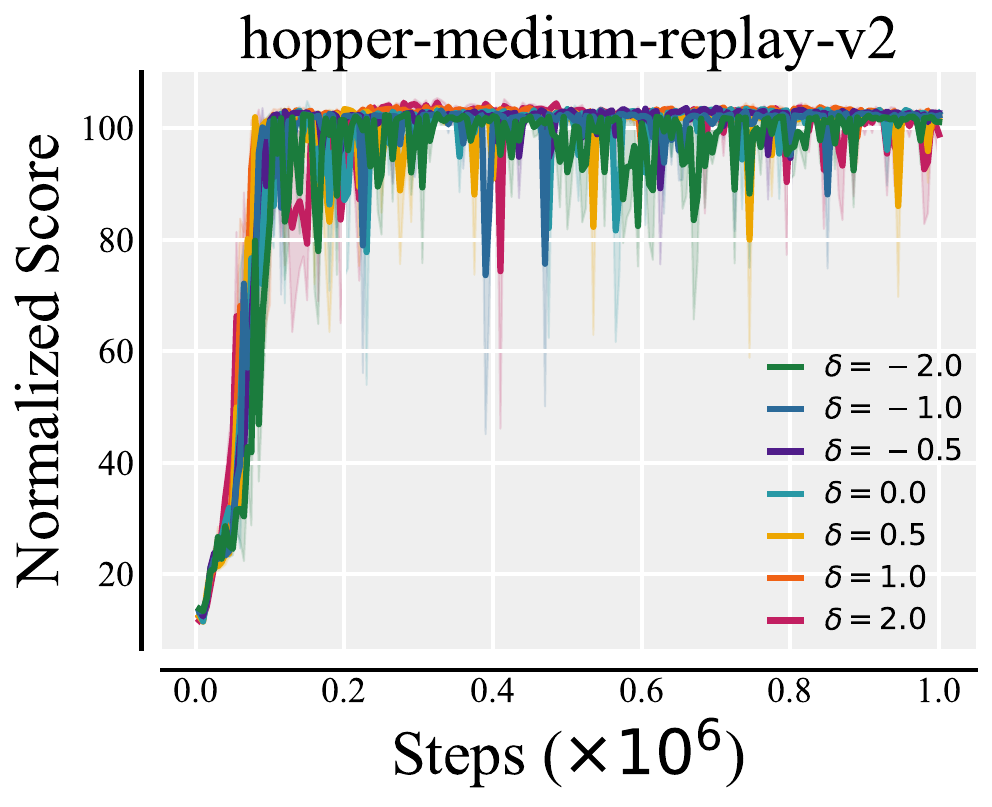}
        \label{fig:ILQ_offset_analysis_hopper-medium-replay-v2}
    }
    \\
    \subfloat[ ]{
        \includegraphics[width=0.18\textwidth]{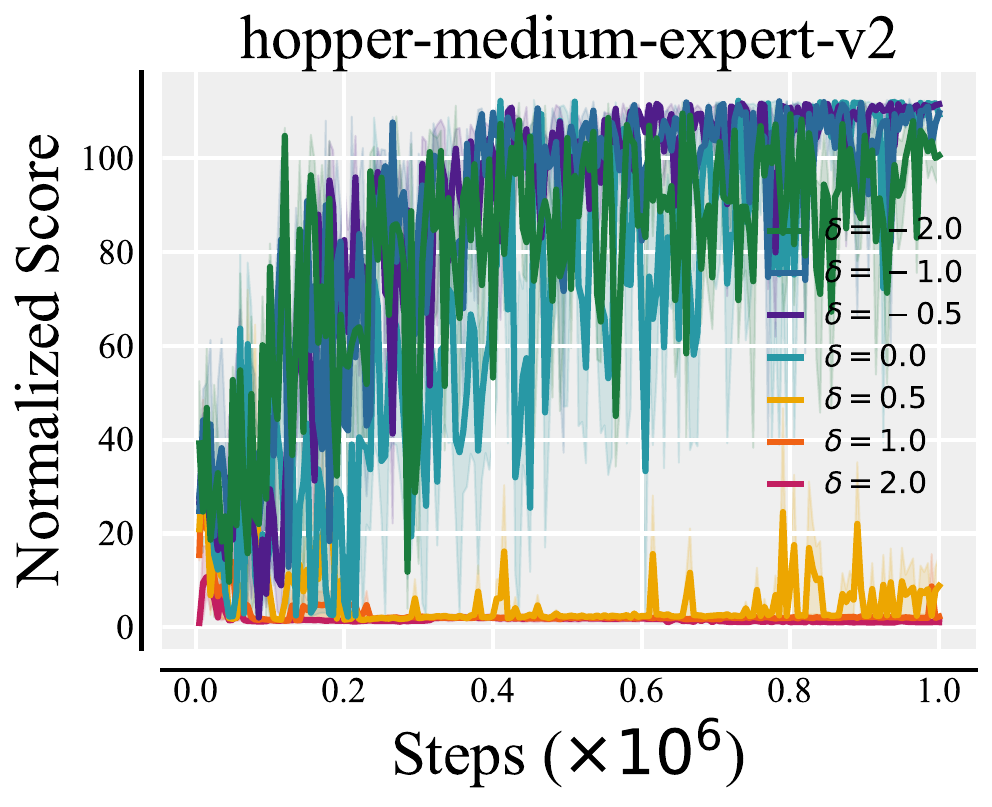}
        \label{fig:ILQ_offset_analysis_hopper-medium-expert-v2}
    }
    \subfloat[ ]{
        \includegraphics[width=0.18\textwidth]{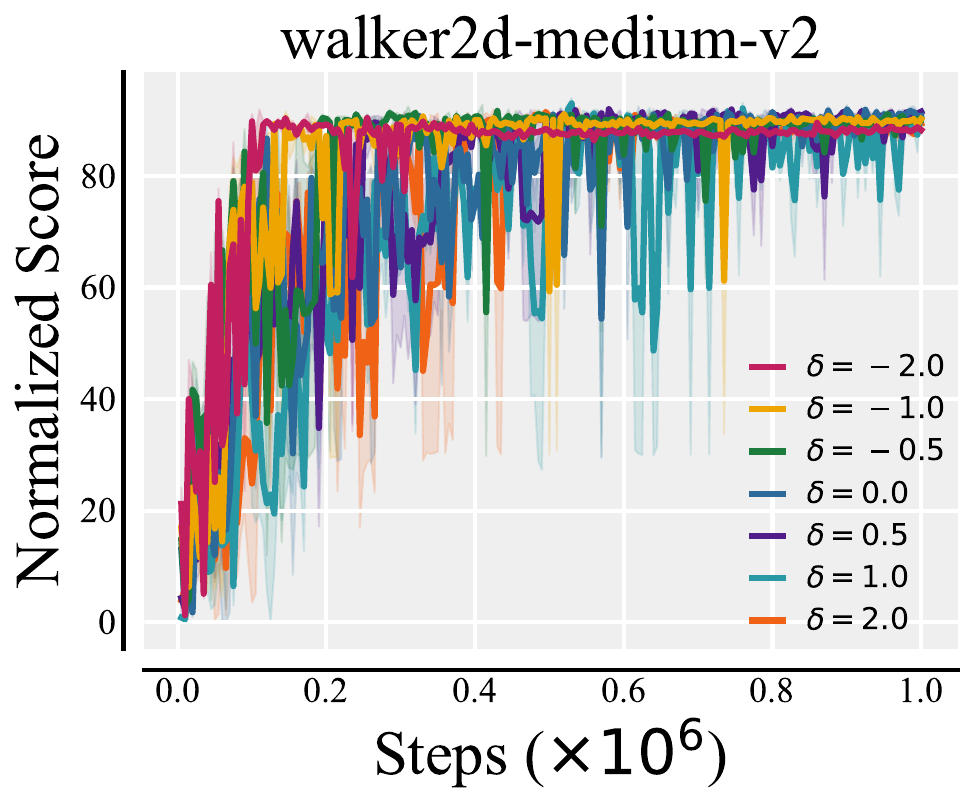}
        \label{fig:ILQ_offset_analysis_walker2d-medium-v2}
    }
    \caption{Performances of ILQ under different values of offset parameter $\delta$.}
    \label{fig:ILQ_offset_analysis}
    \vskip -0.15in
\end{figure}

To assess impacts of parameter $\delta$, we conduct sensitivity analyses by varying $\delta$ across $\{-2, -1, -0.5, 0, 0.5, 1, 2\}$. We evaluate the performance on halfcheetah-me, hopper-mr, hopper-me, and walker2d-m, where each experiment was run over 3 random seeds. 
Figure \ref{fig:ILQ_offset_analysis} shows that the score curve remains stable as $\delta$ changes. Overall, performance is slightly lower when $\delta$ is negative compared to when it is positive. Notably, a negative $\delta$ always ensures a safe learned policy. However, when $\delta$ is positive, meaning the limit exceeds the maximum behavior value a little bit, there may exists policy failure, as illustrated in Fig. \ref{fig:ILQ_offset_analysis}\subref{fig:ILQ_offset_analysis_hopper-medium-expert-v2}. These indeed validate our argument that overly restricting OOD values could inhibit potential performance gains, and taking the maximum behavior value ($\delta=0$) as a ceiling of the imagined value is consistently a solid choice.

\subsubsection{Trade-off Factor $\eta$}
To evaluate the impact of the trade-off factor $\eta$, we conduct sensitivity analyses by varying $\eta$ around its optimal value. 
According to results in Fig. \ref{fig:ILQ_eta_analysis}\subref{fig:ILQ_eta_analysis_halfcheetah-medium-v2} and \ref{fig:ILQ_eta_analysis}\subref{fig:ILQ_eta_analysis_walker2d-medium-replay-v2}, ILQ achieves robust performance for all variations of $\eta$ around $0.9$, making it a typically safe choice for medium and medium-replay tasks. This suggests that ILQ should place greater trust in in-sample value estimates when evaluating policies. In medium tasks, the Q-value of OOD actions generated by the learned policy rarely reach the maximum behavior value. Consequently, assigning too high a weight to these OOD action-values, i.e., a smaller $\eta$, may lead to misplaced trust in OOD actions and ultimately cause policy failure, as illustrated in Fig. \ref{fig:ILQ_eta_analysis}\subref{fig:ILQ_eta_analysis_walker2d-medium-replay-v2}. In medium-expert tasks, the in-sample data comprises a mixture of medium and expert levels, and the value of the OOD actions generated by the learned policy can approach the maximum behavior value. Therefore, the value estimate needs to be more balanced between in-sample and out-of-sample. In this case, the optimal value of $\eta$ is usually around $0.6$ for all medium-expert tasks. The overall results in Fig. \ref{fig:ILQ_eta_analysis} show that ILQ maintains stable performance when $\eta$ varies around its optimal parameter.

\begin{figure}[!tb]
    \centering
    \subfloat[ ]{
        \includegraphics[width=0.18\textwidth]{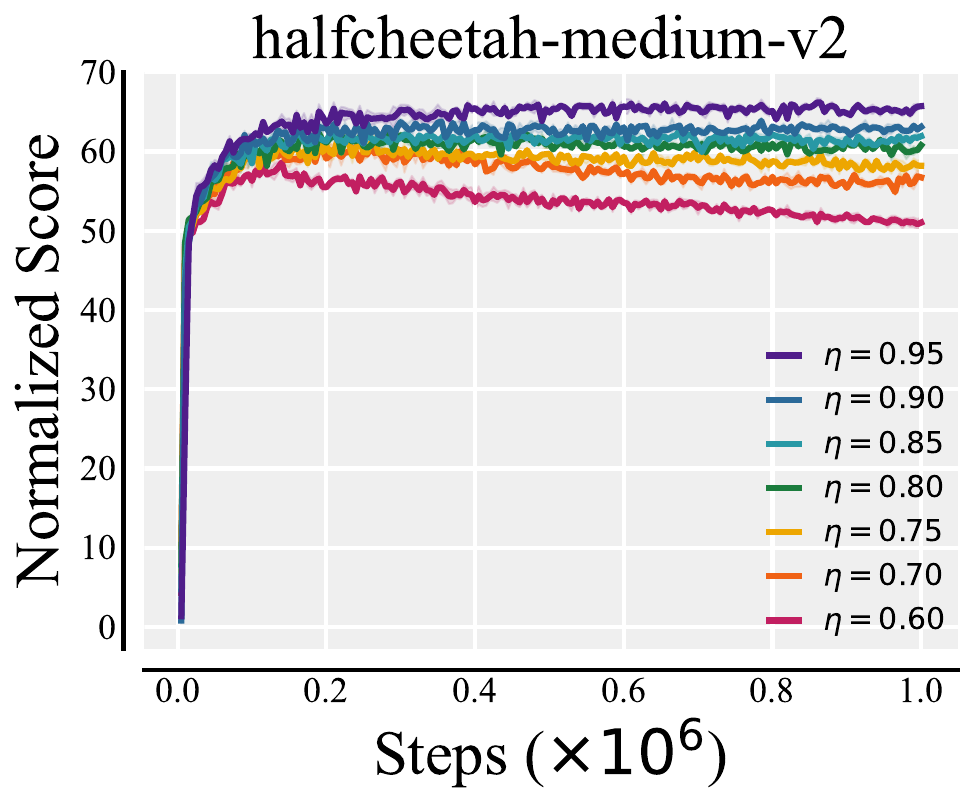}
        \label{fig:ILQ_eta_analysis_halfcheetah-medium-v2}
    }
    \subfloat[ ]{
        \includegraphics[width=0.18\textwidth]{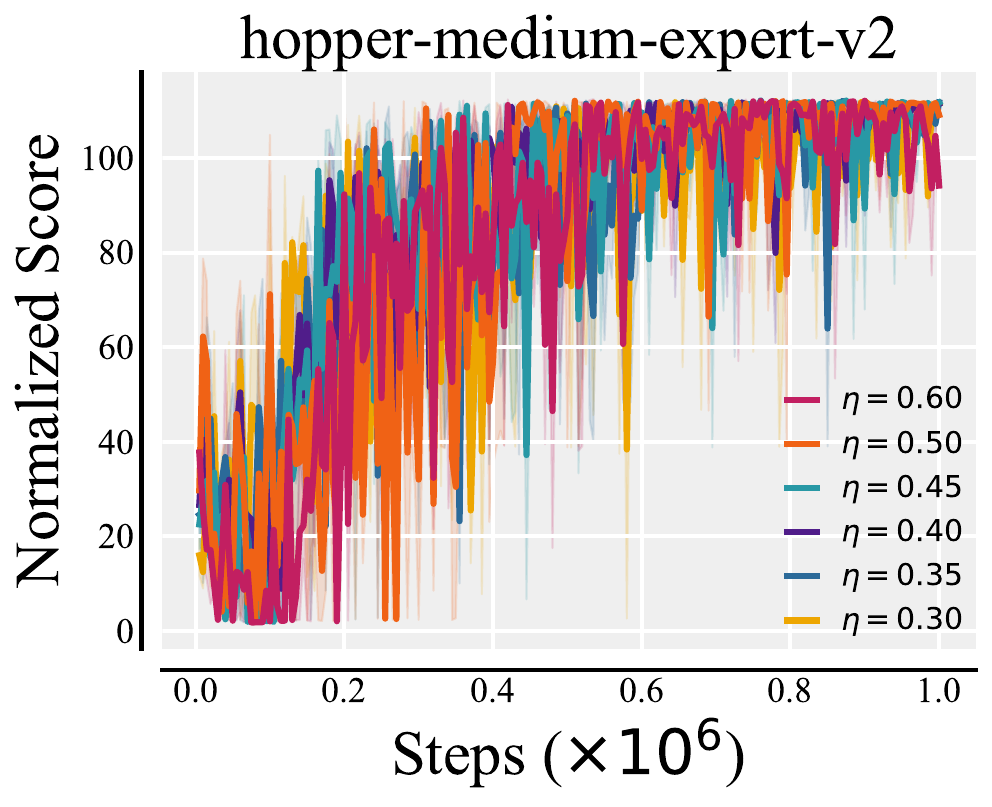}
        \label{fig:ILQ_eta_analysis_hopper-medium-expert-v2}
    }
    \\
    \subfloat[ ]{
        \includegraphics[width=0.18\textwidth]{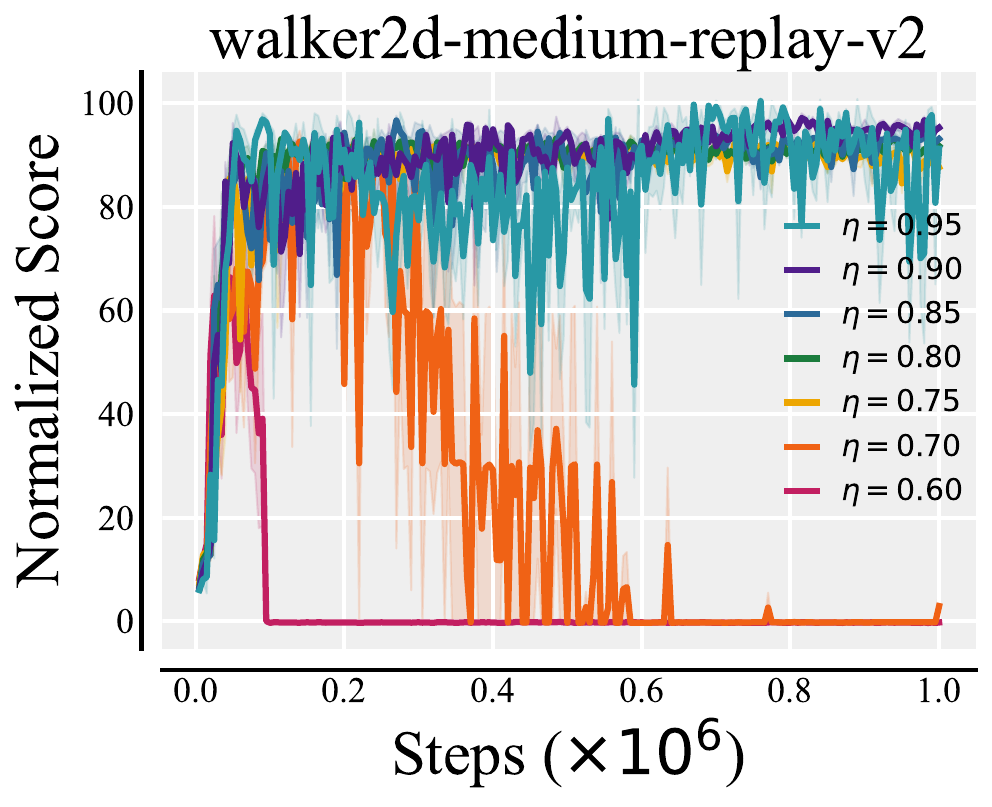}
        \label{fig:ILQ_eta_analysis_walker2d-medium-replay-v2}
    }
    \subfloat[ ]{
        \includegraphics[width=0.18\textwidth]{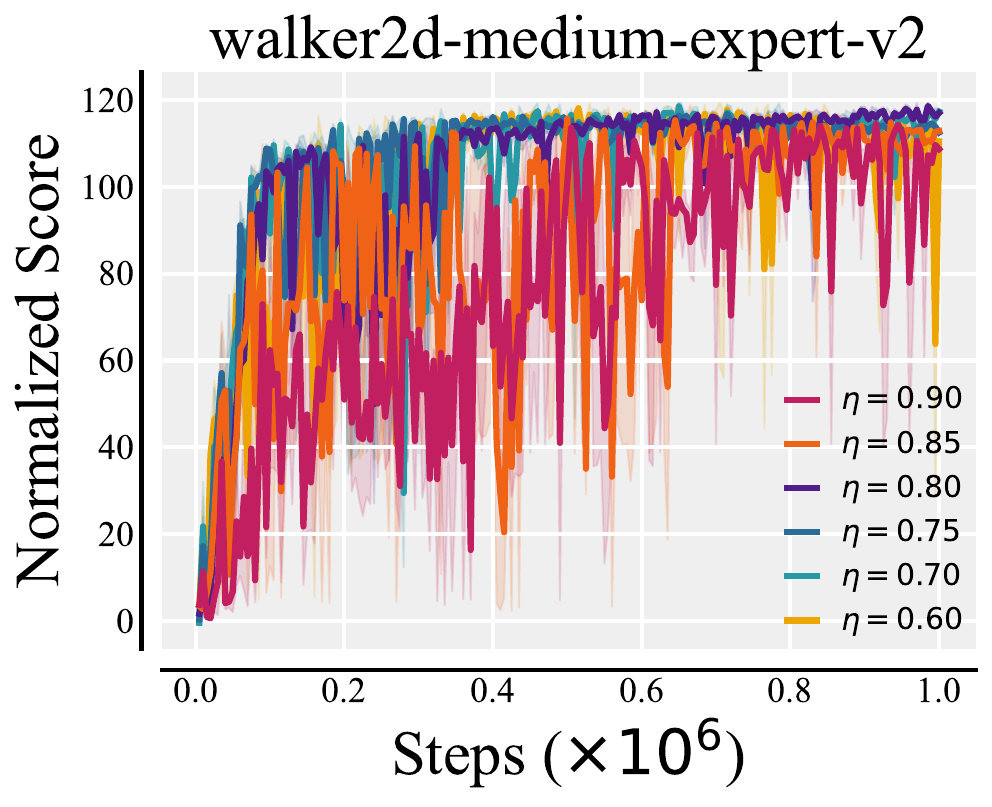}
        \label{fig:ILQ_eta_analysis_walker2d-medium-expert-v2}
    }
    \caption{Performances of ILQ under different values of trade-off factor $\eta$.}
    \label{fig:ILQ_eta_analysis}
    \vskip -0.15in
\end{figure}

\subsection{Ablation Study}
To understand the contribution of each component in our OOD target action-value Eq. \eqref{eq:OOD-sample-target}, we conduct an ablation study. This study evaluates the impact of removing either the imagination component or the limitation value from the target value. All experiments are run over 3 random seeds. More results can be found in the Appendix.

\subsubsection{Without Imagination}
In this part, we assess the performance of ILQ without the imagination component $y_{\rm img}^Q$, indicating solely the maximum behavior value is considered as the regularization target. The results are illustrated in Fig. \ref{fig:ILQ_wo_img}. As shown in Fig. \ref{fig:ILQ_wo_img}\subref{fig:ILQ_woimg_halfcheetah-medium-expert-v2} and \subref{fig:ILQ_woimg_walker2d-medium-v2}, the performance degrades significantly and the curve of normalized score exhibits a very oscillatory behavior. This is because directly using the maximum behavior value as the target values for OOD actions introduces uncontrollable bias and ultimately impairs policy improvement. This demonstrates the necessity of the imagination component for providing calibrated target values under the limitation.

\begin{figure}[!tb]
    \centering
    {
        \includegraphics[width=0.22\textwidth]{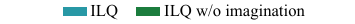}
    }
    \\
    \subfloat[ ]{
        \includegraphics[width=0.18\textwidth]{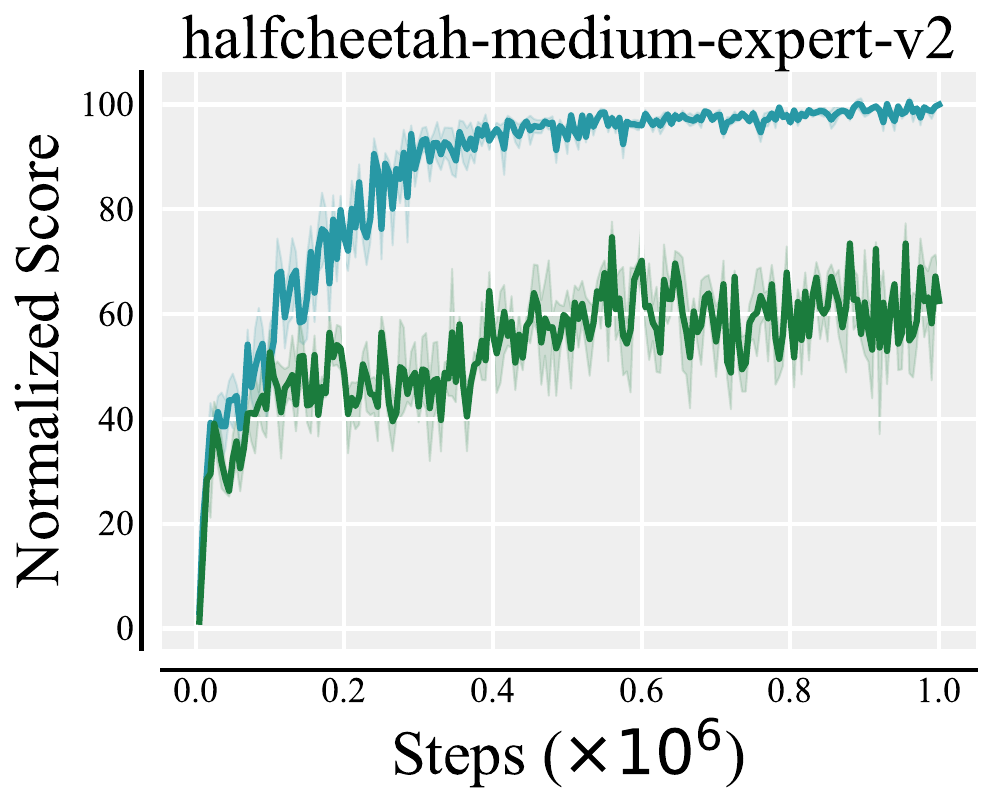}
        \label{fig:ILQ_woimg_halfcheetah-medium-expert-v2}
    }
    \subfloat[ ]{
        \includegraphics[width=0.18\textwidth]{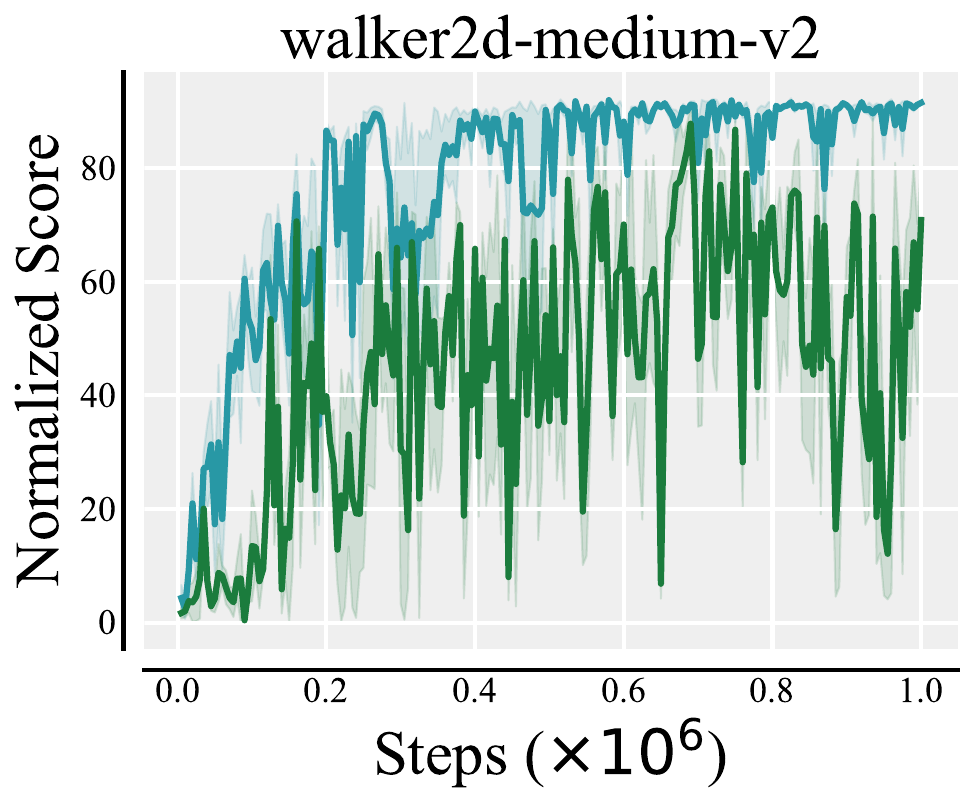}
        \label{fig:ILQ_woimg_walker2d-medium-v2}
    }
    \caption{Performance comparison of the ILQ algorithm with and without the imagined value $y_{\rm img}^Q$ in the target value. }
    \label{fig:ILQ_wo_img}
    \vskip -0.15in
\end{figure}

\subsubsection{Without Limitation}
Here we exclude the limitation component $y_{\rm lmt}^Q$ from the target value, relying solely on $y_{\rm img}^Q$ for learning. This approach is intended to evaluate the importance of the limitation component. In one specific case, Fig. \ref{fig:ILQ_wo_lmt}\subref{fig:ILQ_wolmt_halfcheetah-medium-v2}, performance increased when relying solely on the imagination value, indicating that the imagination component can sometimes provide highly reliable guidance. However, in most tasks, performance dropped significantly, with some policies in hopper-m task Fig. \ref{fig:ILQ_wo_lmt}\subref{fig:ILQ_wolmt_hopper-medium-v2} collapsing completely. This underscores the importance of the limitation component in preventing the incorrectly optimistic estimates. These studies suggest both of the two components play critical role on OOD estimates.

\begin{figure}[!tb]
    \centering
    {
        \includegraphics[width=0.22\textwidth]{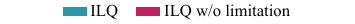}
    }
    \\
    \subfloat[ ]{
        \includegraphics[width=0.18\textwidth]{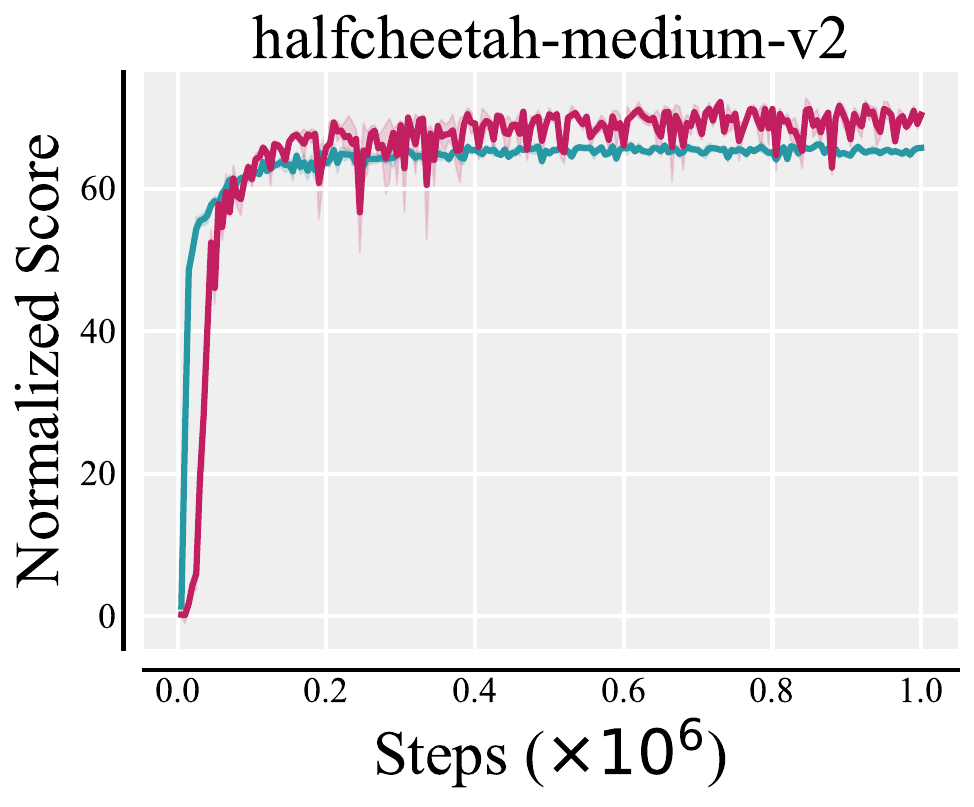}
        \label{fig:ILQ_wolmt_halfcheetah-medium-v2}
    }
    \subfloat[ ]{
        \includegraphics[width=0.18\textwidth]{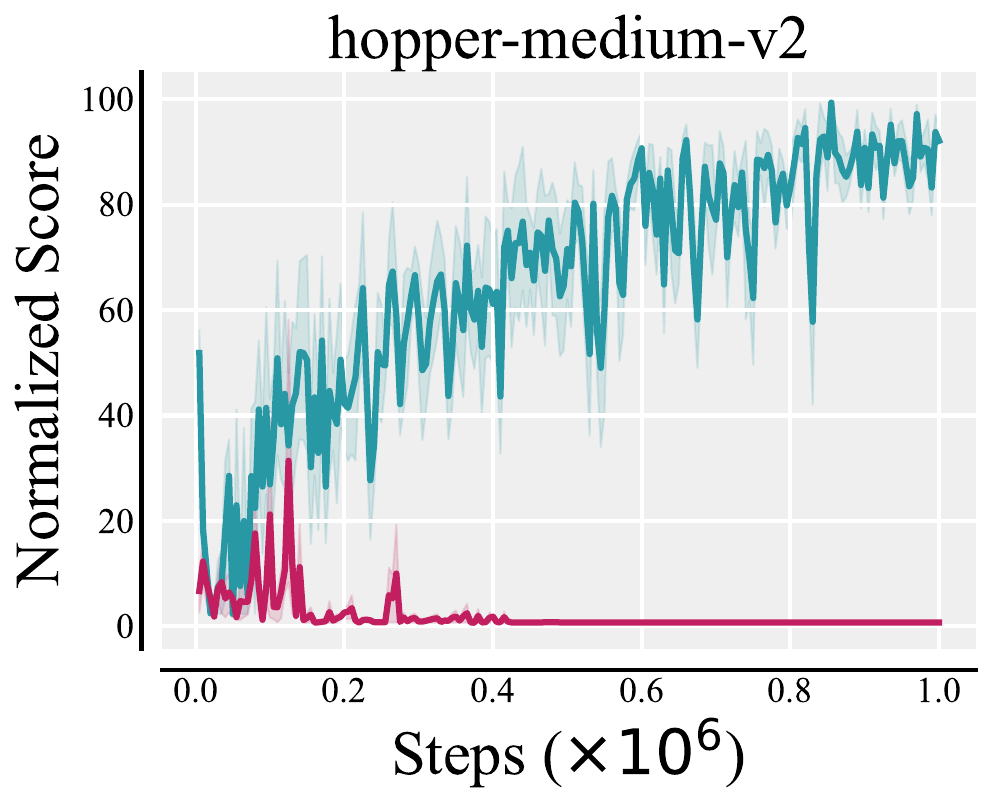}
        \label{fig:ILQ_wolmt_hopper-medium-v2}
    }
    \caption{Performance comparison of the ILQ algorithm with and without the limitation value $y_{\rm lmt}^Q$ in the target value.}
    \label{fig:ILQ_wo_lmt}
    \vskip -0.15in
\end{figure}

\section{Conclusion}
In conclusion, the Imagination-Limited Q-learning (ILQ) method effectively mitigates bias of value estimations by maintaining reasonable evaluations of OOD action-values within appropriate limits. Specifically, it utilizes a dynamics model to help generate imagined values and capping these with the maximum behavior values for OOD actions, while standard target values for in-distribution ones. Theoretical analysis confirms the convergence of ILQ and demonstrates that the error bound between estimated and optimal values for OOD actions is comparable to that for in-distribution actions, thereby enhancing performance improvements. Empirical results show that ILQ achieves state-of-the-art performances on a wide range of tasks in the D4RL benchmark. We hope this work can provide new insights into the value estimates for offline RL.

\section*{Acknowledgments}
This work was partially supported by the National Science Foundation of China (62072185 and U1711262).

\bibliographystyle{named}
\bibliography{ijcai25}

\begin{thebibliography}{}

\bibitem[\protect\citeauthoryear{Bhardwaj \bgroup \em et al.\egroup }{2023}]{Bhardwaj2023ARMOR}
Mohak Bhardwaj, Tengyang Xie, Byron Boots, Nan Jiang, and Ching-An Cheng.
\newblock Adversarial model for offline reinforcement learning.
\newblock In {\em Advances in Neural Information Processing Systems}, volume~36, pages 1245--1269, 2023.

\bibitem[\protect\citeauthoryear{Brandfonbrener \bgroup \em et al.\egroup }{2021}]{Brandfonbrener2021onestep}
David Brandfonbrener, Will Whitney, Rajesh Ranganath, and Joan Bruna.
\newblock Offline rl without off-policy evaluation.
\newblock In {\em Advances in Neural Information Processing Systems}, volume~34, pages 4933--4946, 2021.

\bibitem[\protect\citeauthoryear{Chen \bgroup \em et al.\egroup }{2023}]{chen2023CSVE}
Liting Chen, Jie Yan, Zhengdao Shao, Lu~Wang, Qingwei Lin, Saravanakumar Rajmohan, Thomas Moscibroda, and Dongmei Zhang.
\newblock Conservative state value estimation for offline reinforcement learning.
\newblock In {\em Advances in Neural Information Processing Systems}, volume~36, pages 35064--35083, 2023.

\bibitem[\protect\citeauthoryear{Chen \bgroup \em et al.\egroup }{2024}]{chen2024DTQL}
Tianyu~Chen Chen, Zhendong Wang, and Mingyuan Zhou.
\newblock Diffusion policies creating a trust region for offline reinforcement learning.
\newblock In {\em Advances in Neural Information Processing Systems}, volume~37, pages 1--22, 2024.

\bibitem[\protect\citeauthoryear{Diehl \bgroup \em et al.\egroup }{2021}]{diehl2021umbrella}
Christopher Diehl, Timo Sievernich, Martin Kr{\"u}ger, Frank Hoffmann, and Torsten Bertram.
\newblock {UMBRELLA}: Uncertainty-aware model-based offline reinforcement learning leveraging planning.
\newblock {\em arXiv preprint arXiv:2111.11097}, 2021.

\bibitem[\protect\citeauthoryear{Fu \bgroup \em et al.\egroup }{2020}]{fu2020D4RL}
Justin Fu, Aviral Kumar, Ofir Nachum, George Tucker, and Sergey Levine.
\newblock {D4RL}: Datasets for deep data-driven reinforcement learning.
\newblock {\em arXiv preprint arXiv:2004.07219}, 2020.

\bibitem[\protect\citeauthoryear{Fujimoto and Gu}{2021}]{fujimoto2021TD3-BC}
Scott Fujimoto and Shixiang~Shane Gu.
\newblock A minimalist approach to offline reinforcement learning.
\newblock In {\em Advances in Neural Information Processing Systems}, volume~34, pages 20132--20145, 2021.

\bibitem[\protect\citeauthoryear{Fujimoto \bgroup \em et al.\egroup }{2019}]{fujimoto2019offRL}
Scott Fujimoto, David Meger, and Doina Precup.
\newblock Off-policy deep reinforcement learning without exploration.
\newblock In {\em International Conference on Machine Learning}, pages 2052--2062, 2019.

\bibitem[\protect\citeauthoryear{Gu \bgroup \em et al.\egroup }{2017}]{gu2017robotic}
Shixiang Gu, Ethan Holly, Timothy Lillicrap, and Sergey Levine.
\newblock Deep reinforcement learning for robotic manipulation with asynchronous off-policy updates.
\newblock In {\em IEEE International Conference on Robotics and Automation}, pages 3389--3396, 2017.

\bibitem[\protect\citeauthoryear{Haarnoja \bgroup \em et al.\egroup }{2018}]{haarnoja2018SAC}
Tuomas Haarnoja, Aurick Zhou, Kristian Hartikainen, George Tucker, Sehoon Ha, Jie Tan, Vikash Kumar, Henry Zhu, Abhishek Gupta, Pieter Abbeel, et~al.
\newblock Soft actor-critic algorithms and applications.
\newblock {\em arXiv preprint arXiv:1812.05905}, 2018.

\bibitem[\protect\citeauthoryear{Hansen-Estruch \bgroup \em et al.\egroup }{2023}]{hansen2023IDQL}
Philippe Hansen-Estruch, Ilya Kostrikov, Michael Janner, Jakub~Grudzien Kuba, and Sergey Levine.
\newblock {IDQL}: Implicit q-learning as an actor-critic method with diffusion policies.
\newblock {\em arXiv preprint arXiv:2304.10573}, 2023.

\bibitem[\protect\citeauthoryear{Hasselt}{2010}]{hasselt2010doubleQ}
Hado Hasselt.
\newblock Double q-learning.
\newblock In {\em Advances in Neural Information Processing Systems}, volume~23, pages 1--9, 2010.

\bibitem[\protect\citeauthoryear{Ho \bgroup \em et al.\egroup }{2020}]{ho2020DDPM}
Jonathan Ho, Ajay Jain, and Pieter Abbeel.
\newblock Denoising diffusion probabilistic models.
\newblock In {\em Advances in Neural Information Processing Systems}, volume~33, pages 6840--6851, 2020.

\bibitem[\protect\citeauthoryear{Huang \bgroup \em et al.\egroup }{2024}]{huang2024OAC-BVR}
Longyang Huang, Botao Dong, Wei Xie, and Weidong Zhang.
\newblock Offline reinforcement learning with behavior value regularization.
\newblock {\em IEEE Transactions on Cybernetics}, 54(6):3692--3704, 2024.

\bibitem[\protect\citeauthoryear{Janner \bgroup \em et al.\egroup }{2022}]{janner2022Diffuser}
Michael Janner, Yilun Du, Joshua Tenenbaum, and Sergey Levine.
\newblock Planning with diffusion for flexible behavior synthesis.
\newblock In {\em International Conference on Machine Learning}, pages 9902--9915, 2022.

\bibitem[\protect\citeauthoryear{Kidambi \bgroup \em et al.\egroup }{2020}]{Kidambi2020MOReL}
Rahul Kidambi, Aravind Rajeswaran, Praneeth Netrapalli, and Thorsten Joachims.
\newblock {MOReL}: Model-based offline reinforcement learning.
\newblock In {\em Advances in Neural Information Processing Systems}, volume~33, pages 21810--21823, 2020.

\bibitem[\protect\citeauthoryear{Kingma and Ba}{2014}]{kingma2014Adam}
Diederik~P Kingma and Jimmy Ba.
\newblock Adam: A method for stochastic optimization.
\newblock {\em arXiv preprint arXiv:1412.6980}, 2014.

\bibitem[\protect\citeauthoryear{Kostrikov \bgroup \em et al.\egroup }{2022}]{kostrikov2021IQL}
Ilya Kostrikov, Ashvin Nair, and Sergey Levine.
\newblock Offline reinforcement learning with implicit q-learning.
\newblock In {\em International Conference on Learning Representations}, pages 1--11, 2022.

\bibitem[\protect\citeauthoryear{Kumar \bgroup \em et al.\egroup }{2019}]{kumar2019BEAR}
Aviral Kumar, Justin Fu, Matthew Soh, George Tucker, and Sergey Levine.
\newblock Stabilizing off-policy q-learning via bootstrapping error reduction.
\newblock In {\em Advances in Neural Information Processing Systems}, volume~32, pages 1--11, 2019.

\bibitem[\protect\citeauthoryear{Kumar \bgroup \em et al.\egroup }{2020}]{kumar2020CQL}
Aviral Kumar, Aurick Zhou, George Tucker, and Sergey Levine.
\newblock Conservative q-learning for offline reinforcement learning.
\newblock In {\em Advances in Neural Information Processing Systems}, volume~33, pages 1179--1191, 2020.

\bibitem[\protect\citeauthoryear{Lange \bgroup \em et al.\egroup }{2012}]{lange2012batchRL}
Sascha Lange, Thomas Gabel, and Martin Riedmiller.
\newblock Batch reinforcement learning.
\newblock In {\em Reinforcement learning: State-of-the-art}, pages 45--73. Springer, 2012.

\bibitem[\protect\citeauthoryear{Lee \bgroup \em et al.\egroup }{2021}]{lee2021repbal_model}
Byung-Jun Lee, Jongmin Lee, and Kee-Eung Kim.
\newblock Representation balancing offline model-based reinforcement learning.
\newblock In {\em International Conference on Learning Representations}, pages 1--22, 2021.

\bibitem[\protect\citeauthoryear{Levine \bgroup \em et al.\egroup }{2020}]{levine2020offlineRL}
Sergey Levine, Aviral Kumar, George Tucker, and Justin Fu.
\newblock Offline reinforcement learning: Tutorial, review, and perspectives on open problems.
\newblock {\em arXiv preprint arXiv:2005.01643}, 2020.

\bibitem[\protect\citeauthoryear{Li \bgroup \em et al.\egroup }{2023}]{li2023DOGE}
Jianxiong Li, Xianyuan Zhan, Haoran Xu, Xiangyu Zhu, Jingjing Liu, and Ya-Qin Zhang.
\newblock When data geometry meets deep function: Generalizing offline reinforcement learning.
\newblock In {\em The Eleventh International Conference on Learning Representations}, pages 1--35, 2023.

\bibitem[\protect\citeauthoryear{Lyu \bgroup \em et al.\egroup }{2022}]{lyu2022MCQ}
Jiafei Lyu, Xiaoteng Ma, Xiu Li, and Zongqing Lu.
\newblock Mildly conservative q-learning for offline reinforcement learning.
\newblock In {\em Advances in Neural Information Processing Systems}, volume~35, pages 1711--1724, 2022.

\bibitem[\protect\citeauthoryear{Mnih \bgroup \em et al.\egroup }{2015}]{mnih2015DQNNature}
Volodymyr Mnih, Koray Kavukcuoglu, David Silver, Andrei~A Rusu, Joel Veness, Marc~G Bellemare, Alex Graves, Martin Riedmiller, Andreas~K Fidjeland, Georg Ostrovski, et~al.
\newblock Human-level control through deep reinforcement learning.
\newblock {\em Nature}, 518(7540):529--533, 2015.

\bibitem[\protect\citeauthoryear{Nachum \bgroup \em et al.\egroup }{2019}]{Nachum2019DualDICE}
Ofir Nachum, Yinlam Chow, Bo~Dai, and Lihong Li.
\newblock {DualDICE}: Behavior-agnostic estimation of discounted stationary distribution corrections.
\newblock In {\em Advances in Neural Information Processing Systems}, volume~32, pages 1--11, 2019.

\bibitem[\protect\citeauthoryear{Ovadia \bgroup \em et al.\egroup }{2019}]{Ovadia2019ModelUncertainty}
Yaniv Ovadia, Emily Fertig, Jie Ren, Zachary Nado, D.~Sculley, Sebastian Nowozin, Joshua Dillon, Balaji Lakshminarayanan, and Jasper Snoek.
\newblock Can you trust your model\textquotesingle s uncertainty? evaluating predictive uncertainty under dataset shift.
\newblock In {\em Advances in Neural Information Processing Systems}, volume~32, pages 1--12, 2019.

\bibitem[\protect\citeauthoryear{Prudencio \bgroup \em et al.\egroup }{2023}]{prudencio2023OfflineRL}
Rafael~Figueiredo Prudencio, Marcos~ROA Maximo, and Esther~Luna Colombini.
\newblock A survey on offline reinforcement learning: Taxonomy, review, and open problems.
\newblock {\em IEEE Transactions on Neural Networks and Learning Systems}, 2023.

\bibitem[\protect\citeauthoryear{Sallab \bgroup \em et al.\egroup }{2017}]{sallab2017autodriving}
Ahmad~EL Sallab, Mohammed Abdou, Etienne Perot, and Senthil Yogamani.
\newblock Deep reinforcement learning framework for autonomous driving.
\newblock {\em arXiv preprint arXiv:1704.02532}, 2017.

\bibitem[\protect\citeauthoryear{Song \bgroup \em et al.\egroup }{2021}]{song2021SDE}
Yang Song, Jascha Sohl-Dickstein, Diederik~P Kingma, Abhishek Kumar, Stefano Ermon, and Ben Poole.
\newblock Score-based generative modeling through stochastic differential equations.
\newblock In {\em International Conference on Learning Representations}, pages 1--36, 2021.

\bibitem[\protect\citeauthoryear{Srinivasan and Knottenbelt}{2024}]{srinivasan2024TD3-BST}
Padmanaba Srinivasan and William Knottenbelt.
\newblock Offline reinforcement learning with behavioral supervisor tuning.
\newblock In {\em Proceedings of the Thirty-Third International Joint Conference on Artificial Intelligence}, pages 1--9, 2024.

\bibitem[\protect\citeauthoryear{Sun \bgroup \em et al.\egroup }{2023}]{sun2024PlanCP}
Jiankai Sun, Yiqi Jiang, Jianing Qiu, Parth Nobel, Mykel~J Kochenderfer, and Mac Schwager.
\newblock Conformal prediction for uncertainty-aware planning with diffusion dynamics model.
\newblock In {\em Advances in Neural Information Processing Systems}, volume~36, pages 80324--80337, 2023.

\bibitem[\protect\citeauthoryear{Sutton and Barto}{2018}]{sutton2018RL}
Richard~S Sutton and Andrew~G Barto.
\newblock {\em Reinforcement learning: An introduction}.
\newblock MIT press, 2018.

\bibitem[\protect\citeauthoryear{Wang \bgroup \em et al.\egroup }{2021}]{wang2021ROMI}
Jianhao Wang, Wenzhe Li, Haozhe Jiang, Guangxiang Zhu, Siyuan Li, and Chongjie Zhang.
\newblock Offline reinforcement learning with reverse model-based imagination.
\newblock In {\em Advances in Neural Information Processing Systems}, volume~34, pages 29420--29432, 2021.

\bibitem[\protect\citeauthoryear{Wang \bgroup \em et al.\egroup }{2023}]{wang2022diffusionQL}
Zhendong Wang, Jonathan~J Hunt, and Mingyuan Zhou.
\newblock Diffusion policies as an expressive policy class for offline reinforcement learning.
\newblock In {\em The Eleventh International Conference on Learning Representations}, pages 1--17, 2023.

\bibitem[\protect\citeauthoryear{Wu \bgroup \em et al.\egroup }{2019}]{wu2019BRAC}
Yifan Wu, George Tucker, and Ofir Nachum.
\newblock Behavior regularized offline reinforcement learning.
\newblock {\em arXiv preprint arXiv:1911.11361}, 2019.

\bibitem[\protect\citeauthoryear{Wu \bgroup \em et al.\egroup }{2021}]{wu2021UWAC}
Yue Wu, Shuangfei Zhai, Nitish Srivastava, Joshua~M Susskind, Jian Zhang, Ruslan Salakhutdinov, and Hanlin Goh.
\newblock Uncertainty weighted actor-critic for offline reinforcement learning.
\newblock In {\em International Conference on Machine Learning}, pages 11319--11328, 2021.

\bibitem[\protect\citeauthoryear{Yang \bgroup \em et al.\egroup }{2023}]{yang2023OAP}
Qisen Yang, Shenzhi Wang, Matthieu~Gaetan Lin, Shiji Song, and Gao Huang.
\newblock Boosting offline reinforcement learning with action preference query.
\newblock In {\em International Conference on Machine Learning}, pages 39509--39523, 2023.

\bibitem[\protect\citeauthoryear{Yu \bgroup \em et al.\egroup }{2020}]{yu2020MOPO}
Tianhe Yu, Garrett Thomas, Lantao Yu, Stefano Ermon, James~Y Zou, Sergey Levine, Chelsea Finn, and Tengyu Ma.
\newblock {MOPO}: Model-based offline policy optimization.
\newblock In {\em Advances in Neural Information Processing Systems}, volume~33, pages 14129--14142, 2020.

\bibitem[\protect\citeauthoryear{Yu \bgroup \em et al.\egroup }{2021}]{yu2021combo}
Tianhe Yu, Aviral Kumar, Rafael Rafailov, Aravind Rajeswaran, Sergey Levine, and Chelsea Finn.
\newblock Combo: Conservative offline model-based policy optimization.
\newblock In {\em Advances in Neural Information Processing Systems}, volume~34, pages 28954--28967, 2021.

\end{thebibliography}

\clearpage

\appendix

\section{Appendix}
\newtheorem*{repeatdefinition}{Definition \ref{def:operator}}
\newtheorem*{repeattheorem1}{Theorem \ref{thm:convergence}}
\newtheorem*{repeattheorem2}{Theorem \ref{lem:optmality_Q_gap}}
\newtheorem*{repeattheorem3}{Theorem \ref{lem:img_optimality_Q_gap}}
\newtheorem*{repeattheorem4}{Theorem \ref{thm:action-value-gap}}

\subsection{Detailed Theoretical Analysis}

In this section, we will introduce theoretical properties of the ILQ method in detail. First, we investigate the convergence of value iterations using the ILB operator in tabular MDPs, as confirmed in Theorem \ref{thm:convergence}. Additionally, unlike value regularization methods, we do not intend for ILQ to be a pessimistic algorithm. Instead, it aims to retain reasonable estimates of OOD action-values under appropriate restrictions. Thus, in Theorem \ref{thm:action-value-gap}, we analyze the action-value gap between the fixed point of policy evaluation and the Bellman optimality value. 

Now we begin by presenting the analysis of convergence. To facilitate reading, the definition of our ILB operator is restated here. 

\begin{repeatdefinition}
The Imagination-Limited Bellman (ILB) operator is defined as 

\begin{IEEEeqnarray*}{rl} 
    \mathcal{T}_{\mathrm{ILB}} & Q(s,a) \\
    & = \begin{cases}
r(s,a) + \gamma \mathbb{E}_{s^\prime \sim P} \bigl[ \underset{{\tilde{a}^\prime \sim \pi }}{\max} Q(s^{\prime},\tilde{a}^\prime) \bigr], &{\text{if}}~\beta(a|s)>0 \\ 
\min \left \{ y_{{\rm img}}^{Q}, y_{{\rm lmt}}^{Q} \right \} + \delta, &{\text{otherwise.}} 
\end{cases} \IEEEeqnarraynumspace \IEEEyesnumber
\end{IEEEeqnarray*}
where $\beta$ is the behavior policy, 
\begin{equation} 
    y_{{\rm img}}^{Q} = \widehat{r} (s, a) 
    + \gamma \mathbb{E}_{ {\widehat{s}^{\prime} \sim \widehat{P}(\cdot \mid s, a)} } \left[ \max_{\tilde{a}^{\prime} \sim  \pi } Q (\widehat{s}^{\prime}, \tilde{a}^{\prime}) \right],
\end{equation}
and
\begin{equation}
    y_{{\rm lmt}}^{Q} = \max_{{\widehat{a} \in {\rm Supp}(\beta(\cdot \mid s))}} Q (s, \widehat{a})
\end{equation}
are the imagined value and its limitation, respectively. The $\widehat{P}$ is the empirical transition kernel, $\widehat{r}$ is the empirical reward function, $\delta$ is a hyperparameter with a small absolute value, and ${\rm Supp} (\cdot)$ means support-constrained on the dataset.
\end{repeatdefinition}

\begin{repeattheorem1}[\textbf{Convergence}]
    The ILB operator defined in \eqref{eq:def-ILB-operator} is a $\gamma$-contraction operator in the $\mathcal{L}_{\infty}$ norm, and Q-function iteration rule obeying the ILB operator can converge to a unique fixed point.
\end{repeattheorem1}

\begin{proof}
    Let $Q_1$ and $Q_2$ be two arbitrary Q-functions. To prove the $\gamma$-contraction property of the ILB operator, we have to demonstrate that the following inequality holds:
    \begin{equation}\label{eq:contraction-to-go}
        \begin{aligned}
            &\| \mathcal{T}_{\rm ILB}Q_1 - \mathcal{T}_{\rm ILB}Q_2 \|_\infty \\
            &= \max_{s,a} \left | \mathcal{T}_{\rm ILB}Q_1(s,a) - \mathcal{T}_{\rm ILB}Q_2(s,a) \right | \\
            &\leq \gamma \|Q_1 - Q_2\|_\infty.
        \end{aligned}
    \end{equation}
    
    Thus, we are required to carefully investigate $\left | \mathcal{T}_{\rm ILB}Q_1(s,a) - \mathcal{T}_{\rm ILB}Q_2(s,a) \right |$. We first consider the case of $a \in {\rm Supp}(\beta(\cdot \mid s))$. According to the definition of ILB operator \eqref{eq:def-ILB-operator}, one has
    \begin{equation}\label{eq:in-sample-contraction}
        \begin{aligned}
            &\left | \mathcal{T}_{\rm ILB}Q_1(s,a) - \mathcal{T}_{\rm ILB}Q_2(s,a) \right | \\
            &=\biggl \lvert  \left( r(s,a) + \gamma \mathbb{E}_{s^\prime \sim P}\left[\max_{\tilde{a}^\prime\sim \pi }Q_1(s^\prime,\tilde{a}^\prime)\right] \right) \\
            &\phantom{=\;}  - \left( r(s,a) + \gamma \mathbb{E}_{s^\prime \sim P}\left[\max_{\tilde{a}^\prime\sim \pi }Q_2(s^\prime,\tilde{a}^\prime)\right] \right) \biggr \rvert \\
            &= \gamma \biggl \lvert \mathbb{E}_{s^\prime \sim P}\left[ \max_{\tilde{a}^\prime \sim  \pi}Q_1(s^\prime,\tilde{a}^\prime) - \max_{\tilde{a}^\prime \sim  \pi}Q_2(s^\prime,\tilde{a}^\prime) \right]  \biggr \rvert \\
            &\le \gamma \mathbb{E}_{s^\prime \sim P}\left| \max_{\tilde{a}^\prime \sim  \pi}Q_1(s^\prime,\tilde{a}^\prime) - \max_{\tilde{a}^\prime \sim  \pi}Q_2(s^\prime,\tilde{a}^\prime)  \right| \\
            &\le \gamma \|Q_1 - Q_2\|_\infty.
        \end{aligned}
    \end{equation}
    Otherwise, when $a \notin {\rm Supp}(\beta(\cdot \mid s))$, we have the $\mathcal{T}_{\rm ILB}Q(s,a) = \min \left\{ y^{Q}_{\rm img}, y^{Q}_{\rm lmt} \right\}+\delta$. Therefore,
    \begin{IEEEeqnarray*}{rl} 
            &\left | \mathcal{T}_{\rm ILB}Q_1(s,a) - \mathcal{T}_{\rm ILB}Q_2(s,a) \right | \\
            &=\biggl \lvert \left( \min \left\{ y^{Q_1}_{\rm img}, y^{Q_1}_{\rm lmt} \right\} + \delta \right) - \left( \min \left\{ y^{Q_2}_{\rm img}, y^{Q_2}_{\rm lmt} \right\}  + \delta \right) \biggr \rvert \\
            &=\biggl \lvert \min \left\{ y^{Q_1}_{\rm img}, y^{Q_1}_{\rm lmt} \right\}  - \min \left\{ y^{Q_2}_{\rm img}, y^{Q_2}_{\rm lmt} \right\}  \biggr \rvert \IEEEeqnarraynumspace \IEEEyesnumber \label{eq:operator-diff-ood} \\
    \end{IEEEeqnarray*}
    
    There exist four possible cases for the inner part on the RHS of \eqref{eq:operator-diff-ood} above, including $ \left| y^{Q_1}_{\rm lmt} - y^{Q_2}_{\rm lmt} \right|$, $ \left| y^{Q_1}_{\rm img} - y^{Q_2}_{\rm img} \right|$, $ \left| y^{Q_1}_{\rm img} - y^{Q_2}_{\rm lmt} \right|$ and $ \left| y^{Q_1}_{\rm lmt} - y^{Q_2}_{\rm img} \right|$. For the simplest case $ \left| y^{Q_1}_{\rm lmt} - y^{Q_2}_{\rm lmt} \right|$, one can, analogous to the derivation process in \eqref{eq:in-sample-contraction}, easily verify that the $\gamma$-contraction inequality holds. 
    
    For the second case, we have
    \begin{IEEEeqnarray*}{rl}
            &\left | y^{Q_1}_{\rm img} - y^{Q_2}_{\rm img} \right | \\
            &=\Biggl \lvert \left( \widehat{r}(s,a) + \gamma \mathbb{E}_{\widehat{s}^\prime \sim \widehat{P}(\cdot | s,a)}\left[\max_{\tilde{a}^\prime\sim \pi}Q_1(\widehat{s}^\prime,\tilde{a}^\prime)\right] \right) \\
            &\phantom{=\;} - \left( \widehat{r}(s,a) + \gamma \mathbb{E}_{\widehat{s}^\prime \sim \widehat{P}(\cdot | s,a)}\left[\max_{\tilde{a}^\prime\sim \pi}Q_2(\widehat{s}^\prime,\tilde{a}^\prime)\right] \right) \Biggr \rvert \\
            &= \gamma \left| \sum_{\widehat{s}^\prime} \widehat{P}(\widehat{s}^\prime | s,a) \left[ \max_{\tilde{a}^\prime \sim  \pi}Q_1(\widehat{s}^\prime,\tilde{a}^\prime) - \max_{\tilde{a}^\prime \sim  \pi}Q_2(\widehat{s}^\prime,\tilde{a}^\prime) \right]  \right| \\
            &\le \gamma \sum_{\widehat{s}^\prime} \widehat{P}(\widehat{s}^\prime | s,a)\left| \max_{\tilde{a}^\prime \sim  \pi}Q_1(\widehat{s}^\prime,\tilde{a}^\prime) - \max_{\tilde{a}^\prime \sim  \pi}Q_2(\widehat{s}^\prime,\tilde{a}^\prime)  \right| \\
            &\le \gamma \sum_{\widehat{s}^\prime} \widehat{P}(\widehat{s}^\prime | s,a) \|Q_1 - Q_2\|_\infty \\
            &= \gamma \|Q_1 - Q_2\|_\infty.
    \end{IEEEeqnarray*}
    Now we consider two cross term cases. Without loss of generality, we only proof $\left | y^{Q_1}_{\rm img} - y^{Q_2}_{\rm lmt} \right |$ holds the contraction inequality. In this situation, we have $y^{Q_1}_{\rm img} \leq y^{Q_1}_{\rm lmt}$ and $y^{Q_2}_{\rm lmt} \leq y^{Q_2}_{\rm img}$. Therefore, 
    \begin{IEEEeqnarray*}{rl} 
    \bigl | y^{Q_1}_{\rm img} & -\; y^{Q_2}_{\rm lmt} \bigr | \\
    & = \begin{cases}
    y^{Q_1}_{\rm img} - y^{Q_2}_{\rm lmt} \leq y^{Q_1}_{\rm lmt} - y^{Q_2}_{\rm lmt}, &{\text{if}}~y^{Q_1}_{\rm img} > y^{Q_2}_{\rm lmt} \\ 
    y^{Q_2}_{\rm lmt} - y^{Q_1}_{\rm img} \leq y^{Q_2}_{\rm img} - y^{Q_1}_{\rm img}, &{\text{otherwise.}} 
    \end{cases} \IEEEeqnarraynumspace \IEEEyesnumber\label{eq:ood-ILB-contraction-case3}
    \end{IEEEeqnarray*}
    It can be rewritten as
    \begin{equation} \label{eq:ood-ILB-contraction-case3-v2}
        \left | y^{Q_1}_{\rm img} - y^{Q_2}_{\rm lmt} \right | \leq \max \biggl \{ \left| y^{Q_1}_{\rm lmt} - y^{Q_2}_{\rm lmt} \right|, \left| y^{Q_2}_{\rm img} - y^{Q_1}_{\rm img} \right| \biggr\}.
    \end{equation}
    This means the third case can be bounded by either the first case or the second case. Hence, it also satisfies the $\gamma$-contraction inequality. 
    
    By combining these together, the \eqref{eq:contraction-to-go} is obtained, i.e., the ILB operator is a contraction operator over space $\mathcal{S} \times \mathcal{A}$ with $\mathcal{L}_\infty$ norm when $\gamma < 1$. According to the Banach fixed-point theorem (contraction mapping theorem), the ILB operator converges to a unique fixed point.
\end{proof}

This shows that the convergence of the proposed ILB as a policy evaluation operator is guaranteed. In practice, the $\gamma$ can be utilized to adjust both the convergence speed and the long-term influence of rewards. Nevertheless, it is typically fixed to $0.99$ in almost all algorithms.

In the next step, we will analyze the error bound between the converged Q-value and the optimal Q-value. To accomplish this, we will first introduce the support-constrained Bellman optimality operator.

\begin{lemma}
    The support-constrained Bellman optimality operator 
    \begin{equation*}
        \mathcal{T}_{\rm Supp} Q(s,a) := r(s,a) + \gamma \mathbb{E}_{s^{\prime} \sim P} \left[ \max_{a^{\prime} \in {\rm Supp} (\beta(\cdot|s^{\prime}))} Q(s^{\prime}, a^{\prime}) \right]
    \end{equation*}
    is also a $\gamma$-contraction operator and has a fixed point.
\end{lemma}
\begin{proof}
    This result can be demonstrated by a derivation similar to \eqref{eq:in-sample-contraction}.
\end{proof}
For clarity and ease of reference, the assumptions are also restated here. We make some commonly used assumptions about the reward function \cite[Assumption 1]{huang2024OAC-BVR}.
\begin{enumerate}
    \item The reward function is bounded, i.e., $|r(s,a)| \leq r_{\max}$. Actually, this is consistent with what is required by its definition $r(s,a): \mathcal{S} \times \mathcal{A} \to [-r_{\rm max}, r_{\rm max}]$.
    \item Similar to the Lipschitz condition, i.e., $|r(s,\tilde{a}_1) - r(s,\tilde{a}_2)| \leq \ell \|\tilde{a}_1 - \tilde{a}_2\|_{\infty}$, $\forall s \in \mathcal{S}$ and $\forall \tilde{a}_1, \tilde{a}_2 \in \mathcal{A}$, where $\ell$ is a constant. This requires that the reward function satisfies Lipschitz continuity with respect to actions.
\end{enumerate}
And the error bound assumption between the empirical models and the real ones are required, which is also utilized in both \cite{kumar2020CQL} and \cite{huang2024OAC-BVR}. Suppose the $\widehat{r}$ and $\widehat{P}$ are the empirical reward function and empirical transition dynamics, respectively, the following relationships
    \begin{IEEEeqnarray}{C}
        \Bigl\lVert \widehat{r}(s,a) - r(s,a) \Bigr\rVert_{1} \leq \nicefrac{\zeta_r}{\sqrt{D}}, \\
        \Bigl\lVert \widehat{P}(\cdot \mid s,a) - P(\cdot \mid s,a) \Bigr\rVert_{1} \leq \nicefrac{\zeta_P}{\sqrt{D}}, 
    \end{IEEEeqnarray}
    hold with high probability $\geq 1-\zeta$, $\zeta \in (0,1)$, where $D$ is the constant related to the dataset size, $\zeta_{r}$ and $\zeta_{P}$ are constants related to $\zeta$.

\begin{repeattheorem2}
    Suppose $Q_{\beta^*}$ is the fixed point of the support-constrained Bellman optimality operator. The following gap can be obtained
    \begin{IEEEeqnarray}{C}
        \left \lvert Q_{\beta^*} (s, \pi(s)) - Q_{\beta^*} (s, \beta(s)) \right \rvert
        \leq \ell \epsilon_{\pi} + \gamma \frac{|\mathcal{S}|r_{\rm max}}{1-\gamma} \epsilon_P, \IEEEeqnarraynumspace
    \end{IEEEeqnarray}
    where $\epsilon_{\pi} := \max_{s} \left \lVert \pi(s) - \beta(s) \right \rVert_{\infty}$ and $\epsilon_{P} := \left \lVert P^{\pi} -P^{\beta} \right \rVert_{\infty}$.
\end{repeattheorem2}

\begin{proof}
    Since $Q_{\beta^*}$ is the fixed point of the $\mathcal{T}_{\rm Supp} Q$, the following equation holds
    \begin{equation}\label{eq:opt-Bellman-identity}
        Q_{\beta^*} (s, a) = \mathcal{T}_{\rm Supp} Q_{\beta^*} (s, a) 
    \end{equation}
    for all $(s,a)$ in the state-action space. This yields
    \begin{IEEEeqnarray*}{rl}
            &\lvert Q_{\beta^*} (s, \pi(s))  - Q_{\beta^*} (s, \beta(s)) \rvert \\
            & = \Biggl \lvert r(s,\pi(s)) + \gamma \mathbb{E}_{s^{\prime} \sim P(\cdot \mid s, \pi(s)) } \left[ \max_{a^{\prime} \in {\rm Supp} (\beta(\cdot|s^{\prime}))} Q_{\beta^*}(s^{\prime}, a^{\prime}) \right] \\
            & \phantom{=\;} -  r(s,\beta(s)) - \gamma \mathbb{E}_{s^{\prime} \sim P(\cdot \mid s, \beta(s))} \left[ \max_{a^{\prime} \in {\rm Supp} (\beta(\cdot|s^{\prime}))} Q_{\beta^*}(s^{\prime}, a^{\prime}) \right] \Biggr \rvert \\
            & \leq \left \lvert r(s,\pi(s)) - r(s,\beta(s)) \right \rvert \\
            & \phantom{=\;} + \gamma \Biggl \lvert \mathbb{E}_{s^{\prime} \sim P(\cdot \mid s, \pi(s)) } \left[ \max_{a^{\prime} \in {\rm Supp} (\beta(\cdot|s^{\prime}))} Q_{\beta^*}(s^{\prime}, a^{\prime}) \right] \\
            & \phantom{=\;} - \mathbb{E}_{s^{\prime} \sim P(\cdot \mid s, \beta(s))} \left[ \max_{a^{\prime} \in {\rm Supp} (\beta(\cdot|s^{\prime}))} Q_{\beta^*}(s^{\prime}, a^{\prime}) \right] \Biggr \rvert \\
            & \leq \left \lvert r(s,\pi(s)) - r(s,\beta(s)) \right \rvert  \\
            &\phantom{=\;} + \gamma \sum_{s^\prime \in \mathcal{S}} \Biggl \{ \Bigl \lvert P(s^\prime \mid s, \pi(s)) - P(s^\prime \mid s, \beta(s)) \Bigr \rvert \\
            & \phantom{=\;} \cdot \left \lvert  \max_{a^{\prime} \in {\rm Supp} (\beta(\cdot|s^{\prime}))} Q_{\beta^*}(s^{\prime}, a^{\prime}) \right \rvert \Biggr \} \\
            & \leq \ell \max_{s} \left \lVert \pi(s)-\beta(s) \right \rVert_{\infty} + \gamma \frac{|\mathcal{S}| r_{\rm max}}{1 - \gamma} \left \lVert P^{\pi} -P^{\beta} \right \rVert_{\infty} \IEEEyesnumber \label{eq:last-proof-in-lem-Q-gap} \\
            & = \ell \epsilon_{\pi} + \gamma \frac{|\mathcal{S}| r_{\rm max}}{1 - \gamma} \epsilon_P.
    \end{IEEEeqnarray*}
    The first term in the last inequality \eqref{eq:last-proof-in-lem-Q-gap} holds on the basis of the assumption of Lipschitz condition. The second term holds based on the boundedness of rewards. Actually, for any Q-function, we have
    \begin{equation}\label{eq:Q-bound}
        |Q(s,a)| = \left| \mathbb{E} \sum_{t=1}^{\infty} \gamma^t r_t \right|
        \leq \mathbb{E} \sum_{t=1}^{\infty} \gamma^t |r_t| \leq \frac{r_{\rm max}}{1-\gamma}.
    \end{equation}
    Clearly, it still holds for $Q_{\beta^*}$. We thus obtain the final inequality.
\end{proof}

\begin{repeattheorem3} 
    Suppose $Q_{\beta^*}$ is the fixed point of support-constrained Bellman optimality operator. The gap between the imagination value $y_{\rm img}^{Q_{\beta^*}}$ and $Q_{\beta^*}$ has:
    \begin{equation}
        \begin{aligned}
            &\left \lvert y^{Q_{\beta^*}}_{\rm img} - Q_{\beta^*}(s,a) \right \rvert \\
            & \leq \frac{\zeta_r}{\sqrt{D}} + \gamma \ell \epsilon_{\pi} 
            + \gamma^2 \frac{|\mathcal{S}| r_{\rm max}}{1 - \gamma} \epsilon_P + \gamma \frac{\zeta_P}{\sqrt{D}} \frac{r_{\rm max}}{1-\gamma}.
        \end{aligned}
    \end{equation}
\end{repeattheorem3}

\begin{proof}
    By applying the definition of $y$ and \eqref{eq:opt-Bellman-identity}, we obtain
    \begin{IEEEeqnarray*}{rl}
        &\left \lvert y_{\rm img}^{Q_{\beta^*}} - {Q}_{\beta^*} (s,a) \right \rvert \\
            & = \left \lvert y_{\rm img}^{Q_{\beta^*}} - \mathcal{T}_{\rm Supp} {Q}_{\beta^*} (s,a) \right \rvert\\
            &= \Biggl \lvert \widehat{r} (s,a) + \gamma \mathbb{E}_{ {\widehat{s}^{\prime} \sim \widehat{P}(\cdot \mid s, a)} } \left[ \max_{\tilde{a}^{\prime} \sim  \pi} Q_{\beta^*} (\widehat{s}^{\prime}, \tilde{a}^{\prime}) \right] \\
            &\phantom{=\;} - r(s,a) - \gamma \mathbb{E}_{s^{\prime} \sim P(\cdot \mid s, a)} \left[ \max_{a^{\prime} \in {\rm Supp} (\beta(\cdot|s^{\prime}))} Q_{\beta^*} (s^{\prime}, a^{\prime}) \right] \Biggr \rvert \\
            &\leq |\widehat{r}(s,a) - r(s,a)| + \gamma \Biggl \lvert \mathbb{E}_{ {\widehat{s}^{\prime} \sim \widehat{P}(\cdot \mid s, a)} } \left[ \max_{\tilde{a}^{\prime} \sim  \pi} Q_{\beta^*} (\widehat{s}^{\prime}, \tilde{a}^{\prime}) \right] \\
            &\phantom{=\;} -  \mathbb{E}_{s^{\prime} \sim P(\cdot \mid s, a)} \left[ \max_{a^{\prime} \in {\rm Supp} (\beta(\cdot|s^{\prime}))} Q_{\beta^*} (s^{\prime}, a^{\prime}) \right] \Biggr \rvert \\
            &\leq \frac{\zeta_r}{\sqrt{D}} + \gamma \Biggl \lvert \mathbb{E}_{ {\widehat{s}^{\prime} \sim \widehat{P}(\cdot \mid s, a)} } \left[ \max_{\tilde{a}^{\prime} \sim  \pi} Q_{\beta^*} (\widehat{s}^{\prime}, \tilde{a}^{\prime}) \right] \\
            &\phantom{=\;}  -  \mathbb{E}_{s^{\prime} \sim P(\cdot \mid s, a)} \left[ \max_{a^{\prime} \in {\rm Supp} (\beta(\cdot|s^{\prime}))} Q_{\beta^*} (s^{\prime}, a^{\prime}) \right] \Biggr \rvert \\
    \end{IEEEeqnarray*}
    The last inequality is derived based on the concentration assumption \eqref{eq:emp_reward_assump} of the empirical reward model. Now, using the triangle inequality, we can infer that
    \begin{IEEEeqnarray*}{rl}
        &\left \lvert y_{\rm img}^{Q_{\beta^*}} - {Q}_{\beta^*} (s,a) \right \rvert \\
            &\leq \frac{\zeta_r}{\sqrt{D}} + \gamma \Biggl \lvert \mathbb{E}_{ {\widehat{s}^{\prime} \sim \widehat{P}(\cdot \mid s, a)} } \left[ \max_{\tilde{a}^{\prime} \sim  \pi} Q_{\beta^*} (\widehat{s}^{\prime}, \tilde{a}^{\prime}) \right] \\
            &\phantom{=\;} -  \mathbb{E}_{{\widehat{s}^{\prime} \sim \widehat{P}(\cdot \mid s, a)}} \left[ \max_{a^{\prime} \in {\rm Supp} (\beta(\cdot|\widehat{s}^{\prime}))} Q_{\beta^*} (\widehat{s}^{\prime}, a^{\prime}) \right] \Biggr \rvert \\
            &\phantom{=\;} + \gamma \Biggl \lvert \mathbb{E}_{{\widehat{s}^{\prime} \sim \widehat{P}(\cdot \mid s, a)}} \left[ \max_{a^{\prime} \in {\rm Supp} (\beta(\cdot|\widehat{s}^{\prime}))} Q_{\beta^*} (\widehat{s}^{\prime}, a^{\prime}) \right] \\
            &\phantom{=\;} -  \mathbb{E}_{s^{\prime} \sim P(\cdot \mid s, a)} \left[ \max_{a^{\prime} \in {\rm Supp} (\beta(\cdot|s^{\prime}))} Q_{\beta^*} (s^{\prime}, a^{\prime}) \right] \Biggr \rvert\\
            &\leq \frac{\zeta_r}{\sqrt{D}} + \gamma \mathbb{E}_{ {\widehat{s}^{\prime} \sim \widehat{P}(\cdot \mid s, a)} } \Biggl \lvert  \max_{\tilde{a}^{\prime} \sim  \pi} Q_{\beta^*} (\widehat{s}^{\prime}, \tilde{a}^{\prime})  \\
            &\phantom{=\;} -  \max_{a^{\prime} \in {\rm Supp} (\beta(\cdot|\widehat{s}^{\prime}))} Q_{\beta^*} (\widehat{s}^{\prime}, a^{\prime})  \Biggr \rvert \\
            &\phantom{=\;} + \gamma \Biggl \lvert \mathbb{E}_{{\widehat{s}^{\prime} \sim \widehat{P}(\cdot \mid s, a)}} \left[ \max_{a^{\prime} \in {\rm Supp} (\beta(\cdot|\widehat{s}^{\prime}))} Q_{\beta^*} (\widehat{s}^{\prime}, a^{\prime}) \right] \\
            &\phantom{=\;} -  \mathbb{E}_{s^{\prime} \sim P(\cdot \mid s, a)} \left[ \max_{a^{\prime} \in {\rm Supp} (\beta(\cdot|s^{\prime}))} Q_{\beta^*} (s^{\prime}, a^{\prime}) \right] \Biggr \rvert \\
            &\leq \frac{\zeta_r}{\sqrt{D}} + \gamma \sum_{\widehat{s}^{\prime}} \widehat{P}(\widehat{s}^{\prime} | s,a) \biggl( \ell \left \lVert \pi(\widehat{s}^{\prime})-\beta(\widehat{s}^{\prime}) \right \rVert_{\infty} \\
            &\phantom{\leq {\zeta}_{r} + \gamma \sum_{\widehat{s}^{\prime}} \widehat{P}(\widehat{s}^{\prime} | s,a)}
            + \gamma \frac{|\mathcal{S}| r_{\rm max}}{1 - \gamma} \left \lVert P^{\pi} -P^{\beta} \right \rVert_{\infty} \biggr) \\
            &\phantom{=\;} + \gamma \sum_{s^{\prime}} \biggl( \left\lvert\widehat{P}(s^{\prime}|s,a)-P(s^{\prime}|s,a) \right\rvert \\
            &\phantom{\leq \frac{\zeta_r}{\sqrt{D}} + \gamma \sum_{\widehat{s}^{\prime}} \widehat{P}(\widehat{s}^{\prime} | s,a)}
            \cdot \left\lvert \max_{a^{\prime} \in {\rm Supp} (\beta(\cdot|s^{\prime}))} Q_{\beta^*} (s^{\prime}, a^{\prime}) \right\rvert \biggr) \\
            &\leq \frac{\zeta_r}{\sqrt{D}} + \gamma \ell \epsilon_{\pi}
            + \gamma^2 \frac{|\mathcal{S}| r_{\rm max}}{1 - \gamma} \epsilon_P + \gamma \frac{\zeta_P}{\sqrt{D}} \frac{r_{\rm max}}{1-\gamma}.
    \end{IEEEeqnarray*}
    The second term and third term of the last inequality are obtained by a derivation process similar to the proof of \eqref{eq:last-proof-in-lem-Q-gap}. The last term of the last inequality is built on the error bound assumption of the empirical dynamics model and \eqref{eq:Q-bound}.
\end{proof}

Based on these theorems, we now estimate the action-value gap between the fixed point of the ILB operator and $Q_{\beta^*}$.

\begin{repeattheorem4}[\textbf{Action-value gap}]
    Suppose $Q_{\rm ILB}$ and $Q_{\beta^*}$ denote the fixed point of the ILB operator and support-constrained Bellman optimality operator, separately. The action-value gap can be bounded as 
    \begin{equation}
        \begin{aligned}
        &\left\| Q_{\rm ILB} (s,a) - Q_{\beta^*} (s,a) \right\|_{\infty}\\
            &\leq \frac{1}{1-\gamma} \frac{\zeta_r}{\sqrt{D}} + \frac{\ell}{1-\gamma} \epsilon_{\pi} \\
            &\phantom{=\;} + \frac{\gamma |\mathcal{S}| r_{\rm max}}{(1-\gamma)^2} \epsilon_P + \frac{\gamma r_{\rm max}}{(1-\gamma)^2} \frac{\zeta_P}{\sqrt{D}} + \frac{1}{1-\gamma} \lvert \delta \rvert,
        \end{aligned}
    \end{equation}
    where $\zeta_r,~\zeta_P$ are defined in \eqref{eq:emp_reward_assump} and \eqref{eq:emp_dynamics_assump}, $\epsilon_r,~\epsilon_P$ are defined in Theorem \ref{lem:optmality_Q_gap}, $\delta$ is defined in the ILB operator.
\end{repeattheorem4}

\begin{proof}
    If $a \in {\rm Supp} (\beta(\cdot | s))$, we have
    \begin{equation}
        \mathcal{T}_{\rm ILB} Q_{\beta^*} (s,a)
        = r(s,a) + \gamma \mathbb{E}_{s^\prime \sim P} \left[\max_{\tilde{a}^\prime \sim  \pi} Q_{\beta^*}(s^\prime,\tilde{a}^\prime)\right]
    \end{equation}
    Hence,
    \begin{equation}\label{eq:in-sample-sec-term-estimate}
        \begin{aligned}
            &\lvert \mathcal{T}_{\rm ILB} Q_{\beta^*} (s,a) - \mathcal{T}_{\rm Supp} Q_{\beta^*} (s,a) \rvert \\
            &= \Biggl \lvert \gamma \mathbb{E}_{s^\prime \sim P} \left[\max_{\tilde{a}^\prime \sim  \pi} Q_{\beta^*}(s^\prime,\tilde{a}^\prime)\right]  \\
            &\phantom{=\;}  - \gamma \mathbb{E}_{s^{\prime} \sim P} \left[ \max_{a^{\prime} \in {\rm Supp} (\beta(\cdot|s^{\prime}))} Q_{\beta^*} (s^{\prime}, a^{\prime}) \right] \Biggr \rvert \\
            &\leq \gamma \mathbb{E}_{s^\prime \sim P} \left \lVert Q_{\beta^*} (s^\prime, \pi(s^\prime)) - Q_{\beta^*} (s^\prime, \beta(s^\prime)) \right \rVert_{\infty} \\
            &\leq \gamma \ell \epsilon_{\pi} + \gamma^2 \frac{|\mathcal{S}| r_{\rm max}}{1 - \gamma} \epsilon_P.
        \end{aligned}
    \end{equation}
    The last inequality is obtained by utilizing Theorem \ref{lem:optmality_Q_gap}. Now we can estimate the error bound in the support region of $\beta$.
    
    \begin{IEEEeqnarray*}{rl}
            &\left | {Q}_{\rm ILB} (s,a) - {Q}_{\beta^*} (s,a) \right | \\
            &= \left | \mathcal{T}_{\rm ILB} {Q}_{\rm ILB} (s,a) - \mathcal{T}_{\rm Supp} {Q}_{\beta^*} (s,a) \right | \\
            &= \lvert \mathcal{T}_{\rm ILB} {Q}_{\rm ILB} (s,a) - \mathcal{T}_{\rm ILB} {Q}_{\beta^*} (s,a) \\
            &\phantom{=\;}  + \mathcal{T}_{\rm ILB} {Q}_{\beta^*} (s,a) - \mathcal{T}_{\rm Supp} {Q}_{\beta^*} (s,a) \rvert \\
            &\leq \left | \mathcal{T}_{\rm ILB} {Q}_{\rm ILB} (s,a) - \mathcal{T}_{\rm ILB} {Q}_{\beta^*} (s,a) \right | \\
            &\phantom{=\;} + \left | \mathcal{T}_{\rm ILB} {Q}_{\beta^*} (s,a) - \mathcal{T}_{\rm Supp} {Q}_{\beta^*} (s,a) \right | \\
            &\leq \gamma \left| Q_{\rm ILB} (s,a) - Q_{\beta^*} (s,a) \right| + \gamma \ell \epsilon_{\pi} + \gamma^2 \frac{|\mathcal{S}| r_{\rm max}}{1 - \gamma} \epsilon_P.\IEEEyesnumber\IEEEeqnarraynumspace\label{eq:in-sample-triangle}
    \end{IEEEeqnarray*}
    We use the contraction property of ILB operator to get the first term in the last inequality, and derive the second term by applying equation \eqref{eq:in-sample-sec-term-estimate} directly. By transposing terms, we can get
    \begin{equation}\label{eq:in-sample-gap}
    |Q_{\rm ILB} (s,a) - Q_{\beta^*} (s,a) | \leq \frac{\gamma}{1-\gamma} \ell \epsilon_{\pi} + \gamma^2 \frac{|\mathcal{S}| r_{\rm max}}{(1 - \gamma)^2} \epsilon_P.
    \end{equation}
    At last, we consider the error bound in the case of $a \notin {\rm Supp}(\beta(\cdot|s))$. Similarly, from the fixed point property and $\gamma$-contraction inequality, it follows that
    \begin{equation}\label{eq:out-of-sample-triangle}
        \begin{aligned}
            &\left | {Q}_{\rm ILB} (s,a) - {Q}_{\beta^*} (s,a) \right |\\
            &\leq \left | \mathcal{T}_{\rm ILB} {Q}_{\rm ILB} (s,a) - \mathcal{T}_{\rm ILB} {Q}_{\beta^*} (s,a) \right | \\
            &\phantom{=\;} + \left | \mathcal{T}_{\rm ILB} {Q}_{\beta^*} (s,a) - \mathcal{T}_{\rm Supp} {Q}_{\beta^*} (s,a) \right | \\
            &\leq \gamma \left| Q_{\rm ILB} (s,a) - Q_{\beta^*} (s,a) \right| \\
            &\phantom{=\;} + \left | \mathcal{T}_{\rm ILB} {Q}_{\beta^*} (s,a) - \mathcal{T}_{\rm Supp} {Q}_{\beta^*} (s,a) \right |.
        \end{aligned}
    \end{equation}
    For the second term on the RHS above, we now have
    \begin{equation}
        \begin{aligned}
            &\left | \mathcal{T}_{\rm ILB} {Q}_{\beta^*} (s,a) - \mathcal{T}_{\rm Supp} {Q}_{\beta^*} (s,a) \right |\\
            &= \biggl | \min \left\{ y_{\rm img}^{Q_{\beta^*}}, y_{\rm lmt}^{Q_{\beta^*}} \right\} + \delta - \mathcal{T}_{\rm Supp} {Q}_{\beta^*} (s,a) \biggr | \\
            &= \max \biggl \{ \left| y_{\rm img}^{Q_{\beta^*}} + \delta - \mathcal{T}_{\rm Supp} {Q}_{\beta^*} (s,a) \right|, \\
            &\phantom{=\;} \left| y_{\rm lmt}^{Q_{\beta^*}} + \delta - \mathcal{T}_{\rm Supp} {Q}_{\beta^*} (s,a) \right| \biggr \}.
        \end{aligned}
    \end{equation}
    The first case can be estimated by the Theorem \ref{lem:img_optimality_Q_gap}, and we can see that
    \begin{equation}\label{eq:img_gap}
        \begin{aligned}
            &\left| y_{\rm img}^{Q_{\beta^*}} + \delta - \mathcal{T}_{\rm Supp} {Q}_{\beta^*} (s, a) \right| \\
            &\leq \frac{\zeta_r}{\sqrt{D}} + \gamma \ell \epsilon_{\pi} + \gamma^2 \frac{|\mathcal{S}| r_{\rm max}}{1 - \gamma} \epsilon_P + \gamma \frac{\zeta_P}{\sqrt{D}} \frac{r_{\rm max}}{1-\gamma} + \lvert \delta \rvert.
        \end{aligned}
    \end{equation}

    For the second case, by the Theorem \ref{lem:optmality_Q_gap}, we have
    \begin{equation}\label{eq:lmt_gap}
        \begin{aligned}
            &\left| y_{\rm lmt}^{Q_{\beta^*}} + \delta - \mathcal{T}_{\rm Supp} {Q}_{\beta^*} (s, a) \right| \\
            &= \left| \max_{\widehat{a} \in {\rm Supp} (\beta(\cdot \mid s))} Q_{\beta^*} (s, \widehat{a}) + \delta - {Q}_{\beta^*} (s, a) \right| \\
            &\leq \left| \max_{\widehat{a} \in {\rm Supp} (\beta(\cdot \mid s))} Q_{\beta^*} (s, \widehat{a}) - {Q}_{\beta^*} (s, a) \right| + \lvert \delta \rvert \\
            &\leq \ell \epsilon_{\pi} + \gamma \frac{|\mathcal{S}|r_{\rm max}}{1-\gamma} \epsilon_P + \lvert \delta \rvert.
        \end{aligned}
    \end{equation}
    Combining \eqref{eq:out-of-sample-triangle} to \eqref{eq:lmt_gap} together, we have
    \begin{equation}\label{eq:out-of-sample-gap}
        \begin{aligned}
            &\left | {Q}_{\rm ILB} (s,a) - {Q}_{\beta^*} (s,a) \right | \\
            &\leq \frac{1}{1-\gamma} \left | \mathcal{T}_{\rm ILB} {Q}_{\beta^*} (s,a) - \mathcal{T}_{\rm Supp} {Q}_{\beta^*} (s,a) \right | \\
            &\leq \max \Bigl \{ \frac{1}{1-\gamma} \frac{\zeta_r}{\sqrt{D}} + \frac{\gamma \ell}{1-\gamma} \epsilon_{\pi} \\
            &\phantom{=\;} + \frac{\gamma^2 |\mathcal{S}| r_{\rm max}}{(1-\gamma)^2} \epsilon_P +  \frac{\gamma r_{\rm max}}{(1-\gamma)^2} \frac{\zeta_P}{\sqrt{D}} + \frac{1}{1-\gamma} \lvert \delta \rvert , \\
            &\phantom{=\;} \frac{\ell}{1-\gamma} \epsilon_{\pi} + \frac{\gamma |\mathcal{S}| r_{\rm max}}{(1-\gamma)^2} \epsilon_P + \frac{1}{1-\gamma} \lvert \delta \rvert \Bigr \} \\
            &\leq \frac{1}{1-\gamma} \frac{\zeta_r}{\sqrt{D}} + \frac{\ell}{1-\gamma} \epsilon_{\pi} \\
            &\phantom{=\;} + \frac{\gamma |\mathcal{S}| r_{\rm max}}{(1-\gamma)^2} \epsilon_P + \frac{\gamma r_{\rm max}}{(1-\gamma)^2} \frac{\zeta_P}{\sqrt{D}} + \frac{1}{1-\gamma} \lvert \delta \rvert.
        \end{aligned}
    \end{equation}
    Taking together \eqref{eq:in-sample-gap} and \eqref{eq:out-of-sample-gap}, the theorem is proved.
\end{proof}

According to \eqref{eq:in-sample-gap} and \eqref{eq:out-of-sample-gap}, we conclude that the error bounds for in-sample and out-of-sample actions are of the same magnitude $\mathcal{O}(\nicefrac{r_{\rm max}}{(1-\gamma)^2})$. This result aligns with the conclusion of CQL \cite{kumar2020CQL} within the support region. Notably, the theoretical optimal value of delta is $0$, based on the assumption of no error in the maximum behavior value. In practice, the optimal value may fluctuate around $0$. Nevertheless, $\delta = 0$ consistently provides good performance across all tasks in experiments.

\subsection{Experimental Settings}
\subsubsection{Evaluation Metric}
The standard performance indicator is the normalized score, defined as 
\begin{equation*}
    \text{normalized score} = 100 \times \frac{\text{learned score} - \text{random score}}{\text{expert score} - \text{random score}},
\end{equation*}
where the learned score is obtained by the test method, the expert score and random score are two constants taken from the D4RL \cite{fu2020D4RL} benchmark.

\subsubsection{Competitors}
In the MuJoCo tasks, we compare our method with prior state-of-the-art methods, including BCQ \cite{fujimoto2019offRL}, CQL \cite{kumar2020CQL}, UWAC \cite{wu2021UWAC}, One-step \cite{Brandfonbrener2021onestep}, TD3+BC \cite{fujimoto2021TD3-BC}, IQL \cite{kostrikov2021IQL}, MCQ \cite{lyu2022MCQ}, CSVE \cite{chen2023CSVE}, OAP \cite{yang2023OAP}, DTQL \cite{chen2024DTQL}, OAC-BVR \cite{huang2024OAC-BVR}, and TD3-BST \cite{srinivasan2024TD3-BST}. BC stands for behavior cloning, with results sourced from OAC-BVR. The CQL results are from IQL, while the performances of BCQ and UWAC are derived from the reproduction experiments of MCQ, as their original experiments were conducted on ``-v0" datasets. Results for other algorithms are taken from their respective original papers. 

We also conduct evaluation on Maze2d ``-v1" tasks to further examine the effectiveness of ILQ. Here, we compare our method with ROMI-BCQ \cite{wang2021ROMI}, BEAR \cite{kumar2019BEAR}, CQL \cite{kumar2020CQL}, IQL \cite{kostrikov2021IQL}, MCQ \cite{lyu2022MCQ}, Diffuser \cite{janner2022Diffuser}, and PlanCP \cite{sun2024PlanCP}. As mentioned above, the results of BCQ and CQL are reported from IQL and MCQ, respectively. The performances of other methods are obtained from their original papers. 

To further evaluate the proposed ILQ, we conduct additional comparisons on Adroit tasks. The results for TD3+BC \cite{fujimoto2021TD3-BC} are taken from MCQ \cite{lyu2022MCQ}, as the original paper does not include experiments on Adroit domain. The performance for BCQ \cite{fujimoto2019offRL}, CQL \cite{kumar2020CQL}, and IQL \cite{kostrikov2021IQL} are sourced from DTQL \cite{chen2024DTQL}, while the results for other methods are derived from their respective original reports.

\subsubsection{Parameter Settings} The basic hyperparameters of ILQ are described in Table \ref{tab:basic_hypara}. 
\begin{table}[h]
    \caption{Basic Hyperparameters of ILQ}
    \label{tab:basic_hypara}
    \centering
    \begin{tabular}{@{}c@{~}l@{}}
        \toprule
        \makecell[c]{Hyperparameters} & \makecell[c]{Value} \\
        \midrule
         Actor Architechture & input-256-256-256-output \\
         Critic Architechture & input-256-256-256-1 \\
         Optimizer & Adam \cite{kingma2014Adam} \\
         Batch size & $256$ \\
         \makecell[c]{(Critic, Actor)\\ Learning rate} & \makecell[l]{$(3\times 10^{-4}, 1\times 10^{-4})$ for hopper-r,\\ hopper-mr, walker2d-mr, adroit tasks\\ $(5\times 10^{-4}, 3\times 10^{-4})$ for others}\\
         Entropy & True for all except adroit tasks\\
         Training steps & $10^6$ \\
         Behavior training steps & $3\times10^5$ \\
         Dynamics training epochs & $40$ \\
         Discount factor $\gamma$ & $0.99$ \\
         Target update rate $\tau$ & $0.005$ \\
         Sampling Number $M$ & $10$ \\
        \bottomrule
    \end{tabular}
\end{table}
In addition, hyperparameters of the behavior policy model follow the settings in Diffusion-QL \cite{wang2022diffusionQL}. Thus, a $3$-layer MLPs with $256$ hidden units, $5$ diffusion time steps, and corresponding variance schedule \cite{song2021SDE} are implemented for the diffusion model. For the dynamics model, we follow the implementation of MOPO \cite{yu2020MOPO} with $4$-layers MLPs with $200$ hidden units. We only utilize its reward penalty coefficient 2 for hopper-m, walker2d-mr, and 1 for hopper-r, hopper-mr and walker2d-m tasks. Both of behavior policy and dynamics model are optimized by Adam \cite{kingma2014Adam} with learning rate $3\times 10^{-4}$ and $1\times 10^{-3}$, respectively. In addition, we use a cosine learning schedule for adroit tasks. The main hyperparameters $\eta$ and $\delta$ associated with MuJoCo ``-v2" are listed in Table \ref{tab:main_hypara_mujoco}, and the main hyperparameters associated with Maze2D ``-v1" are listed in Table \ref{tab:main_hypara_maze2d} and Adroit ``-v0" are listed in Table \ref{tab:main_hypara_adroit}. All experiments were conducted on the device with $4 \times$ Tesla V100 GPUs. Our code required for conducting all experiments will be made publicly available upon acceptance.

\begin{table}[ht]
    \caption{Main Hyperparameters on MuJoCo Datasets}
    \label{tab:main_hypara_mujoco}
    \centering
    \begin{tabular}{lcc}
        \toprule
        \makecell[c]{Task} & \makecell[c]{Trade-off Factor $\eta$} & \makecell[c]{Offset $\delta$} \\
        \midrule
         halfcheetah-r & 0.95 & 2 \\
         hopper-r & 0.9 & 1 \\
         walker2d-r & 0.7 & 1 \\
         halfcheetah-m & 0.95 & 1 \\
         hopper-m & 0.95 & -2 \\
         walker2d-m & 0.9 & 0.5 \\
         halfcheetah-mr & 0.95 & 2 \\
         hopper-mr & 0.8 & -0.5 \\
         walker2d-mr & 0.9 & 1 \\
         halfcheetah-me & 0.6 & 1 \\
         hopper-me & 0.4 & -0.5 \\
         walker2d-me & 0.8 & 1 \\
        \bottomrule
    \end{tabular}
\end{table}

\begin{table}[ht]
    \caption{Main Hyperparameters on Maze2D Datasets}
    \label{tab:main_hypara_maze2d}
    \centering
    \begin{tabular}{lcc}
        \toprule
        \makecell[c]{Task} & \makecell[c]{Trade-off Factor $\eta$} & \makecell[c]{Offset $\delta$} \\
        \midrule
         maze2d-u & 0.95 & -0.5 \\
         maze2d-ud & 0.95 & 0 \\
         maze2d-m & 0.95 & 0 \\
         maze2d-md & 0.95 & 0 \\
         maze2d-l & 0.95 & 0 \\
         maze2d-ld & 0.95 & 0 \\
        \bottomrule
    \end{tabular}
\end{table}

\begin{table}[ht]
    \caption{Main Hyperparameters on Adroit Datasets}
    \label{tab:main_hypara_adroit}
    \centering
    \begin{tabular}{lcc}
        \toprule
        \makecell[c]{Task} & \makecell[c]{Trade-off Factor $\eta$} & \makecell[c]{Offset $\delta$} \\
        \midrule
         pen-human & 0.8 & -1 \\
         pen-cloned & 0.8 & 0 \\
        \bottomrule
    \end{tabular}
\end{table}

\subsection{More Experimental Results}

\subsubsection{Score Curve Results}
The score curves for MuJoCo tasks are typically of primary interest, which are illustrated in Fig. \ref{fig:ILQ_score_mujoco}

\begin{figure*}[!tb]
    \centering
    \subfloat[ ]{
        \includegraphics[width=0.2\textwidth]{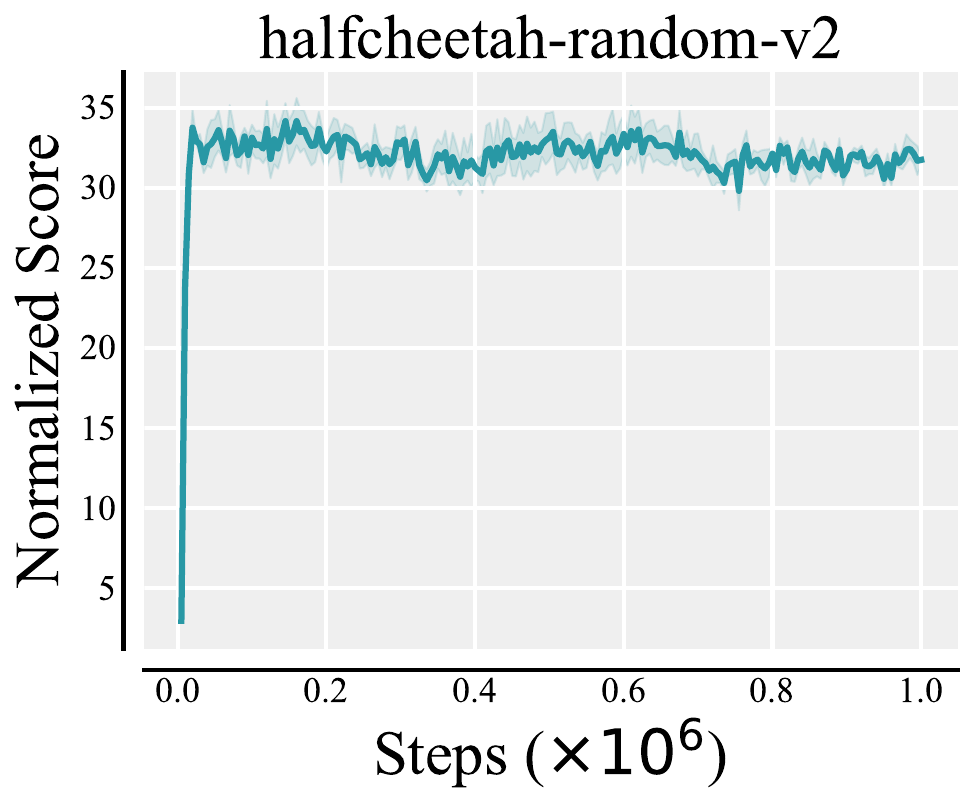}
        \label{fig:ILQ_score_halfcheetah-random-v2}
    }
    \subfloat[ ]{
        \includegraphics[width=0.2\textwidth]{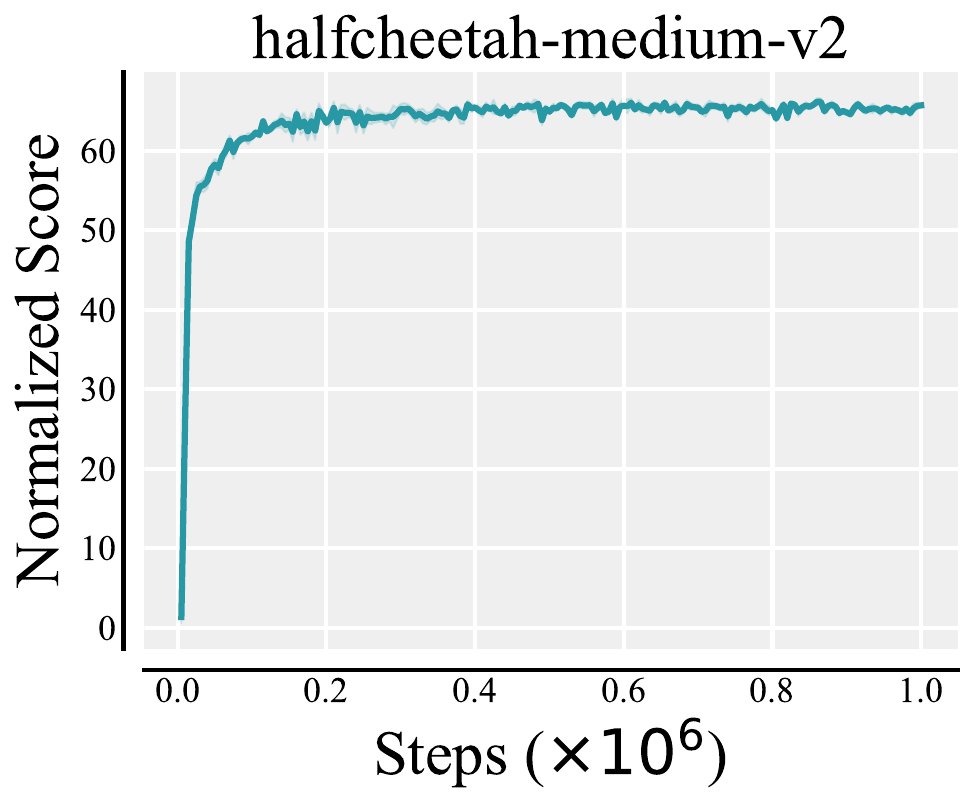}
        \label{fig:ILQ_score_halfcheetah-medium-v2}
    }
    \subfloat[ ]{
        \includegraphics[width=0.2\textwidth]{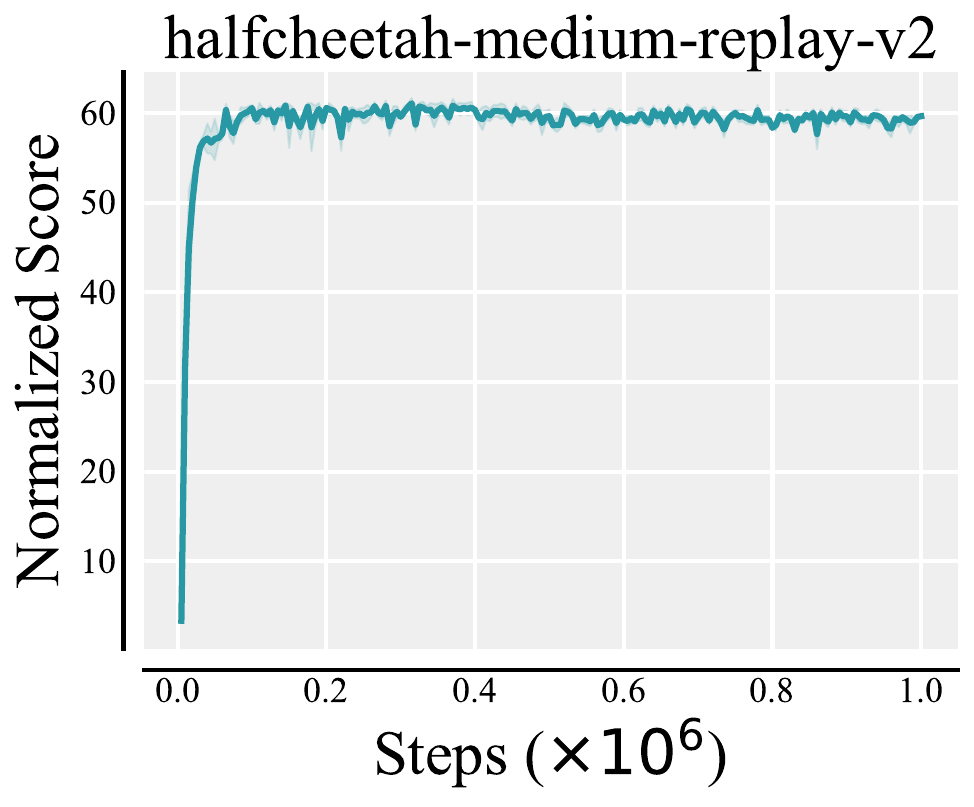}
        \label{fig:ILQ_score_halfcheetah-medium-replay-v2}
    }
    \subfloat[ ]{
        \includegraphics[width=0.2\textwidth]{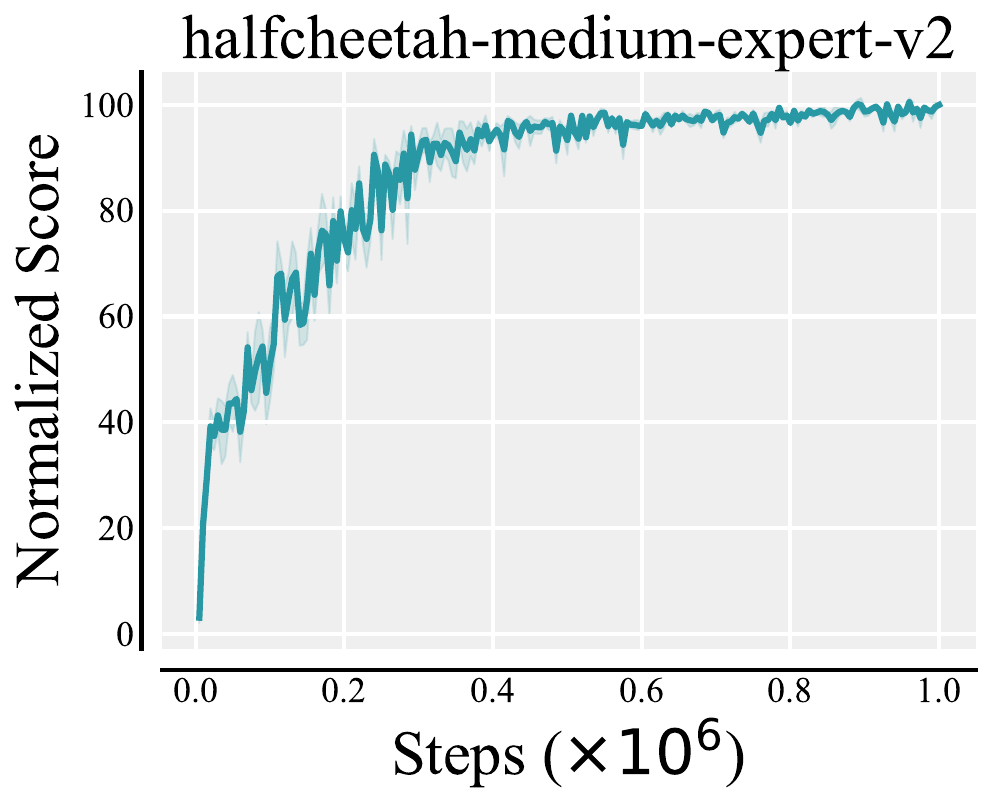}
        \label{fig:ILQ_score_halfcheetah-medium-expert-v2}
    }
    \\
    \subfloat[ ]{
        \includegraphics[width=0.2\textwidth]{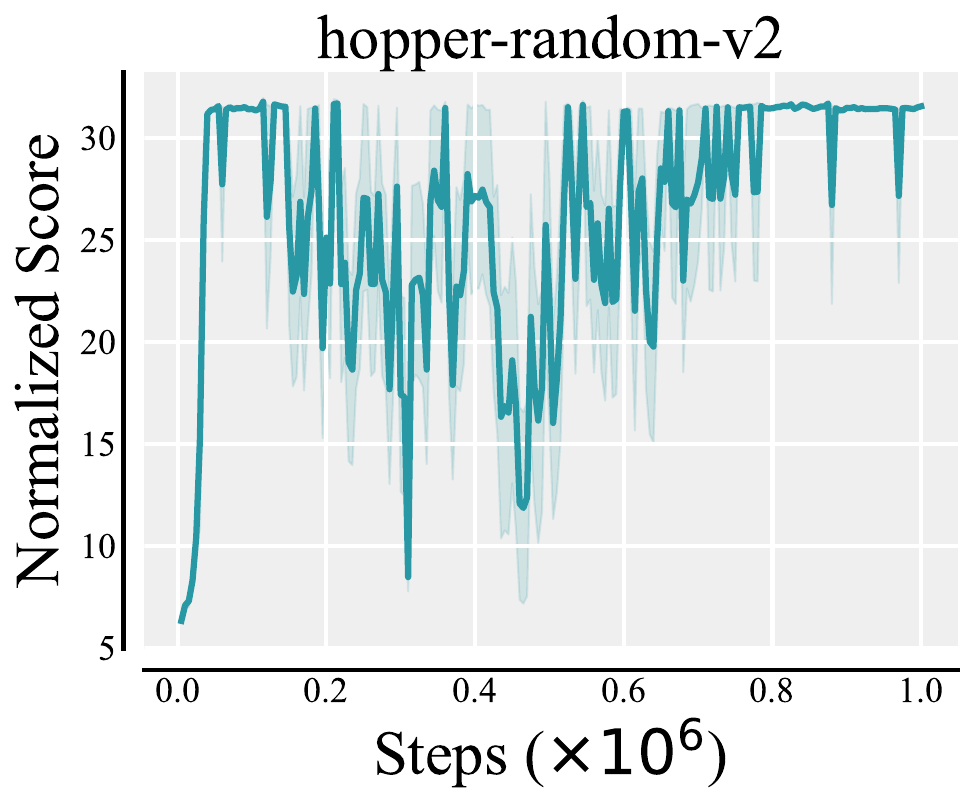}
        \label{fig:ILQ_score_hopper-random-v2}
    }
    \subfloat[ ]{
        \includegraphics[width=0.2\textwidth]{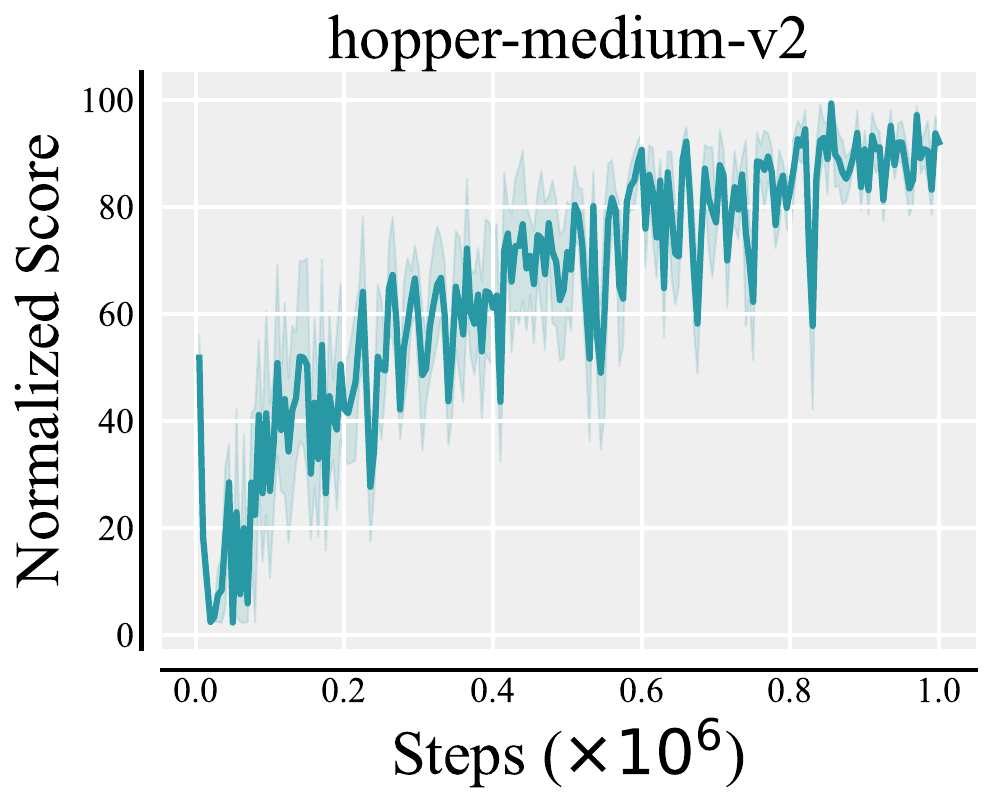}
        \label{fig:ILQ_score_hopper-medium-v2}
    }
    \subfloat[ ]{
        \includegraphics[width=0.2\textwidth]{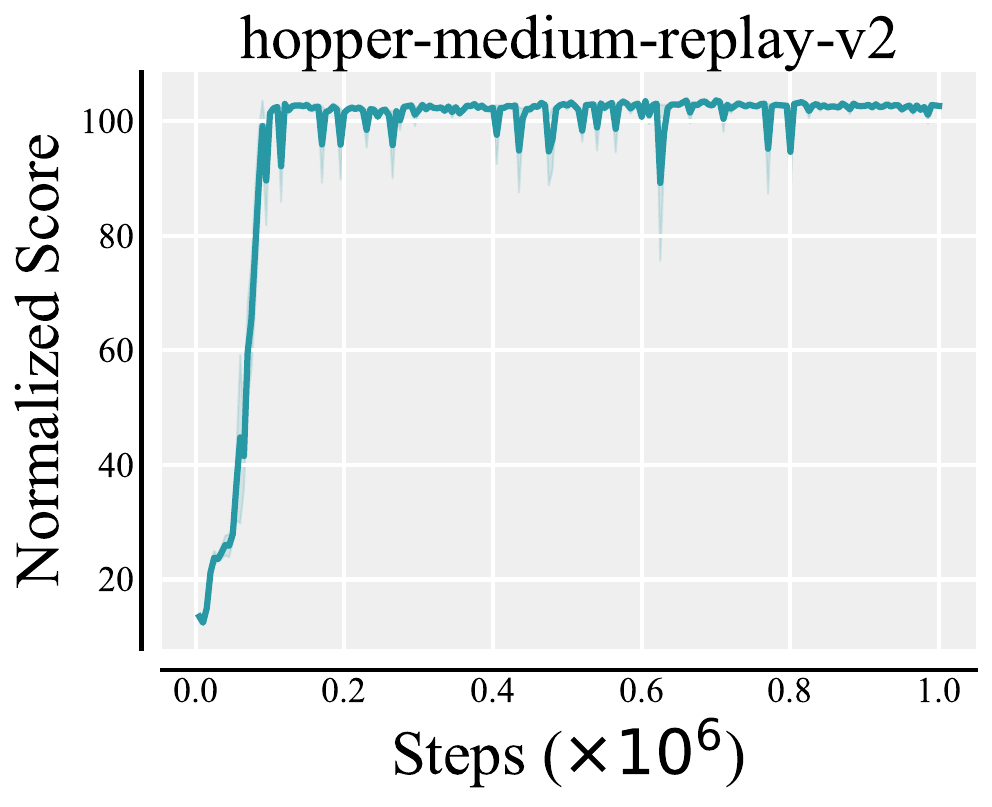}
        \label{fig:ILQ_score_hopper-medium-replay-v2}
    }
    \subfloat[ ]{
        \includegraphics[width=0.2\textwidth]{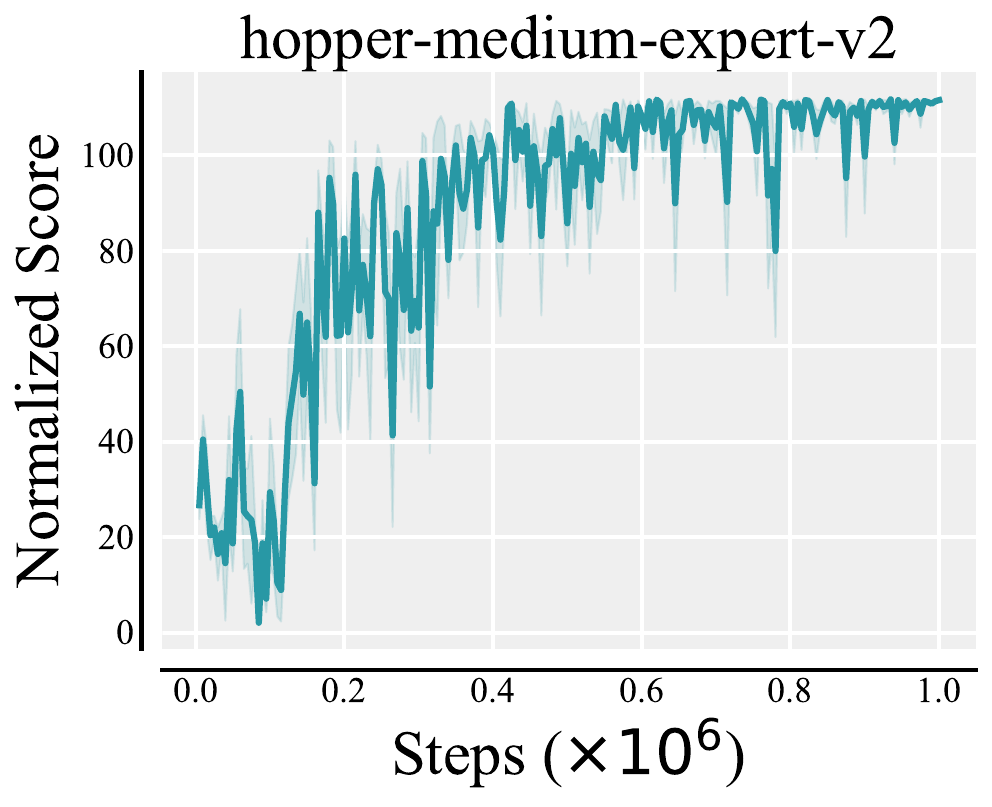}
        \label{fig:ILQ_score_hopper-medium-expert-v2}
    }
    \\
    \subfloat[ ]{
        \includegraphics[width=0.2\textwidth]{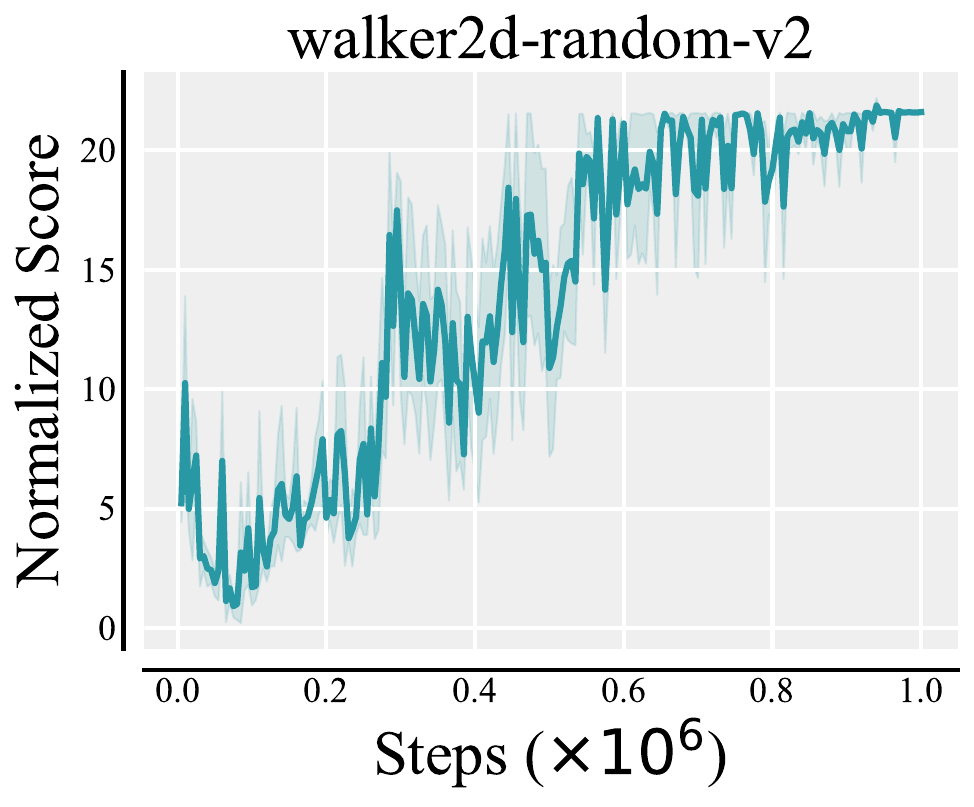}
        \label{fig:ILQ_score_walker2d-random-v2}
    }
    \subfloat[ ]{
        \includegraphics[width=0.2\textwidth]{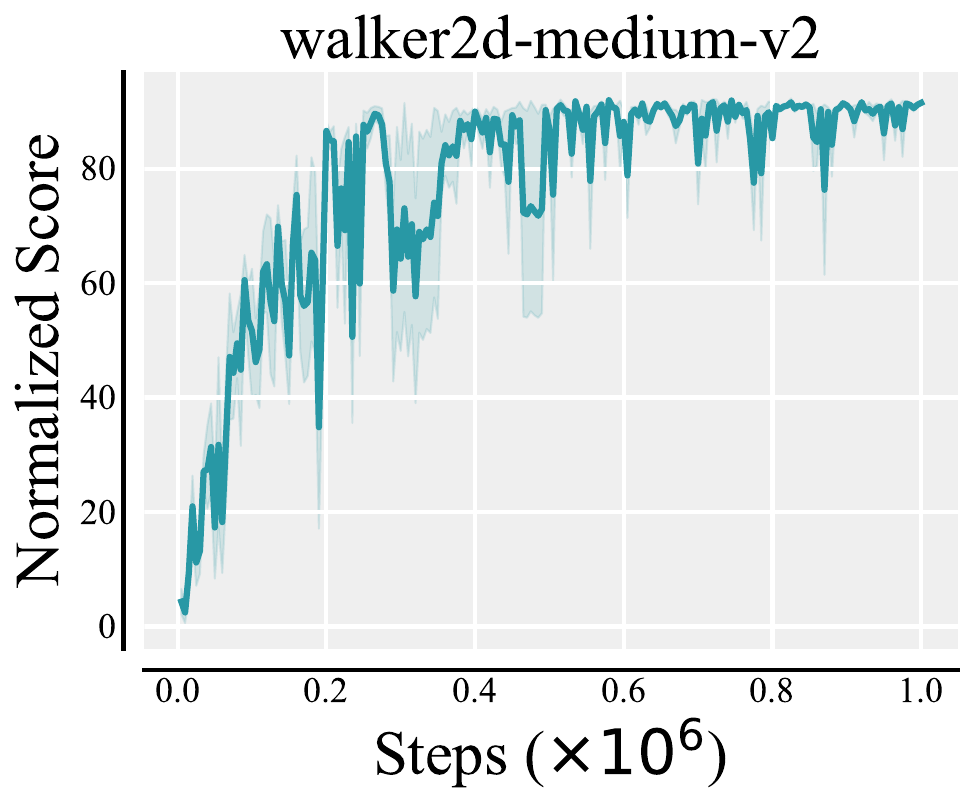}
        \label{fig:ILQ_score_walker2d-medium-v2}
    }
    \subfloat[ ]{
        \includegraphics[width=0.2\textwidth]{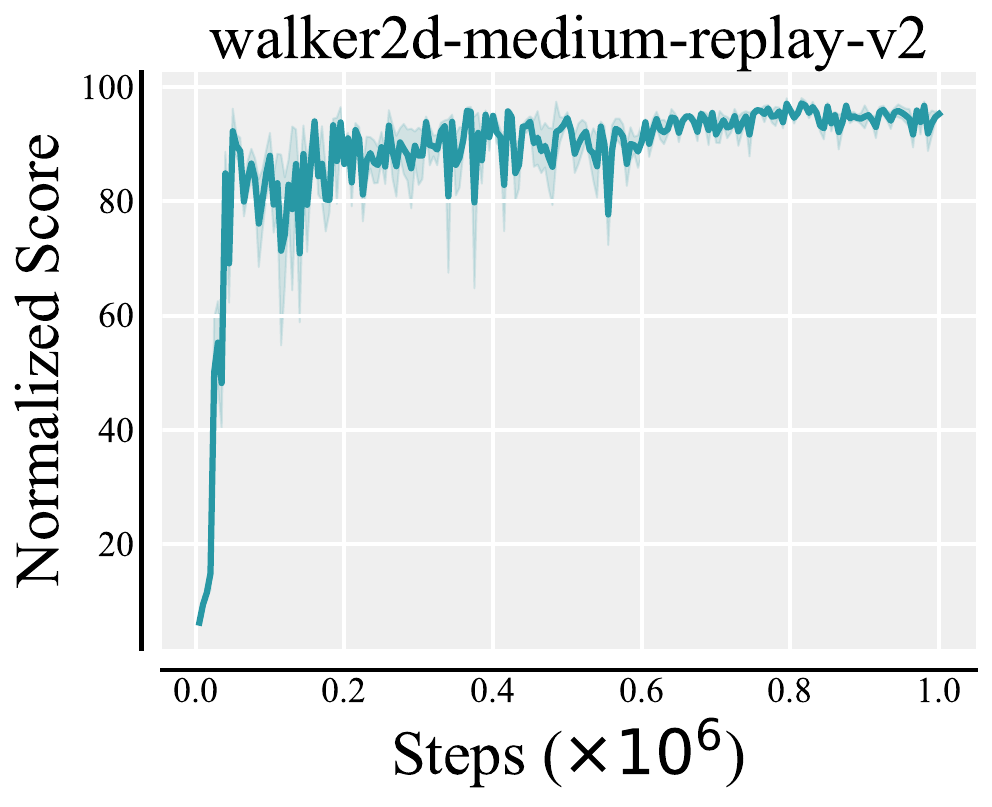}
        \label{fig:ILQ_score_walker2d-medium-replay-v2}
    }
    \subfloat[ ]{
        \includegraphics[width=0.2\textwidth]{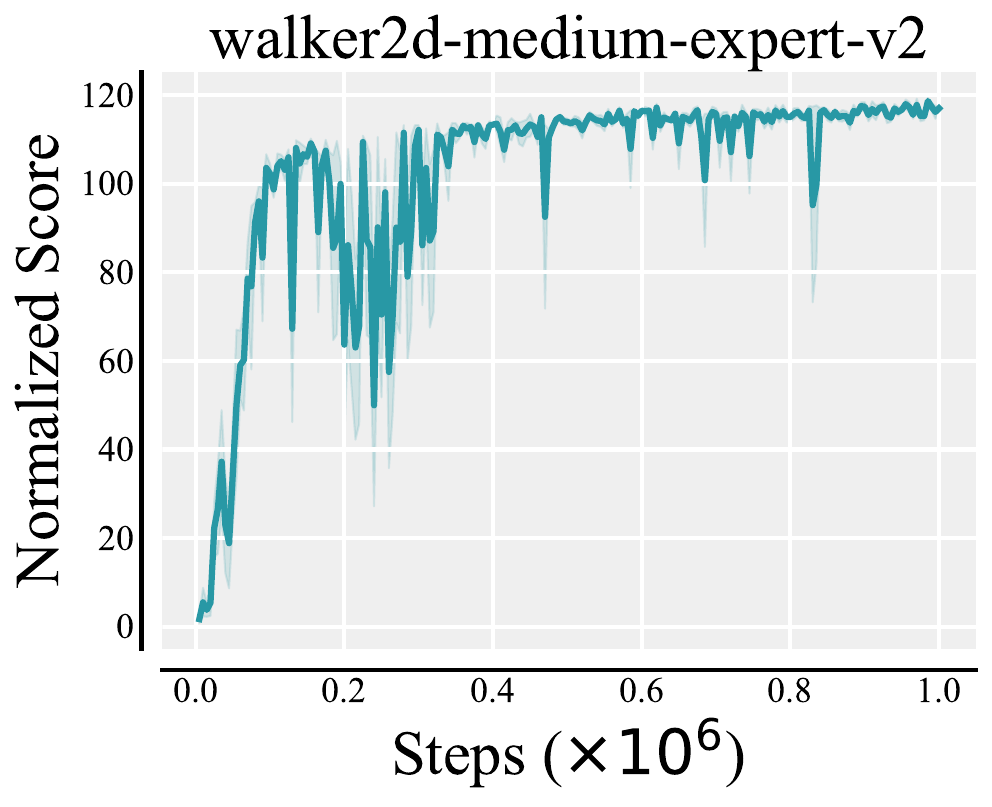}
        \label{fig:ILQ_score_walker2d-medium-expert-v2}
    }
    \caption{Normailzed score curves of ILQ on MuJoCo ``-v2". The results are averaged over 5 different random seeds. Shaded areas indicate standard deviation.}
    \label{fig:ILQ_score_mujoco}
\end{figure*}

\subsection{More Sensitive Analyses}
\subsubsection{Sensitive Analysis of Sampling Number $M$}
We did not finetune the sampling number $M$, which was set to $10$ in all experiments. According to the practical implementation of ILB operator, $M$ implicitly influences the estimation of the maximum behavior value. To assess its impact on performance, we conduct extra sensitivity analyses on the halfcheetah-m, hopper-mr, and walker2d-me tasks. The experimental results indicate that performance on these tasks remains stable when $M$ is set to $5$, $10$, and $15$, respectively, as illustrated in the bar charts in Fig. \ref{fig:N-action-analysis}.

\begin{figure}
    \centering
    \includegraphics[width=0.45\textwidth]{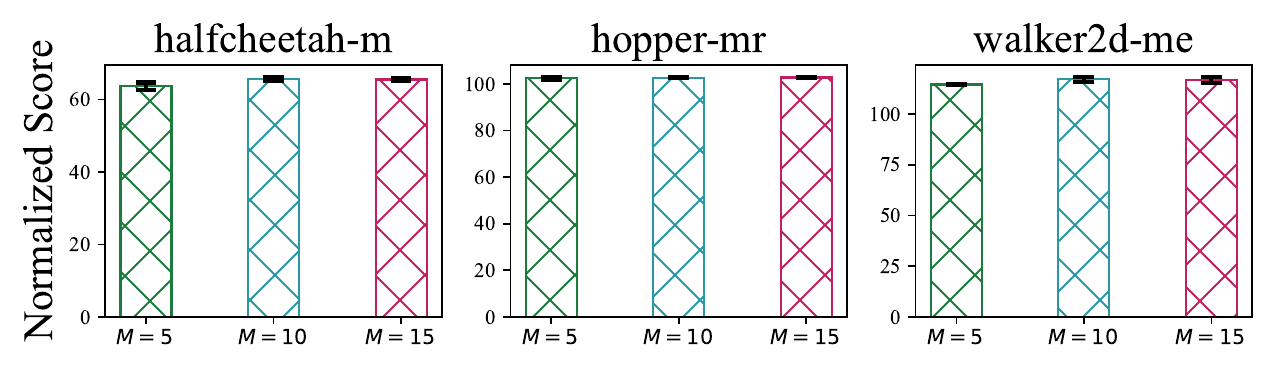}
    \caption{Performances of ILQ under different sampling number $M$.}
    \label{fig:N-action-analysis}
\end{figure}

\subsubsection{Sensitive Analysis of Dynamics Model Accuracy} Our analysis includes: (1) Training epochs: The Gaussian-fitted dynamics model shows stable performance across different epoch settings (Table \ref{tab:sens_epoch}), indicating robustness to this hyperparameter. (2) Training data quantity: We also conducted experiments using 20\% reduced training data for the dynamics model. While Table \ref{tab:sens_data} shows a slight performance decrease on most of tasks when using only 80\% data, the method maintains reasonable effectiveness.

\begin{table}[!ht]
  \caption{Changing training epochs of dynamics model. `ha'=halfcheetah, `ho'=hopper, `wa'=walker2d, `m'=medium, `mr'=medium-replay. Epochs=40 is the default setting.}
  \label{tab:sens_epoch}
  \small
  \centering
  \resizebox{\linewidth}{!}{
  \begin{tabular}{ccccccc}
    \toprule
    Epochs & ha-m & ho-m & wa-m & ha-mr & ho-mr & wa-mr\\
    \midrule
    35    & $64.4_{\pm 0.6}$ & $93.5_{\pm 5.9}$ & $92.0_{\pm 1.1}$ & $58.2_{\pm 1.1}$ & $102.9_{\pm 0.5}$ & $91.5_{\pm 6.0}$ \\
    40    & $65.7_{\pm 0.5}$ & $92.1_{\pm 5.8}$ & $91.5_{\pm 0.7}$ & $59.6_{\pm 1.0}$ & $102.7_{\pm 0.3}$ & $95.3_{\pm 1.8}$ \\
    45    & $65.0_{\pm 0.4}$ & $93.9_{\pm 6.0}$ & $90.0_{\pm 0.5}$ & $58.1_{\pm 0.3}$ & $102.4_{\pm 0.5}$ & $93.3_{\pm 0.7}$ \\
    50    & $64.1_{\pm 0.2}$ & $90.9_{\pm 9.1}$ & $89.4_{\pm 0.8}$ & $59.2_{\pm 0.6}$ & $103.1_{\pm 0.7}$ & $97.9_{\pm 0.7}$ \\
        \bottomrule
  \end{tabular}
  }
\end{table}

\begin{table}[!ht]
  \caption{Reducing training data for dynamics model.}
  \label{tab:sens_data}
  \small
  \centering
  \resizebox{\linewidth}{!}{
  \begin{tabular}{ccccccc}
    \toprule
    Data ratio & ha-m & ho-m & wa-m & ha-mr & ho-mr & wa-mr\\
    \midrule
    All  & $65.7_{\pm 0.5}$ & $92.1_{\pm 5.8}$ & $91.5_{\pm 0.7}$ & $59.6_{\pm 1.0}$ & $102.7_{\pm 0.3}$ & $95.3_{\pm 1.8}$ \\
    $80\%$    & $64.7_{\pm 0.0}$ & $98.4_{\pm 4.9}$ & $89.8_{\pm 3.3}$ & $58.9_{\pm 1.6}$ & $102.6_{\pm 0.2}$ & $88.0_{\pm 5.5}$ \\
        \bottomrule
  \end{tabular}
  }
\end{table}

\subsection{More Ablation Studies}

We conduct additional ablation studies to assess the effectiveness of both the imagination and limitation components. The results without the imagination component are shown in Fig. \ref{fig:ILQ_wo_img-app}. While competitive performance is achieved on some tasks, such as in Fig. \ref{fig:ILQ_wo_img-app}\subref{fig:ILQ_woimg_halfcheetah-medium-v2} and \subref{fig:ILQ_woimg_hopper-medium-v2}, performance significantly drops on other tasks. As illustrated in Fig. \ref{fig:ILQ_wo_img-app}\subref{fig:ILQ_woimg_halfcheetah-medium-expert-v2-app}, \subref{fig:ILQ_woimg_hopper-medium-expert-v2}, \subref{fig:ILQ_woimg_walker2d-medium-expert-v2}, and \subref{fig:ILQ_woimg_walker2d-medium-v2-app}, performance deteriorates sharply, with the normalized score curves exhibiting highly oscillatory behavior. As discussed previously, this is due to the use of the maximum behavior value as the target for OOD actions, which introduces uncontrollable bias and ultimately hampers policy improvement. These results highlight the critical role of the imagination component in providing calibrated target values within proper constraints.

In further studies on the limitation component, slightly better performance is observed when the imagination component is exclusively used, as shown in Fig. \ref{fig:ILQ_wo_lmt-app}\subref{fig:ILQ_wolmt_halfcheetah-medium-v2-app}. This suggests that, in certain cases, the imagination component can offer reliable guidance. However, as seen in Fig. \ref{fig:ILQ_wo_lmt-app}\subref{fig:ILQ_wolmt_halfcheetah-medium-expert-v2}, \subref{fig:ILQ_wolmt_walker2d-medium-expert-v2}, and \subref{fig:ILQ_wolmt_walker2d-medium-v2}, performance drops significantly, and in some cases, policies completely collapse, as evidenced in the hopper-medium-expert and hopper-medium tasks (Fig. \ref{fig:ILQ_wo_lmt-app}\subref{fig:ILQ_wolmt_hopper-medium-expert-v2} and \ref{fig:ILQ_wo_lmt-app}\subref{fig:ILQ_wolmt_hopper-medium-v2-app}). This underscores the importance of the limitation component in preventing overly optimistic estimates. These comprehensive studies demonstrate that both components are essential for accurate OOD action-value estimation.

\begin{figure}[!tb]
    \centering
    {
        \includegraphics[width=0.25\textwidth]{figures/woimg_legend.pdf}
    }
    \\
    \subfloat[ ]{
        \includegraphics[width=0.22\textwidth]{figures/ILQ_woimg_halfcheetah-medium-expert-v2.pdf}
        \label{fig:ILQ_woimg_halfcheetah-medium-expert-v2-app}
    }
    \subfloat[ ]{
        \includegraphics[width=0.22\textwidth]{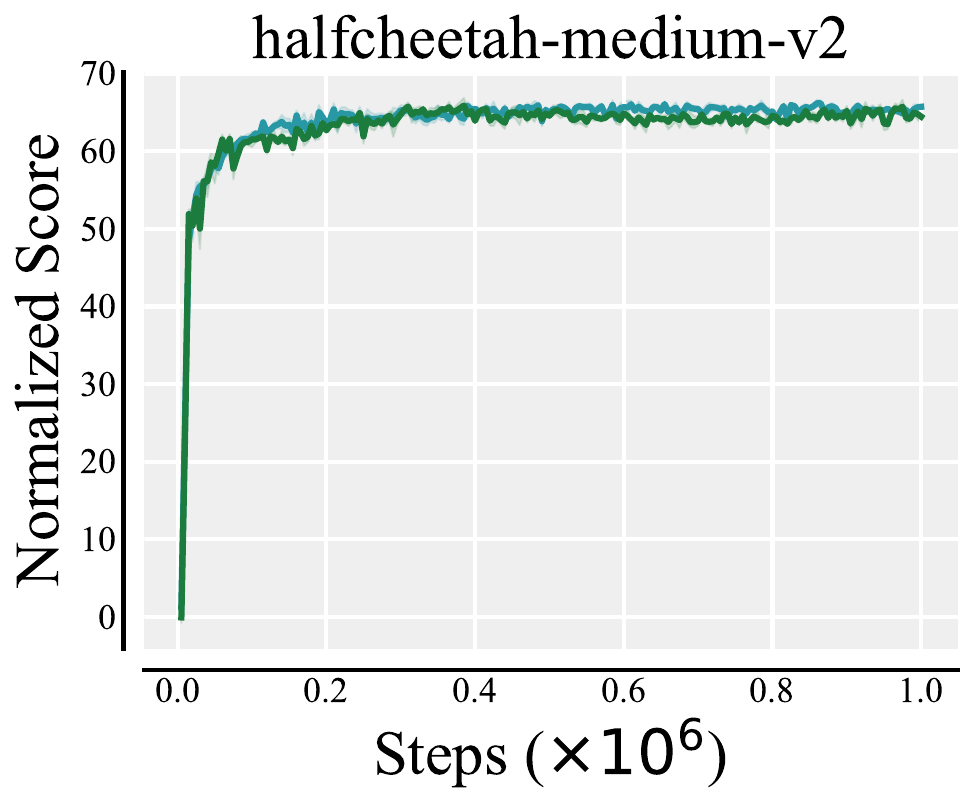}
        \label{fig:ILQ_woimg_halfcheetah-medium-v2}
    }
    \\
    \subfloat[ ]{
        \includegraphics[width=0.22\textwidth]{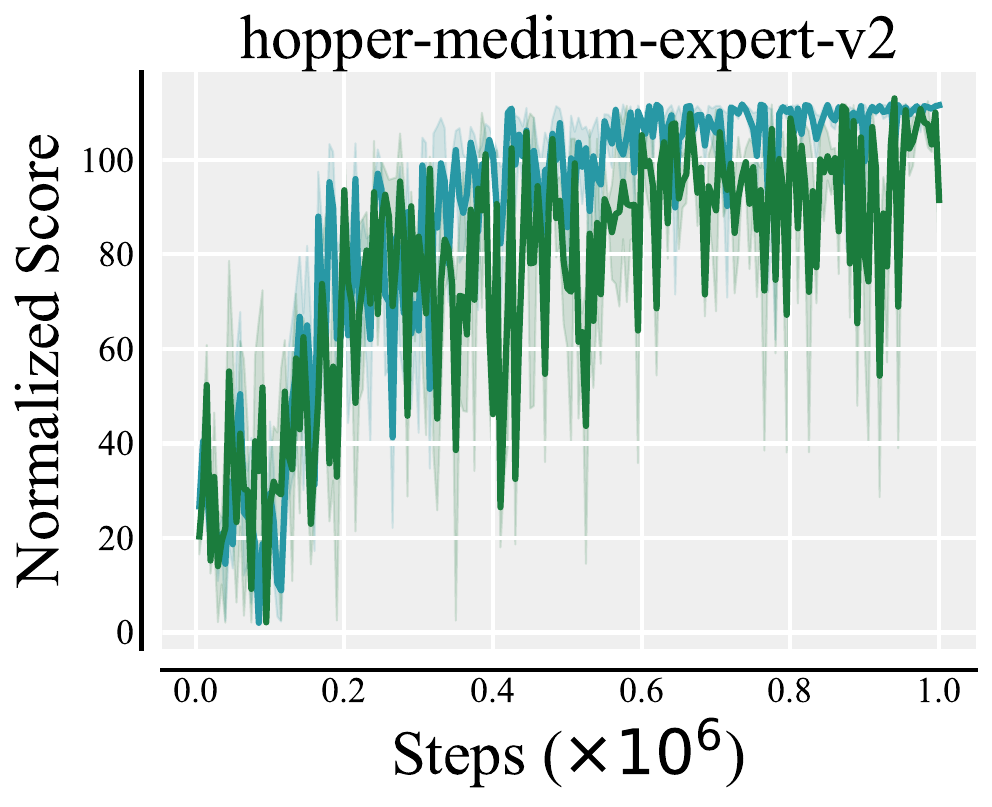}
        \label{fig:ILQ_woimg_hopper-medium-expert-v2}
    }
    \subfloat[ ]{
        \includegraphics[width=0.22\textwidth]{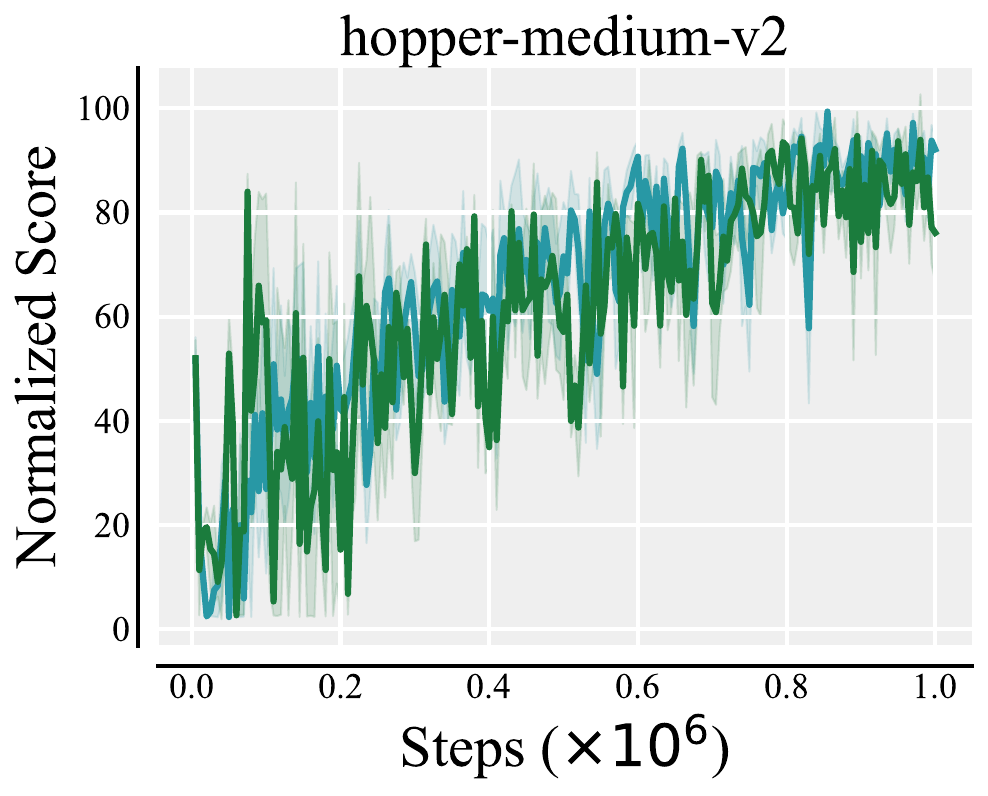}
        \label{fig:ILQ_woimg_hopper-medium-v2}
    }
    \\
    \subfloat[ ]{
        \includegraphics[width=0.22\textwidth]{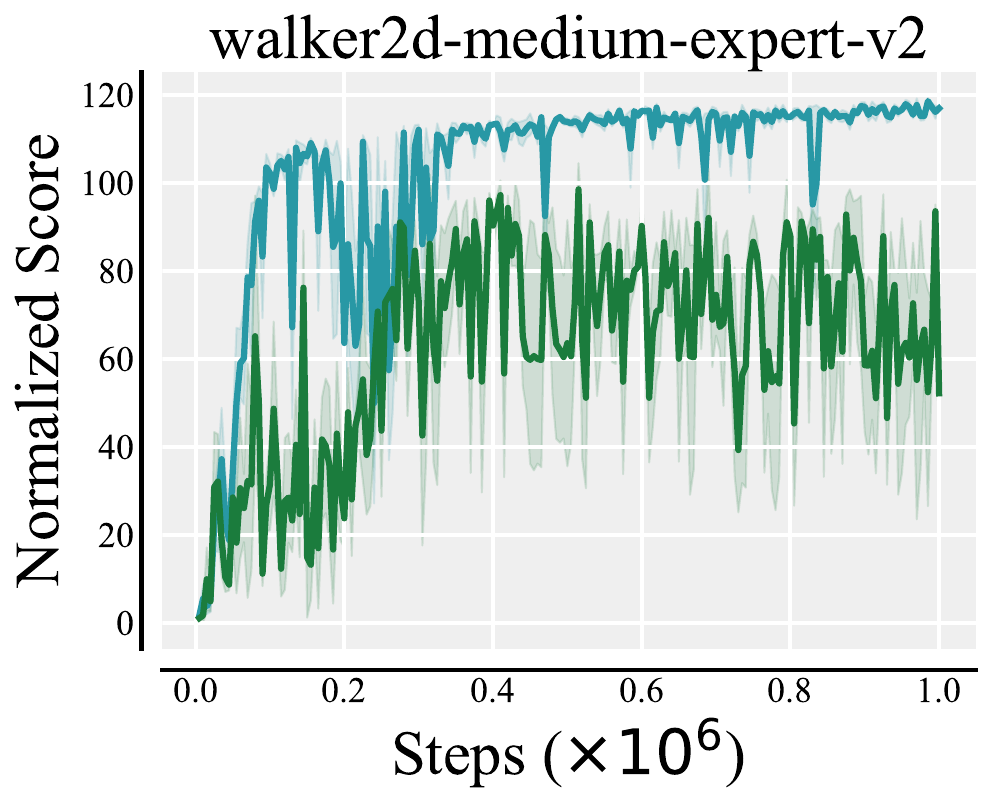}
        \label{fig:ILQ_woimg_walker2d-medium-expert-v2}
    }
    \subfloat[ ]{
        \includegraphics[width=0.22\textwidth]{figures/ILQ_woimg_walker2d-medium-v2.pdf}
        \label{fig:ILQ_woimg_walker2d-medium-v2-app}
    }
    \caption{Performance comparison of the ILQ algorithm with and without the imagined value $y_{\rm img}^Q$ in the target value. }
    \label{fig:ILQ_wo_img-app}
\end{figure}

\begin{figure}[!tb]
    \centering
    {
        \includegraphics[width=0.25\textwidth]{figures/wolmt_legend.pdf}
    }
    \\
    \subfloat[ ]{
        \includegraphics[width=0.22\textwidth]{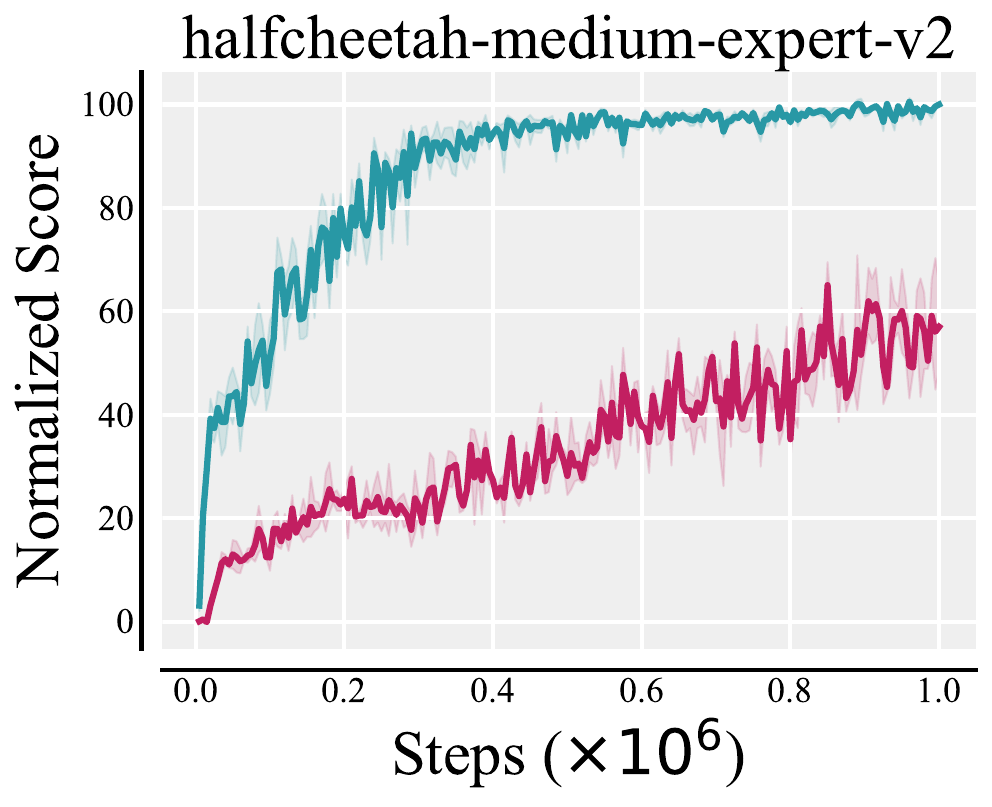}
        \label{fig:ILQ_wolmt_halfcheetah-medium-expert-v2}
    }
    \subfloat[ ]{
        \includegraphics[width=0.22\textwidth]{figures/ILQ_wolmt_halfcheetah-medium-v2.pdf}
        \label{fig:ILQ_wolmt_halfcheetah-medium-v2-app}
    }
    \\
    \subfloat[ ]{
        \includegraphics[width=0.22\textwidth]{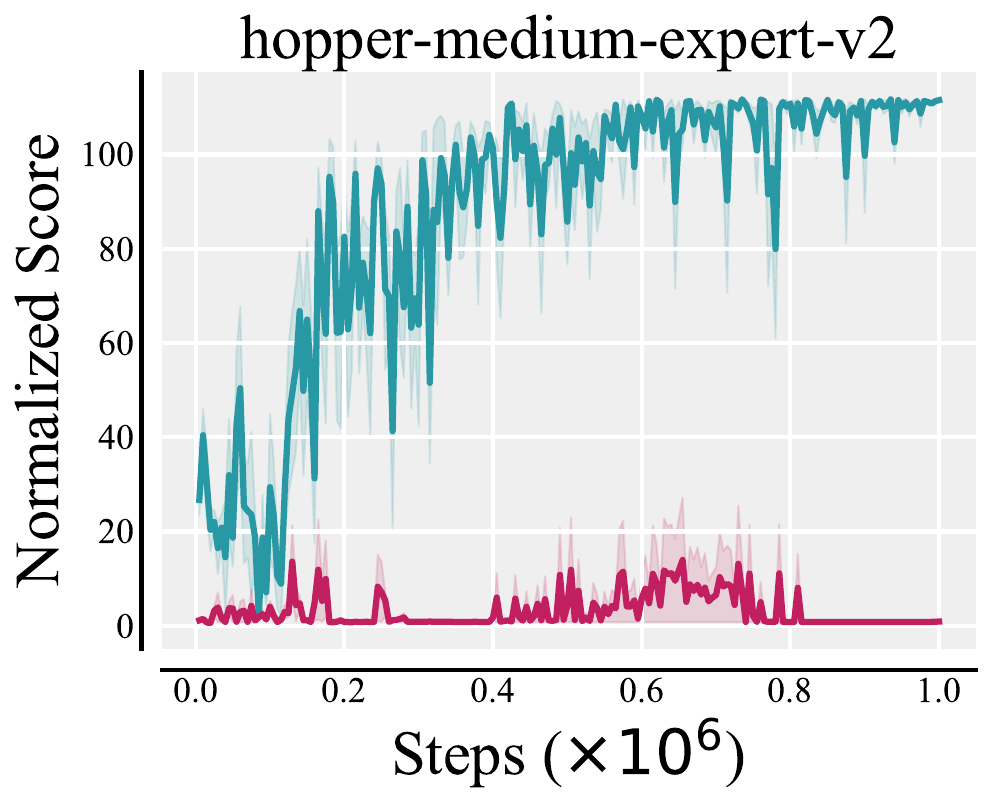}
        \label{fig:ILQ_wolmt_hopper-medium-expert-v2}
    }
    \subfloat[ ]{
        \includegraphics[width=0.22\textwidth]{figures/ILQ_wolmt_hopper-medium-v2.pdf}
        \label{fig:ILQ_wolmt_hopper-medium-v2-app}
    }
    \\
    \subfloat[ ]{
        \includegraphics[width=0.22\textwidth]{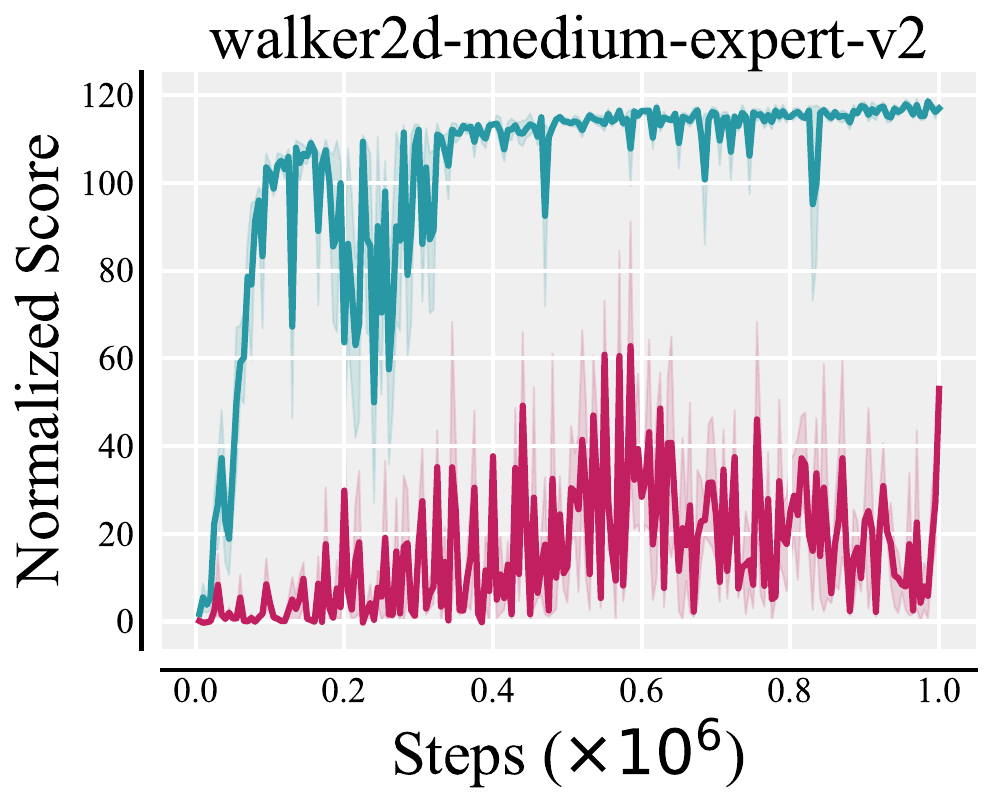}
        \label{fig:ILQ_wolmt_walker2d-medium-expert-v2}
    }
    \subfloat[ ]{
        \includegraphics[width=0.22\textwidth]{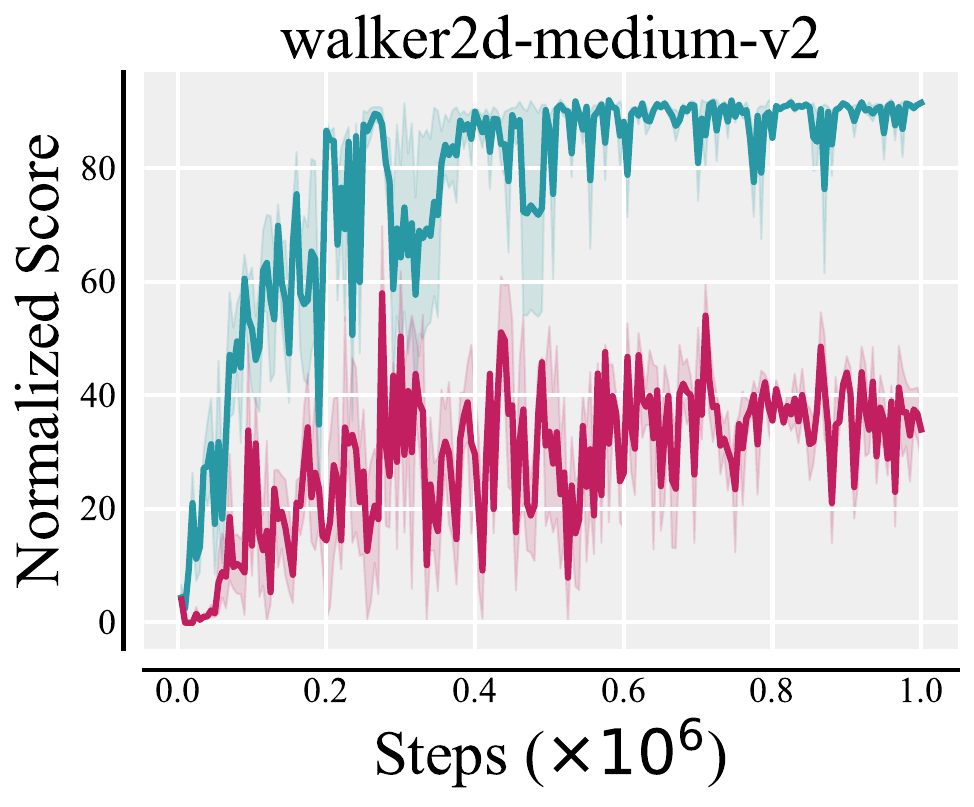}
        \label{fig:ILQ_wolmt_walker2d-medium-v2}
    }
    \caption{Performance comparison of the ILQ algorithm with and without the limitation value $y_{\rm lmt}^Q$ in the target value.}
    \label{fig:ILQ_wo_lmt-app}
\end{figure}

\subsection{Further Verification in Q-value}
ILQ estimates OOD Q-values by preserving the imagined values as much as possible while adhering to the maximum behavior value constraint. This approach ensures appropriately optimistic estimates, as shown in the section of Introduction, thus avoiding the deliberate pessimism of value regularization methods.

To further understand the interaction between the imagined value and the maximum behavioral value, we examined the difference between them, i.e., $y_{\rm img}^Q - y_{\rm lmt}^Q$. Figure \ref{fig:ILQ_deep_demo_in_Q}\subref{fig:Q_deviation} illustrates how the range (cyan area) of this difference evolves during training. For the upper boundary curve (green), where the imagined value exceeds the limiting value, we retain the limiting value. Conversely, for the lower boundary curve (red), where the limiting value is higher than the imagined value, we retain the imagined value. This suggests a mutually constraining relationship between the two components. As seen in Fig. \ref{fig:ILQ_wo_img}, the absence of the imagined value leads to a significant performance decrease. Meanwhile, as shown in Fig. \ref{fig:ILQ_wo_lmt}, unconstrained imagining can result in false optimistic estimates, potentially causing the policy to collapse. Notably, in the scenario depicted in Fig. \ref{fig:ILQ_wo_lmt}\subref{fig:ILQ_wolmt_hopper-medium-v2}, this false optimistic estimation even grows exponentially, reaching a Q-value of $10^{13}$, as shown in Fig. \ref{fig:ILQ_deep_demo_in_Q}\subref{fig:Q_value_wolmt}.

\begin{figure}[!tb]
    \centering
    \subfloat[ ]{
        \includegraphics[width=0.22\textwidth]{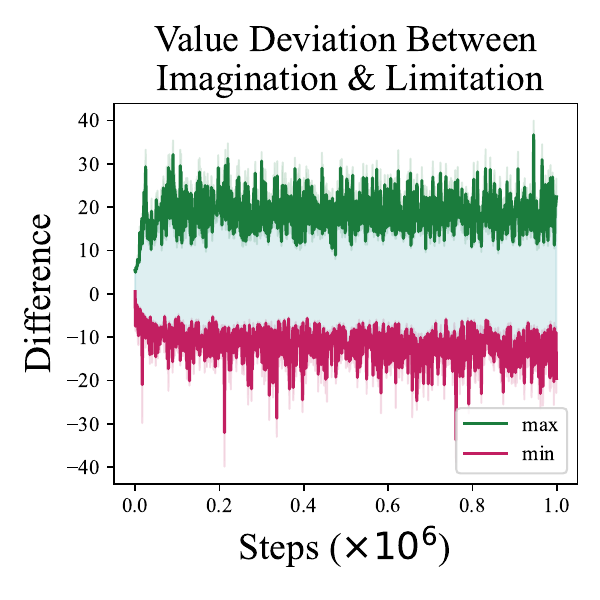}
        \label{fig:Q_deviation}
    }
    \subfloat[ ]{
        \includegraphics[width=0.22\textwidth]{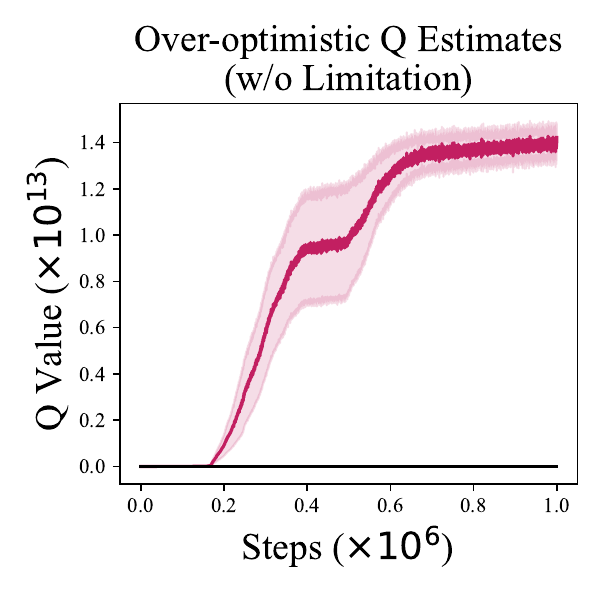}
        \label{fig:Q_value_wolmt}
    }
    \caption{(a) illustrates the evolving range of the difference between the imagined value and the limiting value during training. (b) shows the exponential growth of false optimistic value estimations in the w/o limitation scenario on the hopper-medium-v2 task.}
    \label{fig:ILQ_deep_demo_in_Q}
\end{figure}

\subsection{Limitations of Theoretical Results}
Our theoretical analyses - consistent with most theoretical works in both online and offline RL - assumes tabular MDPs, as formal guarantees under neural network function approximation remain challenging. We will note it as a direction for future research.

\subsection{Computational Cost} 
Regarding computational efficiency, we provide detailed comparisons of computation costs across different methods (measured on a Tesla V100 server) in Table \ref{tab:computation_cost}, which shows ILQ achieves competitive efficiency relative to baselines.
\begin{table}[!ht]
  \caption{Training time per 100 steps on hopper-medium task.}
  \label{tab:computation_cost}
  \small
  \centering
  \begin{tabular}{ccccccc}
    \toprule
    & BEAR & CQL & IQL & MCQ & DTQL & ILQ\\
    \midrule
    Time(s)    & 3.06 & 2.14  & 1.01 & 3.05       & 2.24       & 2.21 \\
        \bottomrule
  \end{tabular}
\end{table}

\subsection{Discussion of Lipschitz Continuity Assumption} 
The Lipschitz condition on reward functions is commonly adopted in offline RL theoretical analyses \cite{huang2024OAC-BVR}, though it represents a strong practical assumption.
Mathematically, any continuously differentiable function on a compact set satisfies the Lipschitz condition. In practice, we can verify this by checking: (1) whether the real-world reward function is sufficiently smooth (continuously differentiable), and (2) whether the action space is bounded (compact).
For our experimental environments: (1) In MuJoCo and Adroit tasks, the reward functions are continuous and actions are bounded within $[-1,1]^{|\mathcal{A}|}$, so the condition typically holds. (2) For Maze tasks with sparse rewards, the assumption theoretically fails.

\end{document}